\documentclass{article}

\usepackage[most]{tcolorbox}
\tcbuselibrary{listings, skins, breakable}
\usepackage{listings}
\usepackage{float}
\usepackage{placeins}  

\lstdefinelanguage{json}{
  basicstyle=\small\ttfamily,
  showstringspaces=false,
  breaklines=true,
  columns=flexible,
  keepspaces=true,
  morestring=[b]",
  stringstyle=\color{black},
  literate=
   *{0}{{0}}1 {1}{{1}}1 {2}{{2}}1 {3}{{3}}1 {4}{{4}}1
    {5}{{5}}1 {6}{{6}}1 {7}{{7}}1 {8}{{8}}1 {9}{{9}}1
    {:}{{:}}1 {,}{{,}}1 {\{}{{\{}}1 {\}}{{\}}}1
    {[}{{[}}1 {]}{{]}}1
}

\usepackage{microtype}
\usepackage{graphicx}
\usepackage{subcaption}
\usepackage{booktabs} 

\usepackage{hyperref}

\usepackage[accepted]{icml2026}

\usepackage{amsmath}
\usepackage{amssymb}
\usepackage{mathtools}
\usepackage{amsthm}

\usepackage{booktabs}
\usepackage{multirow} 
\usepackage{graphicx}
\usepackage{subcaption}
\usepackage{hyperref}
\usepackage{tabularx}
\usepackage{enumitem}
\usepackage{titletoc}
\usepackage{threeparttable}
\usepackage{url}
\usepackage[most]{tcolorbox} 
\usepackage{amsmath,amssymb,mathtools,bm}
\usepackage{algorithm}
\usepackage{algorithmic}
\usepackage{microtype}
\usepackage{makecell}
\usepackage{amsthm,amsmath,amssymb}
\usepackage[table]{xcolor}
\usepackage{listings}
\usepackage{xcolor}
\usepackage{wrapfig}
\usepackage{siunitx}

\usepackage[capitalize,noabbrev]{cleveref}

\theoremstyle{plain}
\newtheorem{theorem}{Theorem}[section]
\newtheorem{proposition}[theorem]{Proposition}
\newtheorem{lemma}[theorem]{Lemma}

\theoremstyle{definition}

\newtheorem{assumption}[theorem]{Assumption}
\theoremstyle{remark}
\newtheorem{remark}[theorem]{Remark}

\definecolor{lightorange}{HTML}{FAE3D6}
\definecolor{lightblue}{rgb}{0.8, 0.9, 1}

\usepackage[textsize=tiny]{todonotes}

\icmltitlerunning{Letting Trajectories Spread: Quality-Preserving Control for Diverse Flow Matching}

\begin{document}

\twocolumn[
  \icmltitle{\texorpdfstring{Letting Trajectories Spread: \\ Quality-Preserving Control for Diverse Flow Matching}
  {Letting Trajectories Spread: Quality-Preserving Control for Diverse Flow Matching}}
  
  \icmlsetsymbol{equal}{*}
  \icmlsetsymbol{correspond}{\ensuremath{\dagger}}

  \begin{icmlauthorlist}
      \icmlauthor{Jingxuan Wu}{equal,unc}
      \icmlauthor{Zhenglin Wan}{equal,nus}
      \icmlauthor{Xingrui Yu}{correspond,cfar,ihpc}
      \icmlauthor{Yuzhe Yang}{ucsb}
      \icmlauthor{Bo An}{ntu}
      \icmlauthor{Ivor Tsang}{cfar,ihpc,ntu}
      \icmlauthor{Yang You}{nus}
  \end{icmlauthorlist}

  \icmlaffiliation{unc}{The University of North Carolina at Chapel Hill}
  \icmlaffiliation{ntu}{Nanyang Technological University}
  \icmlaffiliation{cfar}{CFAR, Agency for Science, Technology and Research, Singapore}
  \icmlaffiliation{ihpc}{IHPC, Agency for Science, Technology and Research, Singapore}
  \icmlaffiliation{ucsb}{University of California, Santa Barbara}
  \icmlaffiliation{nus}{National University of Singapore}

  \icmlcorrespondingauthor{Xingrui Yu}{Yu\_Xingrui@a-star.edu.sg}

  \icmlkeywords{Machine Learning, ICML}

  \vskip 0.3in
]



\printAffiliationsAndNotice{}  

\begin{abstract}
Flow-based text-to-image models follow deterministic trajectories, making it costly to explore diverse modes under limited sampling budgets.
Existing approaches to improving diversity often rely on retraining or degrade image fidelity.
To address this limitation, we present a training-free, inference-time control mechanism that makes the flow itself diversity-aware.
 Our core insight is to encourage diversity through guidance that is geometrically decoupled from the model’s quality-seeking direction.
Our method simultaneously encourages lateral spread among trajectories via a feature-space objective and reintroduces uncertainty through a time-scheduled stochastic perturbation. Crucially, this perturbation is projected to be orthogonal to the generation flow, a geometric constraint that allows it to boost variation without degrading image details or prompt fidelity.
Theoretically, we show that this design monotonically increases a volume surrogate while approximately preserving the marginal distribution, providing a principled explanation for the robustness of generation quality.
Empirically, across multiple text-to-image settings under fixed sampling budgets, our method consistently improves diversity metrics such as the Vendi Score and Brisque over strong baselines, while upholding image quality and alignment.
\end{abstract}

\section{Introduction}

Recent advances in text-to-image (T2I) synthesis have unlocked unprecedented capabilities, enabling the creation of photorealistic visuals for applications ranging from digital art and design to scientific visualization \citep{rombach2022high, saharia2022photorealistic, ramesh2022hierarchical, zhang2023text}. Among the leading paradigms, Flow Matching (FM) and Rectified Flow (RF) have gained prominence due to their efficient inference and solid theoretical foundation \citep{lipman2022flow, liu2022flow}. While significant efforts have pushed the fidelity and prompt-alignment of these models to new heights \citep{rombach2022high, esser2024scaling}, a critical challenge remains: a striking lack of semantic diversity.

\begin{figure}[t]
    \centering
    \includegraphics[width=1\linewidth
    ]{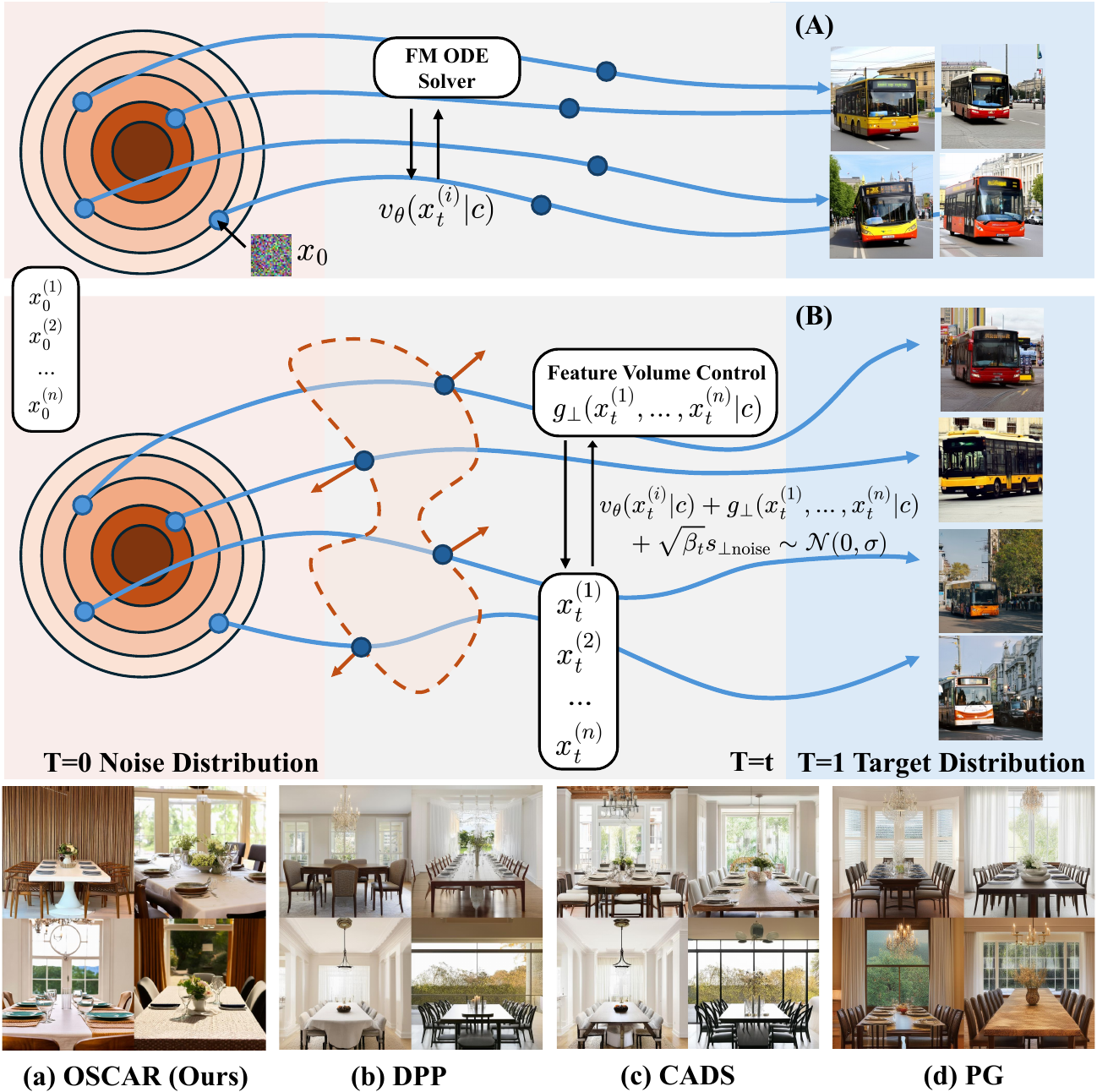}
\caption{
    \textbf{Up}: A conceptual comparison of the generation process. \textbf{(A)} The standard flow matching shows independent trajectories collapsing to similar modes. \textbf{(B)} \textbf{OSCAR} (ours) introduces an orthogonal control mechanism that forces the interacting trajectories to diverge and cover a wider semantic space. 
    \textbf{Down}: A qualitative comparison of generated images against strong baselines. The combined results illustrate that our method significantly increases the diversity of generations while retaining high output quality.
}
\vspace{-1em}
\label{fig:my_flow}
\end{figure}

This challenge manifests as an  ``illusion of variety”: while users can generate numerous images by changing random seeds, the outputs often collapse to a few high-probability modes, exploring only a narrow slice of the concept’s true semantic space. This tendency is not merely an incidental artifact. Still, it is exacerbated by the field's prevailing focus on aligning models with human preferences, a process that often implicitly rewards outputs that conform to common expectations at the expense of novelty \citep{hemmat2023feedback, hall2023dig}. This limitation is then deeply rooted in the models' mechanics, as the learned flows are often contractive, pulling different initial trajectories toward similar high-density modes of the target distribution. This effect is further amplified by strong Classifier-Free Guidance (CFG), which over-weights the conditional score, narrowing the explored region of the conditional manifold, improving prompt fidelity but reducing semantic diversity \citep{ho2022classifier}. Naively drawing more samples is an inefficient remedy, as the required Number of Function Evaluations (NFE) to find rare modes is prohibitive under finite computational budgets.

Existing approaches to mitigate this issue, whether applied during training or inference, often force a difficult compromise. Training-time solutions are often sensitive to the specific dataset and training parameters, and are typically computationally prohibitive, rendering them inapplicable to pre-trained models. While more practical, training-free methods for enhancing diversity, whether by modifying the sampling process or directly augmenting the generation dynamics, typically trade off sample quality for diversity, often producing artifacts. Consequently, a low-overhead, training-free method that can enhance diversity while rigorously preserving quality remains a key desideratum.

To address these limitations, we introduce \textbf{O}rthogonal \textbf{S}tochastic \textbf{C}ontrol for \textbf{A}lignment-\textbf{R}especting diversity (OSCAR), a novel, training-free control mechanism for continuous-time flow-matching inference. Our core idea is to reshape the sampling dynamics to encourage trajectories under the same conditions to naturally fan out toward complementary semantics. At each time step, we first employ a finite-difference endpoint extrapolation, inspired by Heun \citep{suli2003introduction}, to predict a local endpoint for the current state (see App.~\ref{sec:appendix_details} for details). We then maximize the feature-space volume of these predicted endpoints, defined via the log-determinant of their centered Gram matrix. The gradient of this volume potential is efficiently pulled back to the latent space. To further encourage exploration, we complement this deterministic guidance with a controllable, time-scheduled stochastic noise. 


The key to our method's success lies in a unifying geometric principle, which can be summarized as follows.
\textbf{(I) Orthogonal guidance.}
Both the deterministic control signal and the stochastic noise are projected to be orthogonal to the base flow velocity. This design ensures that our guidance provides only a ``lateral'' push for diversity, without opposing the model's forward, quality-seeking dynamics.
\textbf{(II) Resolved fidelity--diversity trade-off.}
Through comprehensive numerical comparisons against strong baselines, we demonstrate that this principle effectively resolves the trade-off between output diversity and generation quality in conditional flow matching.
\textbf{(III) Consistent empirical gains.}
As a result, our method achieves superior performance on diversity-centric metrics, including mode coverage and intra-class entropy, while consistently preserving generation quality.
The overall process and its effect on sampling trajectories are conceptually illustrated in Figure~\ref{fig:my_flow}.

\section{Related work}
\label{sec:related_work}

\noindent\textbf{From Diffusion to Continuous-Time Generative Flows.}
The development of modern generative models began with Denoising Diffusion Probabilistic Models (DDPM) \citep{ho2020denoising}, which achieved stable training and strong fidelity by simulating a progressive denoising process. A key drawback of the canonical sampler, however, is its high computational cost, requiring a large neural function evaluations (NFE) to produce a high-quality image. This limitation spurred two complementary research streams. On the quality and alignment side, researchers steadily improved realism and prompt adherence by scaling model capacity, strengthening text encoders, and incorporating preference alignment objectives \citep{meng2021sdedit, saharia2022photorealistic, esser2024scaling, xu2023imagereward}. On the efficiency side, few-step Ordinary Differential Equation (ODE) solvers and consistency-based models substantially reduced the NFE while preserving fidelity \citep{song2020denoising, lu2022dpm, song2023consistency, luo2023latent}. While continuous-time generative frameworks like Flow Matching (FM) and Rectified Flow (RF) have greatly improved the efficiency and fidelity of modern generative models \citep{lipman2022flow, liu2022flow}, the field's primary focus has remained on these aspects. Consequently, diversity and mode coverage have been overlooked, particularly under strong CFG, which compresses the solution space and causes generations to collapse to similar high-probability modes \citep{ho2022classifier, sadat2023cads}.

\noindent\textbf{Training-Time Diversity Enhancement.}
One line of work enhances diversity during training, often inspired by reinforcement learning. For instance, \citet{miao2024training} proposes a framework that employs a reward model to score the diversity of generated images. However, this method's effectiveness hinges on a pre-defined reference distribution of real images, a requirement that can be impractical to satisfy and limits the approach's generality. Other methods frame the reverse sampling process as a multi-step Markov Decision Process and apply policy gradient optimization to the entire sampling trajectory \citep{zhang2024large, black2023training, mcallister2025flow}. While powerful, this web-scale training paradigm is not only computationally prohibitive, but its effectiveness is also highly sensitive to the specific dataset and training parameters. Our work is therefore situated within the more practical and versatile paradigm of training-free, inference-time enhancement.

\noindent\textbf{Inference-Time Diversity: Sampling Strategies vs. Gradient-Based Guidance.}
Inference-time methods for diversity enhancement can be broadly categorized into two paradigms: sampling strategies and gradient-based guidance. The first class indirectly broadens exploration by modifying the sampling process. For instance, \citet{sehwag2022generating} proposes biasing sampling toward low-density regions of the data manifold, which may lead to a distribution shift. Building on this idea, Condition-annealed diffusion sampling (CADS) dynamically anneals the conditioning signal to encourage exploration in early sampling stages \citep{sadat2023cads}. Unlike the scheduling-based CADS, whose performance can be highly dependent on the precise schedule and scale of the injected noise, our method introduces noise strictly in the subspace orthogonal to the base flow velocity. This quality-preserving design makes our method far less dependent on the injection schedule and thus more robust across CFG scales. Another related method, adaptive projected guidance (APG)~\citep{sadat2024eliminating}, also uses an orthogonal projection form, but for a different purpose: it decomposes the classifier-free guidance update to reduce oversaturation and improve quality under strong guidance. In contrast, OSCAR derives its guidance from a set-level feature-volume objective over predicted endpoints, and uses orthogonal projection as a fidelity safeguard to prevent the diversity-inducing control from interfering with the base flow direction.

A second paradigm, gradient-based guidance, takes a more direct approach by actively pushing samples apart using diversity objectives. While early methods, such as Particle Guidance (PG), applied repulsion directly in the latent space, this can conflict with the model's quality-generating flow \citep{corso2023particle}. A significant evolution is DiverseFlow, which elevates the guidance objective to a semantic feature space operating on predicted endpoints \citep{morshed2025diverseflow}. Our work builds on this idea by introducing a more efficient $O(K^2)$ objective that largely reduces DiverseFlow's $O(K^3)$ complexity cost, and, more critically, a set of rigorous geometric constraints for fidelity preservation. Orthogonal to these deterministic guidance methods, another line of thought replaces the ODE solver with a Stochastic Differential Equation (SDE) solver inspired by diffusion models \citep{liu2025flow}. However, such generic stochasticity is not explicitly designed to enhance semantic diversity. Our method is unique in that it integrates the strengths of both approaches: we utilize principled semantic guidance while also injecting staged, controllable noise, both of which are constrained to be orthogonal to the base flow, ensuring a safe and effective diversity boost.
\section{Problem Formulation}

Our work is situated within the framework of continuous-time generative models, particularly those based on FM and RF. Specifically, the RF framework defines the path between a real data sample $x_1$ and a noise sample $x_0$ via a simple linear interpolation $x_t = (1-t)x_1 + t x_0$, for $t \in [0, 1]$. A model is then trained to learn a velocity field $v_\theta(x_t, t \mid c)$ by minimizing an objective function:
$$
    \mathcal{L}(\theta) = \mathbb{E}_{t, x_0, x_1} \left[ \left\| (x_0 - x_1) - v_\theta(x_t, t \mid c) \right\|^2 \right]
$$
This process drives $v_\theta$ to fit the constant target velocity field connecting the data and noise distributions. During inference, this learned velocity field $v_\theta$ guides a deterministic ODE from noise to data: $\frac{d x_t}{dt} = v_\theta(x_t,t \mid c), t\in[1,0]$.

However, the deterministic and often contractive nature of the inference ODE suffers from a significant drawback in the form of a striking lack of set-level diversity. To empirically characterize the lack of set-level diversity in deterministic ODEs, we compare standard FM and OSCAR on a 2D ring-structured GMM. Initial particles, sampled from an inner concentric disk, are evolved under a fixed NFE budget. As illustrated in Figure~\ref{fig:toy_example}, FM induces purely radial contraction, causing trajectories to cluster within sparse angular sectors and fail to explore more modes. In contrast, OSCAR preserves radial convergence while introducing a tangential diversification force that drives particles to spread uniformly across all modes. Consequently, OSCAR achieves superior mode coverage and a balanced angular distribution, demonstrating its capacity to mitigate mode collapse at no additional computational cost.

\begin{figure}[h]

    \includegraphics[width=1\linewidth]{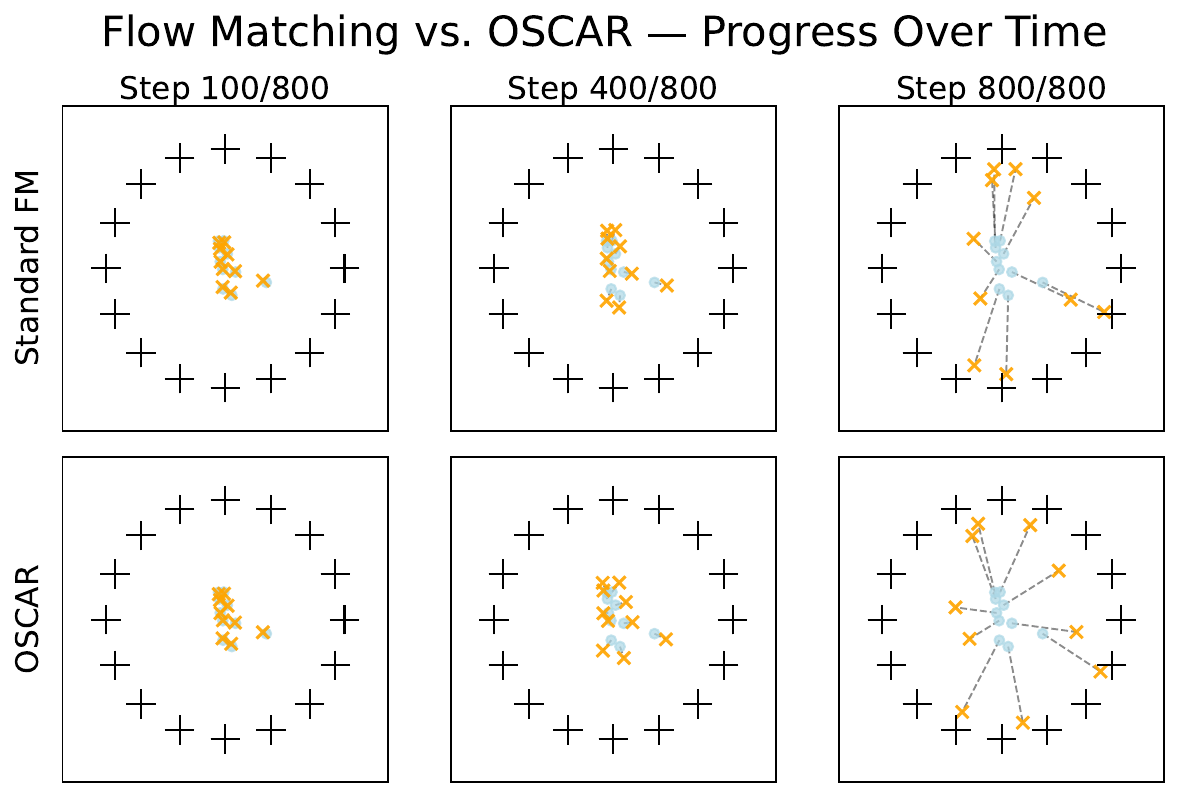}
    \caption{
        \textbf{Toy example on a ring of Gaussian modes illustrating mode exploration.} 
        Rows: Standard Flow Matching vs. OSCAR. Columns: snapshots at early, mid, and final generation steps under the same NFE. 
        Black plus signs (+) denote mode centers; light-blue circles represent initial particles; orange crosses ($\times$) mark particle locations, with dashed lines tracing individual trajectories. 
        While standard FM exhibits significant mode collapse due to purely contractive dynamics, OSCAR promotes \textbf{tangential diversification}, enabling particles to cover more modes earlier and more uniformly within the same sampling budget.
    }
    \label{fig:toy_example}
\end{figure}

The central challenge in overcoming this limitation lies in the well-known trade-off between diversity and quality. Naively augmenting the learned ODE dynamics, for instance, by adding arbitrary noise or simple repulsion forces, can easily disrupt the model's carefully calibrated path to high-fidelity samples, often introducing artifacts and degrading generation quality. This motivates the search for diversity-promoting interventions that encourage trajectories to spread without interfering with the model’s forward, quality-seeking dynamics.

\section{Methodology}

In this work, we evaluate set-level diversity, defined as the ability of a model to generate a semantically varied set of a given size for a fixed condition, purely by changing the initial random seed. Therefore, our goal is to enhance set-level diversity by augmenting the base dynamics without retraining the pretrained velocity field $v_\theta$. We achieve this by introducing two components: (1) a deterministic, geometry-aware control signal $g(x_t, t)$, designed to drive samples apart in a semantic feature space, and (2) a controllable stochastic noise term $\sqrt{\beta(t)}\, dW_t$ to inject quality-preserving uncertainty. This transforms the original ODE into a controlled SDE that governs our sampler:
\begin{equation}
    \label{eq:controlled_sde}
    dx_t = \big[v_\theta(x_t,t \mid c) - g(x_t,t \mid c)\big]\,dt + \sqrt{\beta(t)}\,dW_t,
\end{equation}
where $W_t$ is standard Brownian motion. The principled design of the control signal $g$ and its associated fidelity safeguards are detailed in the following sections.


\subsection{Diversity via Trajectory Spreading in Feature Space}
\label{sec:method-core}
To drive the set of trajectories $\{x^{(i)}_t\}$ towards diverse configurations, our control mechanism operates not on the current states, but on their local estimates of the final trajectory endpoint $\{x^{(i)}_0\}$ in a semantic feature space. The core idea is to define a differentiable, permutation-invariant set energy \footnote{The term ``energy" is used by analogy to physics: we define an objective function where lower values correspond to more diverse sets, and its negative gradient provides a 'force' that drives diversity.} $\mathcal{E}_s(Z)$ for a collection of endpoint features $Z = [z_1;\ldots;z_m]$. This energy is designed to be low when features are spread out and high when they are clustered. We then steer the features towards a lower energy state using a stochastic control law:
\begin{equation}
\label{eq:feature_sde}
dZ = -\gamma(t)\,\nabla_Z \mathcal{E}_s(Z)\,dt + \sqrt{\beta(t)}\,dW_t,
\end{equation}
where $Z \in \mathbb{R}^{m\times D}$ stacks the features, and schedules $\gamma(t)$ and $\beta(t)$ modulate the diversity-seeking drift and exploration. The key components of this control law are derived as follows.

\noindent\textbf{Endpoint-Aware Features.} At any time $t$, we form an endpoint-aware feature by first applying a one-step predictor $\hat\psi(x,t)$ and then a semantic encoder $\phi:\mathcal{X}\to\mathbb{R}^D$ (e.g., a pretrained image tower):
{\small
\[
z_i(t)=\phi(\hat\psi(x^{(i)}(t),t)),\qquad
Z(t)=[z_1(t);\ldots;z_m(t)].
\]
}

For the predictor $\hat\psi$, we use a finite-difference endpoint extrapolation, which provides an accurate local estimate of the trajectory's endpoint without any extra NFE.

\noindent\textbf{Set Energy as Feature Volume.} We instantiate the set energy $\mathcal{E}_s(Z)$ using a log-determinant objective that corresponds to the volume spanned by the features. To avoid manual centering, we use a trace-stabilized form:
\begin{equation}
\label{eq:logdet_energy}
\mathcal{E}_s(Z) = -\tfrac{1}{2}\,\log\det\!\Big(I + \tau\, ZZ^\top \Big),
\end{equation}
where we fix $\tau=1$ in practice across all experiments. Minimizing this energy is equivalent to maximizing the regularized volume of the parallelepiped formed by the feature vectors $\{z_i\}$.

\noindent\textbf{Pulling the Gradient Back to State Space.} The feature-space drift from Eq.\ref{eq:feature_sde} is mapped back to the sampler's state space $\mathcal{X}$ to form the control signal $g(x,t)$. We achieve this via the chain rule:
\begin{equation}
\label{eq:pullback}
\resizebox{0.9\linewidth}{!}{$ 
    g_i(x_i,t) = \Big(J_{x}\hat\psi(x_i,t)\Big)^\top \Big(J_{u}\phi(u)\Big)^\top \Big[ \nabla_Z \mathcal{E}_s(Z) \Big]_i \Big|_{u=\hat\psi(x_i,t)}
$}
\end{equation}
where $[ \cdot ]_i$ denotes the gradient component for the $i$-th sample. Equation~\ref{eq:pullback}, which defines the control signal $g(x,t)$ used in our final controlled SDE, is obtained by pulling back the feature-space gradient via the chain rule and implemented efficiently per sample using two reverse-mode vector–Jacobian products (VJPs) without explicitly forming Jacobians (see full derivation in Appendix~\ref{sec:appendix_theory}, Lemma~1).

\subsection{Fidelity Safeguards}
\label{sec:quality_safeguards}
A powerful diversity signal is not sufficient. It must be applied in a manner that respects the underlying quality manifold learned by the pretrained model. To this end, we introduce two principled, geometry-aware constraints that ensure our control mechanism is quality-preserving.

\begin{algorithm}[h]
\caption{\textbf{The OSCAR Sampling Algorithm}}
\label{alg:oscar_sampler}
\small
\begin{algorithmic}[1]
\STATE \textbf{Input:} Initial noise samples $\{x^{(i)}_0\}_{i=1}^m \sim \mathcal{N}(0,I)$; condition $c$; models: velocity field model $v_\theta$, feature encoder $\phi$, endpoint predictor $\hat\psi$; time grid $0=t_0 < \cdots < t_T=1$.
\STATE \textbf{Hyperparameters:} Stability $\varepsilon$, volume scale $\tau$, orthogonality $\lambda\!\in\![0,1]$, leverage exponent $\alpha\!\in\![0.5,1]$, noise exponent $p$.
\STATE \textbf{Output:} Final samples $\{x^{(i)}_1\}_{i=1}^m$.
\vspace{0.3em}

\FOR{$\ell=0$ \textbf{to} $T-1$}
    \STATE Let $x_i \leftarrow x^{(i)}_{t_\ell}$ be the current sample at time $t_\ell$.
    \STATE Set step size $\Delta t \leftarrow t_{\ell+1} - t_\ell$; evaluate base velocity $v_i \leftarrow v_\theta(x_i, t_\ell \mid c)$.
    \STATE \textbf{Predict endpoints:} Use extrapolation $\hat\psi$ to get $\hat x^{\mathrm{end}}_i$ from $(x_i, v_i, t_\ell)$.
    \STATE \textbf{Compute diversity gradient $g_i$ (Sec.~\ref{sec:method-core}):}
    \STATE \quad Encode endpoints $z_i \leftarrow \phi(\hat x^{\mathrm{end}}_i)$, stack into matrix $Z=[z_1;\ldots;z_m]$
    \STATE \quad Compute leverage scores $s_i$ and weights $w_i$ from Gram matrix $ZZ^\top$ (Eq.~\ref{eq:reweighted_gram}). 
    \STATE \quad Compute reweighted feature-space gradient $G^Z = \nabla_Z\mathcal{E}_s(Z)$.
    \STATE \quad Pull back to state space: $g_i \leftarrow (J\hat\psi)^\top (J\phi)^\top [G^Z]_i$ (Eq.~\ref{eq:pullback}).
    \vspace{0.3em}
    \STATE \textbf{Apply fidelity safeguards(Sec.~\ref{sec:quality_safeguards}):}
    \STATE \quad Project determin control: $g_i \leftarrow \left( g_i - \lambda\,\frac{\langle g_i,v_i\rangle}{\|v_i\|^2+\varepsilon}\,v_i \right) \times \min\left(1, \frac{\|v_i\|}{\|g_i\|+\varepsilon}\right)$. 
    \STATE \quad Project stochastic noise: $\xi_i \sim \mathcal N(0,I)$; $\xi_i^\perp \leftarrow \xi_i-\frac{\langle \xi_i,v_i\rangle}{\|v_i\|^2+\varepsilon}\,v_i$; $\eta_i \leftarrow \sqrt{t_\ell^p |\Delta t|}\,\xi_i^\perp$.
    \vspace{0.3em}
    \STATE \textbf{Perform controlled Step (App.~\ref{sec:extrapolation}):} 
    \STATE \quad Effective velocity: $v^{\mathrm{eff}}_i \leftarrow v_i - \gamma(t_\ell)\, g_i$.
    \STATE \quad Update: $x^{(i)}_{t_{\ell+1}} \leftarrow x_i + \Delta t\, v^{\mathrm{eff}}_i + \eta_i$.

\ENDFOR
\STATE Let final samples be $x^{(i)}_1 \leftarrow x^{(i)}_{t_T}$.
\end{algorithmic}
\end{algorithm}

\subsubsection{Orthogonal Projection}
\label{sec:safeguards_OP}
Simply adding the control signal $g(x_t, t)$ to the base velocity $v_\theta$ can create conflicts, as the direction for maximal diversity may oppose the direction for maximal fidelity. To resolve this, we enforce that the control operates strictly in a subspace that is orthogonal to the primary alignment/quality direction. We define this direction as the base velocity itself, $q(x,t) \approx v_\theta(x,t \mid c)$, which empirically points towards the conditional data manifold.

We then construct a projection operator $\Pi_\perp$ that nullifies any component of our control signal parallel to this quality direction:
\begingroup
\setlength{\abovedisplayskip}{6pt}
\setlength{\belowdisplayskip}{6pt}
\[
g_i^{\perp}(t)=g_i(t)-\lambda \tfrac{g_i(t)^\top v_i(t)}{\|v_i(t)\|^2+\delta}\,v_i(t).
\]
\endgroup

where $\lambda \in [0, 1]$ controls the strictness of the orthogonality. By applying this projection to both the deterministic control $g$ and the stochastic noise $dW_t$, we ensure our guidance only performs a ``lateral spread" to enhance diversity, without interfering with the model's primary, ``forward" trajectory towards a high-fidelity output (See theoretical analysis in Appendix~\ref{sec:appendix_theory}).

\subsubsection{Redundancy-Aware Reweighting}
\label{sec:safeguards_RR}
Instead of applying a uniform, hard cap to the control signal, we employ a more adaptive and fine-grained mechanism to modulate its strength on a per-sample basis. This is achieved by reweighting the set energy $\mathcal{E}_s$ based on the geometric arrangement of the endpoint features. Specifically, we compute per-sample leverage scores $s_i$ from the inverse of the feature Gram matrix $K = ZZ^\top$: $s_i = \big[(K + \varepsilon I)^{-1}\big]_{ii}.$ The score $s_i$ is inversely proportional to how ``central" or ``redundant" a sample is within the set. Geometrically, under-covered samples receive higher scores. These scores are then used to form normalized weights $w_i$, which reweight the Gram matrix used in our volume objective:
\begin{equation}
\small
    \label{eq:reweighted_gram}
    \tilde{K} = W^{1/2} K W^{1/2}, \quad \text{where} \quad W = \mathrm{diag}(w_1,\ldots,w_m).
\end{equation}
where $w_i = \frac{s_i^{\alpha}}{\sum_j s_j^{\alpha}}$. By using this reweighted Gram matrix $\tilde{K}$ inside the log-determinant energy (Eq.~\ref{eq:logdet_energy}), our method naturally allocates more guidance strength to the under-represented samples that need it most, while attenuating the influence of dominant, redundant samples. This ``push weak, not strong" principle provides a more stable and efficient control than a uniform cap. The complete, step-by-step implementation of our guided sampler is presented in Algorithm~\ref{alg:oscar_sampler}.

\section{Experiments and Results}
\label{sec:experiments}

\noindent\textbf{Setup.}
We evaluate three conditional generation settings:
\textbf{(i)} class-conditional generation on COCO \citep{lin2014microsoft};
\textbf{(ii)} text-to-image (T2I) using a pretrained Stable Diffusion backbone in Flow-matching framework;
and \textbf{(iii)} attribute-conditioned generation, where we assess the quality and diversity of images generated for specific attributes.  All methods are applied at inference time to the same pretrained and frozen backbone model, ensuring a fair comparison. More implementation details can be found in Appendix \ref{sec:appendix_details}.

\noindent\textbf{Baselines.}
We compare OSCAR with several training-free inference-time baselines. 
FM-SD3.5 denotes the original frozen flow-matching backbone. 
CADS~\citep{sadat2023cads} promotes diversity by annealing the conditioning signal, while Particle Guidance (PG)~\citep{corso2023particle} introduces repulsive interactions among samples. 
We also compare with DiverseFlow~\citep{morshed2025diverseflow}, which uses a DPP-inspired endpoint guidance objective, so we denote this baseline as DPP. 
Finally, we include APG~\citep{sadat2024eliminating}, an orthogonal-guidance method for improving classifier-free guidance. 

\noindent\textbf{Fidelity and Alignment.} To quantify performance across these setups, our evaluation is structured around three core axes. We first assess the fidelity and alignment of generated samples. This includes the standard Fréchet Inception Distance (FID)~\citep{heusel2017gans} and Kernel Inception Distance (KID)~\citep{binkowski2018demystifying} for measuring the discrepancy between generated samples and the reference distribution. For semantic and human-preference alignment, we measure CLIP-Score for text-prompt agreement and Image Reward~\citep{xu2023imagereward} for alignment with human aesthetic preferences. Finally, to assess perceptual quality and detect artifacts, we employ the no-reference Blind/Referenceless Image Spatial Quality Evaluator (BRISQUE)~\citep{mittal2012no} as well as the learned CLIP-IQA~\citep{wang2023exploring} metric.

\noindent\textbf{Distributional Coverage.}
Beyond the quality of individual images, we measure distributional coverage to understand how well the model captures the breadth of the true data manifold. We plot Improved Precision-Recall (IPR) curves~\citep{kynkaanniemi2019improved} and report the iso-precision $\Delta$Recall, which directly quantifies coverage expansion at a comparable fidelity level. This is complemented by Coverage@$\tau$ on pre-clustered real-data features to evaluate discrete mode discovery. The Does-It Metric (DIM)~\citep{teotia2025dimcim} further probes coverage by measuring the balanced generation of attributes from coarse prompts.

\begin{table*}[t]
\centering
\caption{ \small
    Quantitative comparison of our method against baselines across different CFG levels for the ``truck" concept, you can see other concept results in Appendix~\ref{sec:appendix_main_results}. Our method consistently improves diversity metrics while achieving state-of-the-art performance on quality and alignment scores.
}
\label{tab:all_metrics_comparison}
\footnotesize
\setlength{\tabcolsep}{5pt}
\renewcommand{\arraystretch}{0.65}
\setlength{\aboverulesep}{2pt}
\setlength{\belowrulesep}{2pt}
\begin{tabularx}{\textwidth}{l@{\extracolsep{\fill}}cccccc}
\toprule
& \multicolumn{3}{c}{\textbf{Guidance Scale (CFG)}} & \multicolumn{3}{c}{\textbf{Guidance Scale (CFG)}} \\
\cmidrule(lr){2-4}\cmidrule(lr){5-7}
\textbf{Method} & $3.0$ & $5.0$ & $7.5$ & $3.0$ & $5.0$ & $7.5$ \\
\midrule

& \multicolumn{3}{c}{\textbf{Brisque} $\downarrow$} & \multicolumn{3}{c}{$\mathbf{1-\text{MS-SSIM}}$(\%) $\uparrow$} \\
\cmidrule(lr){2-4}\cmidrule(lr){5-7}
FM-SD3.5  & 19.58 $\pm$ 1.95 & 23.40 $\pm$ 2.27 & 33.11 $\pm$ 1.82 & 83.93 $\pm$ 2.05 & 82.62 $\pm$ 2.63 & 80.46 $\pm$ 2.82 \\
PG        & 40.11 $\pm$ 5.83 & 40.34 $\pm$ 6.11 & 47.66 $\pm$ 3.92 & 84.60 $\pm$ 2.76 & 83.76 $\pm$ 2.37 & 81.07 $\pm$ 1.22 \\
CADS      & 21.15 $\pm$ 0.76 & 23.97 $\pm$ 1.59 & 32.73 $\pm$ 1.96 & \textbf{86.12 $\pm$ 2.76} & 83.82 $\pm$ 3.10 & 82.71 $\pm$ 2.53 \\
DPP       & 18.73 $\pm$ 0.91 & 24.51 $\pm$ 1.98 & 33.88 $\pm$ 0.98 & 85.75 $\pm$ 1.32 & 83.42 $\pm$ 1.86 & 81.90 $\pm$ 1.73 \\
APG       & 24.89 $\pm$ 2.23 & 27.97 $\pm$ 2.04 & 30.09 $\pm$ 2.20 & 85.16 $\pm$ 2.14 & 83.53 $\pm$ 1.81 & 82.93 $\pm$ 2.55 \\
\rowcolor{lightblue}\textbf{Ours} & \textbf{18.50 $\pm$ 1.62} & \textbf{21.72 $\pm$ 1.31} & \textbf{27.80 $\pm$ 0.85} & 84.80 $\pm$ 1.03 & \textbf{83.91 $\pm$ 1.04} & \textbf{83.22 $\pm$ 1.21} \\
\midrule

& \multicolumn{3}{c}{\textbf{Vendi Score Pixel} $\uparrow$} & \multicolumn{3}{c}{\textbf{Vendi Score Inception} $\uparrow$} \\
\cmidrule(lr){2-4}\cmidrule(lr){5-7}
FM-SD3.5  & 2.48 $\pm$ 0.19 & 2.31 $\pm$ 0.11 & 2.23 $\pm$ 0.11 & 4.55 $\pm$ 0.28 & 3.86 $\pm$ 0.13 & 3.76 $\pm$ 0.14 \\
PG        & 2.45 $\pm$ 0.18 & 2.30 $\pm$ 0.13 & 2.25 $\pm$ 0.08 & 4.80 $\pm$ 0.18 & 4.22 $\pm$ 0.14 & 3.81 $\pm$ 0.11 \\
CADS      & 2.63 $\pm$ 0.09 & 2.34 $\pm$ 0.12 & 2.25 $\pm$ 0.10 & 4.58 $\pm$ 0.28 & 3.88 $\pm$ 0.14 & 3.71 $\pm$ 0.14 \\
DPP       & 2.48 $\pm$ 0.09 & 2.30 $\pm$ 0.09 & 2.17 $\pm$ 0.03 & 4.24 $\pm$ 0.32 & 3.58 $\pm$ 0.23 & 3.15 $\pm$ 0.10 \\
APG       & 2.61 $\pm$ 0.28 & 2.40 $\pm$ 0.30 & 2.31 $\pm$ 0.30 & 4.76 $\pm$ 0.67 & 4.27 $\pm$ 0.40 & 4.05 $\pm$ 0.29 \\
\rowcolor{lightblue}\textbf{Ours} & \textbf{2.66 $\pm$ 0.18} & \textbf{2.47 $\pm$ 0.11} & \textbf{2.34 $\pm$ 0.10} & \textbf{4.93 $\pm$ 0.33} & \textbf{4.34 $\pm$ 0.13} & \textbf{4.08 $\pm$ 0.04} \\
\midrule

& \multicolumn{3}{c}{\textbf{CLIP-IQA} $\uparrow$} & \multicolumn{3}{c}{\textbf{CLIP Score} $\uparrow$} \\
\cmidrule(lr){2-4}\cmidrule(lr){5-7}
FM-SD3.5  & 6.26 $\pm$ 0.61 & 6.48 $\pm$ 0.44 & 6.52 $\pm$ 0.48 & 26.97 $\pm$ 0.58 & 26.16 $\pm$ 0.52 & 26.02 $\pm$ 0.54 \\
PG        & 6.22 $\pm$ 0.57 & 6.43 $\pm$ 0.47 & 6.48 $\pm$ 0.42 & 27.15 $\pm$ 0.61 & 26.57 $\pm$ 0.55 & 26.52 $\pm$ 0.53 \\
CADS      & \textbf{6.61 $\pm$ 0.57} & 6.65 $\pm$ 0.51 & 6.69 $\pm$ 0.46 & 26.89 $\pm$ 0.68 & 26.53 $\pm$ 0.62 & 26.55 $\pm$ 0.53 \\
DPP       & 6.29 $\pm$ 0.58 & 6.42 $\pm$ 0.50 & 6.44 $\pm$ 0.46 & 27.06 $\pm$ 0.78 & 26.68 $\pm$ 0.67 & 26.46 $\pm$ 0.59 \\
APG       & 6.49 $\pm$ 0.65 & 6.56 $\pm$ 0.49 & 6.62 $\pm$ 0.45 & \textbf{27.30 $\pm$ 0.55} & 26.73 $\pm$ 0.50 & \textbf{26.67 $\pm$ 0.47} \\
\rowcolor{lightblue}\textbf{Ours} & 6.57 $\pm$ 0.62 & \textbf{6.76 $\pm$ 0.51} & \textbf{6.79 $\pm$ 0.48} & 27.20 $\pm$ 0.55 & \textbf{26.76 $\pm$ 0.55} & 26.56 $\pm$ 0.48 \\
\midrule

& \multicolumn{3}{c}{\textbf{FID} $\downarrow$} & \multicolumn{3}{c}{\textbf{Image Reward} $\uparrow$} \\
\cmidrule(lr){2-4}\cmidrule(lr){5-7}
FM-SD3.5  & 166.18 $\pm$ 1.10 & 165.51 $\pm$ 0.91 & 166.08 $\pm$ 0.65 & 0.40 $\pm$ 0.32 & 0.52 $\pm$ 0.26 & 0.60 $\pm$ 0.27 \\
PG        & 167.34 $\pm$ 2.12 & 166.50 $\pm$ 2.39 & 165.39 $\pm$ 2.03 & 0.38 $\pm$ 0.38 & 0.54 $\pm$ 0.30 & 0.60 $\pm$ 0.29 \\
CADS      & 165.84 $\pm$ 0.95 & 166.13 $\pm$ 1.22 & 167.34 $\pm$ 1.03 & \textbf{0.45 $\pm$ 0.32} & 0.56 $\pm$ 0.27 & 0.64 $\pm$ 0.26 \\
DPP       & 166.51 $\pm$ 1.52 & 165.89 $\pm$ 0.72 & 166.21 $\pm$ 1.28 & 0.34 $\pm$ 0.35 & 0.49 $\pm$ 0.28 & 0.55 $\pm$ 0.27 \\
APG       & 166.62 $\pm$ 1.11 & 165.55 $\pm$ 1.48 & 165.77 $\pm$ 1.61 & 0.42 $\pm$  0.33 & 0.54 $\pm$ 0.27 & 0.62 $\pm$ 0.28 \\
\rowcolor{lightblue}\textbf{Ours} & \textbf{165.40 $\pm$ 0.94} & \textbf{164.75 $\pm$ 0.66} & \textbf{165.31 $\pm$ 0.79} & 0.43 $\pm$ 0.35 & \textbf{0.57 $\pm$ 0.29} & \textbf{0.65 $\pm$ 0.28} \\
\bottomrule
\end{tabularx}
\end{table*}

\noindent\textbf{Intra-Set Diversity.} 
Finally, for the crucial user-facing scenario of generating multiple candidates from a single prompt, we evaluate intra-set diversity. We report pairwise $1 - \text{MS-SSIM}$~\citep{wang2003multiscale} to capture perceptual and structural variation. Our primary metric here is the reference-free Vendi Score~\citep{friedman2022vendi}, which reflects the ``effective number" of distinct items in a set. We report both Inception and pixels representations to disentangle low-level from semantic diversity. The Can-It Metric (CIM)~\citep{teotia2025dimcim} confirms that this diversity is controllable, assessing the ability to produce distinct outputs when explicitly conditioned on different attributes.


\subsection{Main Results}

The following sections describe our main findings, and additional experiments are provided in Appendix~\ref{sec:appendix_additional_exp}.

\noindent\textbf{Fidelity-Coverage Trade-off.}
Our quantitative evaluation shows that our method improves the trade-off between generation fidelity and diversity, which is illustrated by the Precision-Recall for Distributions (PRD) curves in Figure~\ref{fig:prd_suitcase}. While baselines such as DPP and PG suffer from a sharp precision collapse as recall begins to increase, our method maintains a more robust curve, demonstrating a superior ability to expand sample diversity without a severe penalty to fidelity. This overall advantage is confirmed by our method achieving the highest Area Under the Curve (AUC) across all CFG scales. This result underscores our method's ability to effectively explore the data manifold for diverse samples while maintaining high-quality outputs. As further demonstrated by extensive PRD results in Appendix~\ref{sec:appendix_prd}, our method outperforms all baselines on the majority of evaluated concepts, indicating that the observed improvement is not limited to a specific class.

\begin{figure}[h]
  \centering
  \includegraphics[width=1\linewidth]{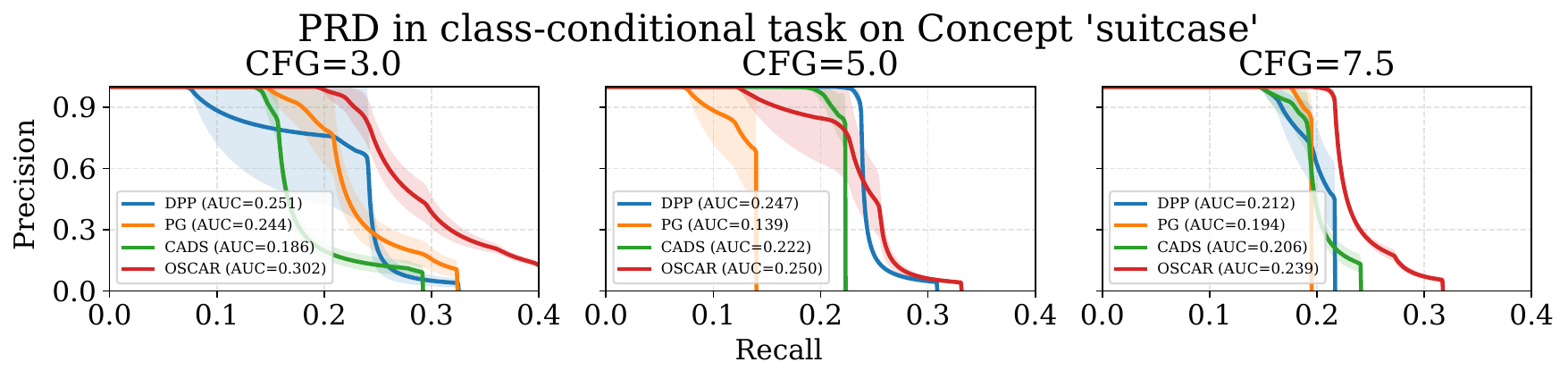}
  \caption{
PRD curves on the ``suitcase'' concept, comparing methods across different CFG levels. For most CFG settings, our method’s curves are shifted toward the top-right, indicating a superior precision–recall trade-off.}
  \label{fig:prd_suitcase}
  \vspace{-1em}
\end{figure}

\begin{figure}[h]
  \centering

  \includegraphics[width=0.8\linewidth]{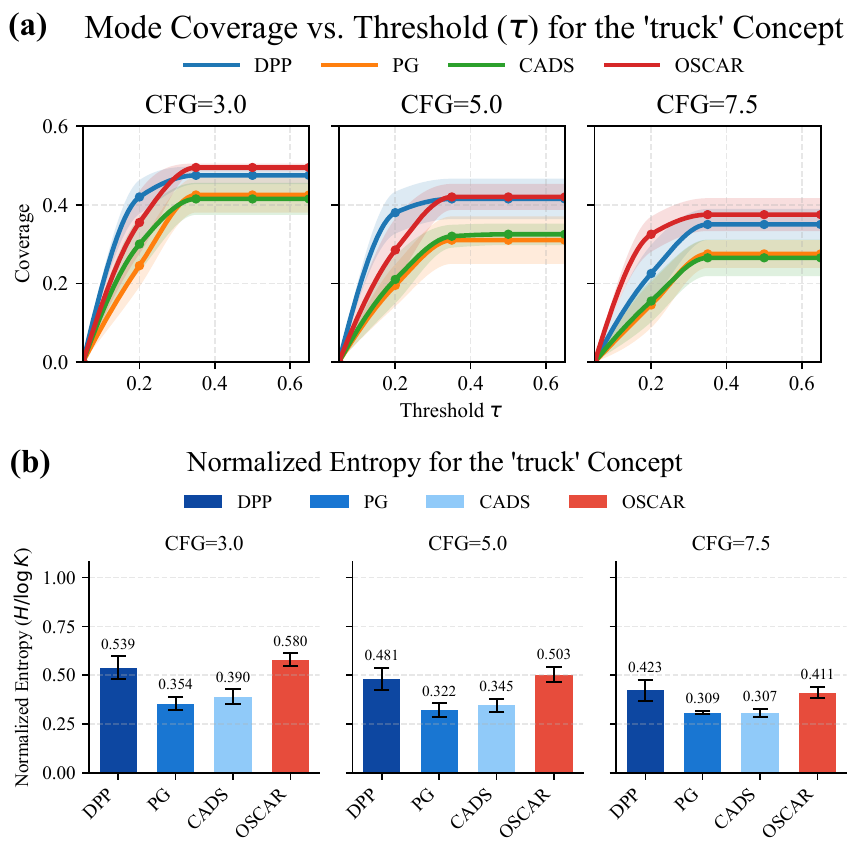}

  \caption{ 
  \textbf{(a)} Mode Coverage vs. Threshold ($\tau$) for the `truck' concept, comparing our method against baselines at three fixed CFG levels. Higher curves indicate a larger fraction of real-data clusters are covered, signifying superior mode discovery. Shaded bands denote the $\pm 1$ standard deviation across multiple seeds and prompts. \textbf{(b)} Normalized entropy for the `truck' concept, comparing methods at different CFG levels. Higher bars indicate a more uniform distribution of samples across discovered modes, signifying less redundancy. Error bars denote the $\pm 1$ standard deviation.}
  \label{fig:exp3}
  \vspace{-1em}
\end{figure}

\noindent\textbf{Reference-Free Diversity and Fidelity.}
As detailed in Table \ref{tab:all_metrics_comparison}, our approach enhances diversity while maintaining state-of-the-art sample quality and prompt adherence. Our method is the clear leader in diversity, achieving the highest scores across all CFG levels for both Vendi Score Pixel($\mathrm{VS}_{p}$), which measures low-level variation, and Vendi Score Inception($\mathrm{VS}_{I}$), which captures semantic differences. Crucially, this diversity is not an artifact of degraded quality. Our method delivers top-tier generation fidelity, evidenced by FID and Brisque scores, while CLIP Scores confirm that text-image alignment is rigorously maintained. To ensure the robustness of these findings, all reported metrics are aggregated over multiple random seeds and synonymous prompts for each concept (see more results and the example of our prompt structure in Appendix~\ref{sec:add_quantitative}).

\begin{table}[t]
    \centering
    \caption{\small
    Quantitative results on the \textit{spatial color} and \textit{complex} subsets of T2I-CompBench.
    }
    \label{tab:t2i_compbench_complex}
    \scriptsize
    \setlength{\tabcolsep}{2.5pt}
    \renewcommand{\arraystretch}{0.85}
    \resizebox{\columnwidth}{!}{%
    \begin{tabular}{@{}lccccc@{}}
        \toprule
        \textbf{Method} &
        \textbf{Vendi (Inc.)} $\uparrow$ &
        \textbf{Vendi (Pixel)} $\uparrow$ &
        \textbf{CLIP} $\uparrow$ &
        $\mathbf{1-\text{MS-SSIM}}$ $\uparrow$ &
        \textbf{KID} $\downarrow$ \\
        \midrule
        CADS      & 3.45 $\pm$ 0.11 & 2.43 $\pm$ 0.07 & 39.04 $\pm$ 0.20 & \textbf{0.891 $\pm$ 0.025} & 28.50 $\pm$ 0.80 \\
        FM-SD3.5  & 3.44 $\pm$ 0.22 & 2.42 $\pm$ 0.13 & 38.96 $\pm$ 0.23 & 0.889 $\pm$ 0.026 & 28.52 $\pm$ 0.78 \\
        DPP       & 3.41 $\pm$ 0.09 & 2.56 $\pm$ 0.09 & 38.96 $\pm$ 0.46 & 0.877 $\pm$ 0.017 & 29.17 $\pm$ 0.77 \\
        PG        & 3.42 $\pm$ 0.15 & 2.46 $\pm$ 0.07 & \textbf{39.17 $\pm$ 0.40} & 0.861 $\pm$ 0.029 & 30.30 $\pm$ 0.72 \\
        APG       & 3.34 $\pm$ 0.10 & 2.54 $\pm$ 0.12 & 39.08 $\pm$ 0.28 & 0.882 $\pm$ 0.027 & 30.34 $\pm$ 0.91 \\
        \rowcolor{lightblue}
        \textbf{OSCAR} & \textbf{3.96 $\pm$ 0.13} & \textbf{2.92 $\pm$ 0.09} & 39.15 $\pm$ 0.30 & 0.890 $\pm$ 0.029 & \textbf{25.15 $\pm$ 0.58} \\
        \bottomrule
    \end{tabular}%
    }
    \vspace{-2.5em}
\end{table}

\noindent\textbf{Generalization to Complex Prompts.}
To further evaluate whether our method remains effective under more challenging text conditions, we conduct additional experiments on the \textit{spatial color} and \textit{complex} subsets of T2I-CompBench~\citep{huang2023t2i}. For each subset, we randomly sample 100 prompts and generate 16 images per prompt, following the same evaluation protocol as in the main experiments. As detailed in Table~\ref{tab:t2i_compbench_complex}, our method achieves the best overall diversity--quality trade-off on these compositional prompts. It consistently improves sample diversity while preserving strong quality and prompt adherence, indicating that the advantage of our approach is not limited to simple prompts. In contrast, although Mix ODE/SDE is competitive on certain diversity measurements, it exhibits a noticeable degradation in quality-related scores. Visual comparisons are provided in Appendix~\ref{sec:appendix_more_visuals}.

\noindent\textbf{Intra-Class Mode Discovery.} The robust performance stems from our method's ability to discover and represent fine-grained, intra-class modes more effectively than baselines. This is demonstrated through a two-fold analysis. First, the mode coverage results in Figure \ref{fig:exp3}(a) show that while our method is competitive at stricter thresholds, it consistently converges to a higher final coverage plateau across all CFG scales. This indicates that our method ultimately identifies a broader set of modes. Furthermore, the normalized entropy results in Figure \ref{fig:exp3}(b) highlight our method's clear superiority in distributing samples among these modes. It achieves a substantially and consistently higher entropy, providing strong evidence that our generated samples are more uniformly distributed, leading to a less redundant and more semantically rich output set (see additional results in Appendix~\ref{sec:appendix_coverage_entropy}).

\begin{figure}
    \centering
    \includegraphics[width=0.7\linewidth]{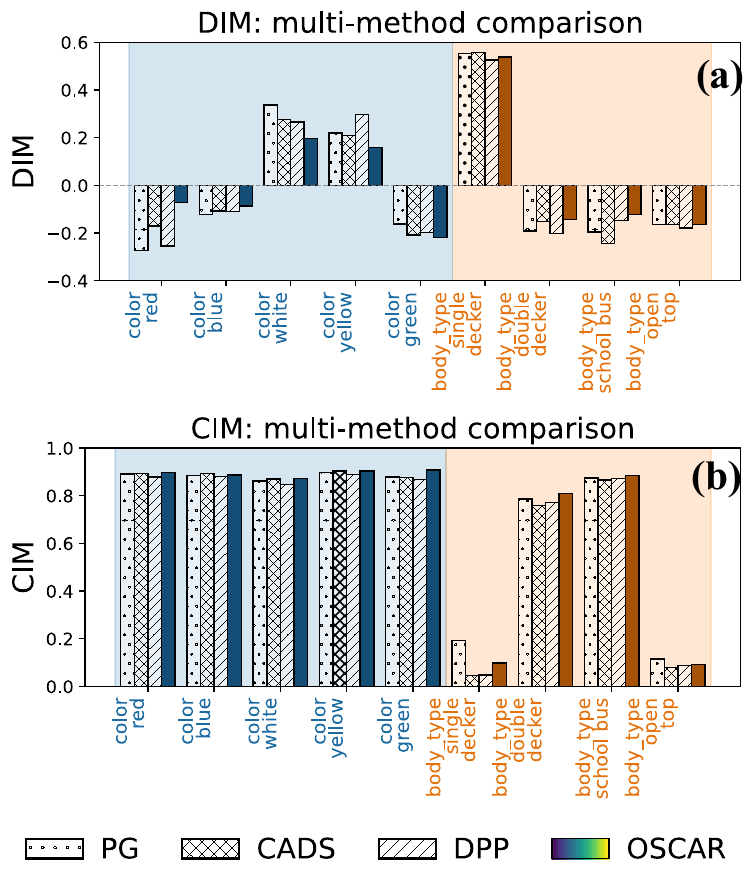}
    \caption{Evaluation of attribute-level diversity and generalization on the \textit{bus} concept. DIM quantifies the balance of attributes generated from coarse prompts (a score closer to 0 indicates better balance), while CIM assesses the ability to follow explicit attribute requests (a higher score is better).}
    \label{fig:DIM_CIM_bus}
    \vspace{-1em}
\end{figure}

\noindent\textbf{Attribute-Level Diversity.}
At the semantic level, we analyze our method's ability to balance default-mode diversity with explicit controllability. For this evaluation on the \textit{bus} concept, we generated images for each core attribute using the structured prompt system detailed in Appendix~\ref{sec:appendix_dim_cim_details}. The results in Figure~\ref{fig:DIM_CIM_bus}(a) show that while baseline methods exhibit strong biases, our method consistently achieves scores closer to zero, signifying a more balanced generation. Crucially, this improved balance does not compromise instruction-following, as Figure \ref{fig:DIM_CIM_bus}(b) shows our method maintains high, competitive CIM scores. This demonstrates that our method successfully overcomes the typical trade-off between default-mode diversity and explicit control.

\subsection{Ablation Study}

This section explores the role of the most important safeguards and parameters in our method on the final quality and diversity of generated samples. For a more detailed analysis of our method's robustness to different noise distributions, please see Appendix~\ref{sec:appendix_robust}, where we show that our approach remains effective across different sampling strategies, backbone architectures, and feature encoders.

\begin{figure}[h]
    \centering
    \includegraphics[width=1\linewidth]{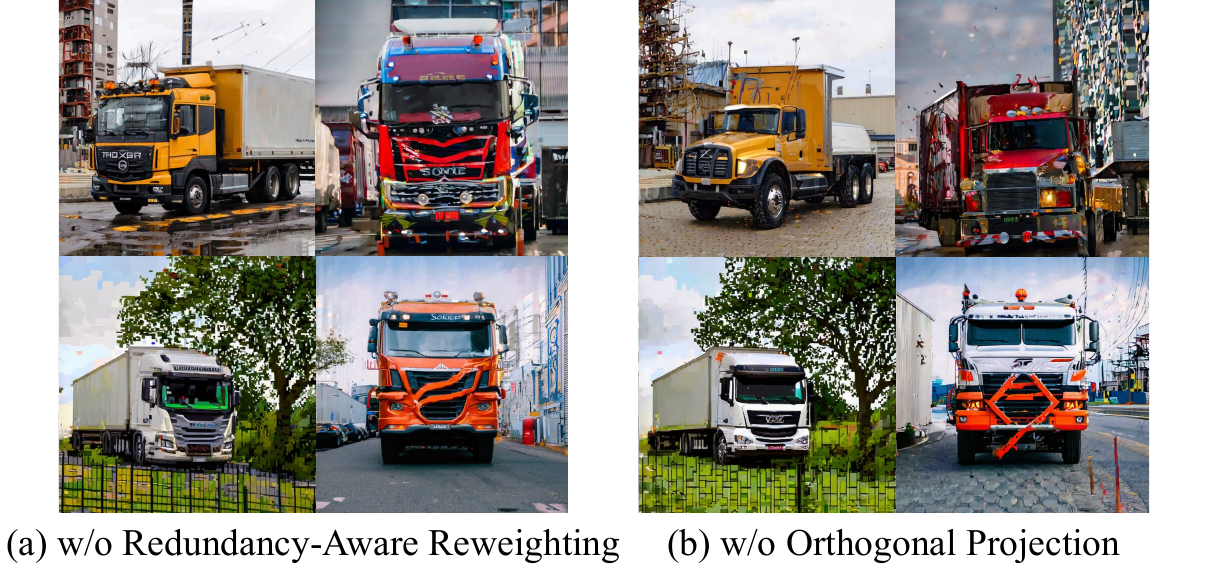}
    \caption{Ablation study of two fidelity safeguards.}
    \label{fig:ablation_wo_visuals}
\end{figure}

\begin{table}[t]
    \centering
    \caption{\small
    Comparison with baselines and stepwise ablation of our fidelity safeguards at a fixed guidance $CFG{=}5.0$.
    Our full model outperforms both the foundational baseline and prior methods.
    The ablation study shows that each component contributes to the final performance.
    }
    \label{tab:ablation_quality}
    
    \renewcommand{\arraystretch}{0.75} 
    \setlength{\aboverulesep}{2pt} 
    \setlength{\belowrulesep}{2pt}
    
    \setlength{\tabcolsep}{1.5pt} 
    
    \resizebox{\columnwidth}{!}{%
        \begin{tabular}{lcccccc}
            \toprule
            \textbf{Method} &
            \makecell{\textbf{CLIP}\\\textbf{Score}$\uparrow$} & 
            \makecell{\textbf{Vendi}\\\textbf{(Pixel)} $\uparrow$} &
            \makecell{\textbf{Vendi}\\\textbf{(Inc.)} $\uparrow$} & 
            \makecell{\textbf{FID}$\downarrow$} &
            \makecell{\textbf{BRISQUE}$\downarrow$} \\
            
            \specialrule{\heavyrulewidth}{\aboverulesep}{\belowrulesep}
            FM-SD3.5
            & 28.24 $\pm$ 0.18 & 2.45 $\pm$ 0.13 & 5.37 $\pm$ 0.27 & 164.4 $\pm$ 1.8 & 23.4 $\pm$ 1.4 \\
            
            \specialrule{\heavyrulewidth}{\aboverulesep}{\belowrulesep}
            DPP
            & 28.84 $\pm$ 0.22 & 2.55 $\pm$ 0.23 & 5.49 $\pm$ 0.16 & 171.2 $\pm$ 1.9 & 26.1 $\pm$ 3.0 \\
            CADS
            & 27.32 $\pm$ 0.21 & 2.54 $\pm$ 0.28 & 5.58 $\pm$ 0.19 & 169.1 $\pm$ 2.5 & 24.0 $\pm$ 1.9 \\
            PG
            & 28.27 $\pm$ 0.27 & 2.54 $\pm$ 0.13 & 5.43 $\pm$ 0.18 & 163.7 $\pm$ 1.3 & 24.2 $\pm$ 3.3 \\
            
            \specialrule{\heavyrulewidth}{\aboverulesep}{\belowrulesep}
            w/o OP \& RR
            & 26.70 $\pm$ 0.23 & 2.82 $\pm$ 0.20 & 4.86 $\pm$ 0.24 & 185.9 $\pm$ 1.8 & 50.1 $\pm$ 2.8 \\
            w/o RR
            & 27.26 $\pm$ 0.39 & 2.77 $\pm$ 0.16 & 5.23 $\pm$ 0.28 & 174.4 $\pm$ 3.0 & 35.0 $\pm$ 1.5 \\
            w/o OP
            & 26.59 $\pm$ 0.25 & 2.75 $\pm$ 0.19 & 5.06 $\pm$ 0.16 & 177.8 $\pm$ 3.2 & 38.8 $\pm$ 1.5 \\
            
            \rowcolor{lightblue}
            \textbf{Oscar}
            & \textbf{28.26 $\pm$ 0.22} & \textbf{2.86 $\pm$ 0.05} & \textbf{5.63 $\pm$ 0.22} & \textbf{163.3 $\pm$ 1.6} & \textbf{21.2 $\pm$ 1.5} \\
            \bottomrule
        \end{tabular}%
    }
\end{table}

\noindent\textbf{The Core Role of Fidelity Safeguards.}
We validate our fidelity safeguards by comparing the full OSCAR model against two ablated variants: one without Orthogonal Projection (w/o OP), and one without Redundancy-Aware (w/o RR). The quantitative results in Table~\ref{tab:ablation_quality} show that removing these components consistently harms all metrics, with variants that drop OP suffering the most severe degradation. The visual results in Figure~\ref{fig:ablation_wo_visuals} further confirm this degradation, illustrating that these ablated variants are prone to chaotic textures and implausible artifacts, whereas our method maintains both diversity and realism. More comprehensive qualitative comparisons are provided in Appendix~\ref{sec:appendix_more_visuals}.

\subsection{Sensitive Analysis}
This section will explore the role of the most important parameters in our method on the final quality and diversity of generated samples.

\begin{figure*}[t]
    \centering
    \begin{minipage}[t]{0.32\textwidth}
        \centering
        \textbf{(a) Influence of $\boldsymbol{\lambda}$} \\ \vspace{2pt}
        \begin{tabular}{@{}c@{\hspace{2pt}}c@{}}
             \footnotesize $\lambda=0.3$ & \footnotesize $\lambda=0.9$ \\
             \includegraphics[width=0.48\linewidth]{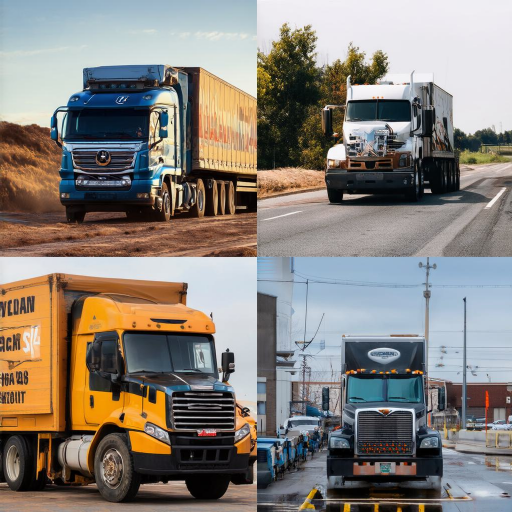} &
             \includegraphics[width=0.48\linewidth]{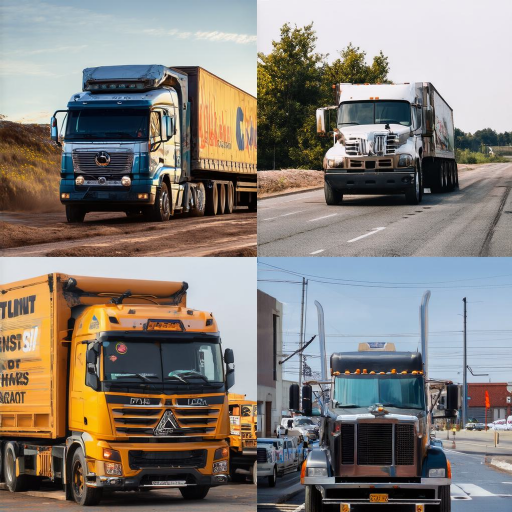}
        \end{tabular}
    \end{minipage}
    \hfill
    \begin{minipage}[t]{0.32\textwidth}
        \centering
        \textbf{(b) Influence of $\boldsymbol{\alpha}$} \\ \vspace{2pt}
        \begin{tabular}{@{}c@{\hspace{2pt}}c@{}}
             \footnotesize $\alpha=0.1$ & \footnotesize $\alpha=1.0$ \\
             \includegraphics[width=0.48\linewidth]{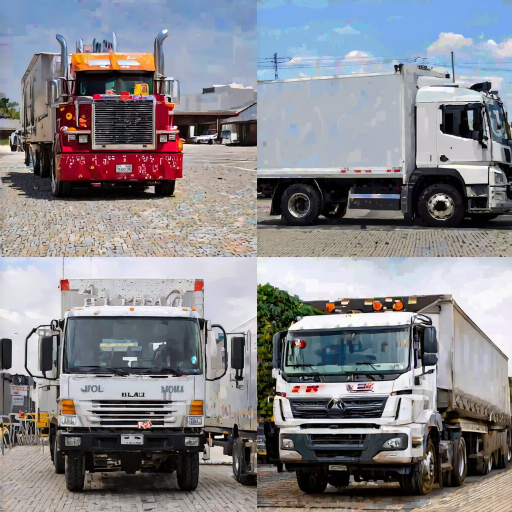} &
             \includegraphics[width=0.48\linewidth]{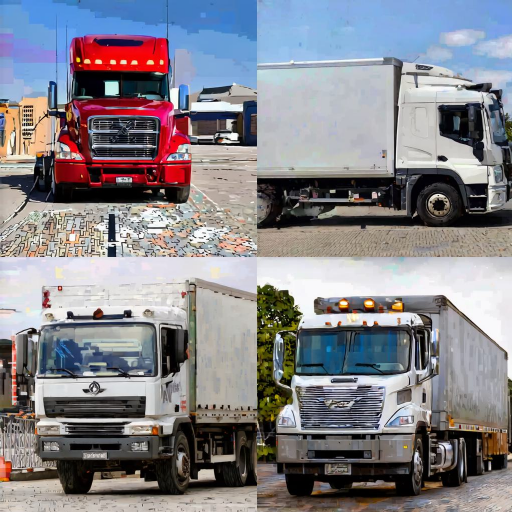}
        \end{tabular}
    \end{minipage}
    \hfill
    \begin{minipage}[t]{0.32\textwidth}
        \centering
        \textbf{(c) Influence of $\boldsymbol{t_{\text{gate}}}$} \\ \vspace{2pt}
        \begin{tabular}{@{}c@{\hspace{2pt}}c@{}}
             \footnotesize $t=0.70$ & \footnotesize $t=1.00$ \\
             \includegraphics[width=0.48\linewidth]{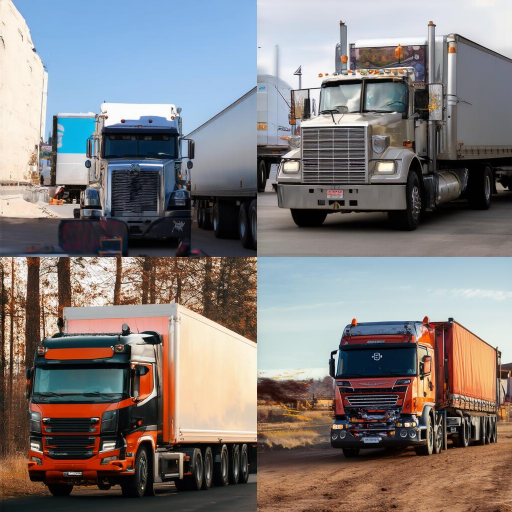} &
             \includegraphics[width=0.48\linewidth]{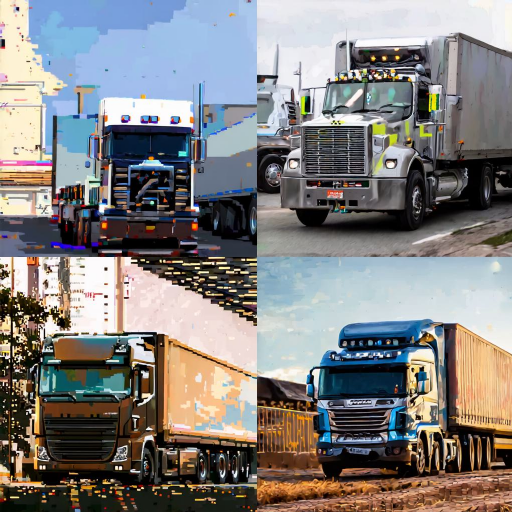}
        \end{tabular}
    \end{minipage}
    \vspace{-0.5em}
    \caption{Visual ablation study. We compare visual outcomes for varying $\lambda$ (a), $\alpha$ (b), and $t_{\text{gate}}$ (c). The corresponding quantitative metrics are detailed in Table~\ref{tab:ablation_combined}.}
    \label{fig:ablation_visuals}
    \vspace{-1em}
\end{figure*}

\noindent\textbf{Orthogonality $\boldsymbol{\lambda}$.}
The influence of the orthogonality strictness $\lambda$ is detailed in Figure~\ref{fig:ablation_visuals}(a). This parameter has a critical impact on generation quality but a minor effect on diversity. As the data shows, relaxing the constraint by decreasing $\lambda$ leads to a noticeable degradation in fidelity, confirming that a strict orthogonal projection ($\lambda \approx 0.9$) is essential for preserving image quality.

\noindent\textbf{Leverage exponent $\boldsymbol{\alpha}$.}
The effect of the leverage exponent $\alpha$ is shown in Figure~\ref{fig:ablation_visuals}(b). The results reveal that an intermediate value (e.g., $\alpha=0.5$) is better, as both fully adaptive ($\alpha=1.0$) and uniform ($\alpha=0.0$) weighting are found to degrade sample quality and reduce diversity.

\noindent\textbf{Noise gate $\boldsymbol{t_{\text{gate}}}$.}
Figure~\ref{fig:ablation_visuals}(c) illustrates the impact of the noise injection window by comparing extreme scenarios. While a conservative, late-stage injection is effective, extending the window to the entire generation process degrades sample quality as expected. However, these extremes belie the method's true stability, as our approach is in fact highly robust across a wide range of reasonable settings. As detailed in Appendix~\ref{sec:add_noise}, we further demonstrate robustness with respect to the noise gate, the noise scale, and the choice of noise schedule.

\begin{table}[t]
    \centering
    \caption{Quantitative ablation results for key parameters.}
    \label{tab:ablation_combined}
    \resizebox{\columnwidth}{!}{%
        \begin{tabular}{ccc @{\hspace{1.5em}} ccc @{\hspace{1.5em}} ccc}
            \toprule
            \multicolumn{3}{c}{\textbf{(a) Influence of $\boldsymbol{\lambda}$}} & 
            \multicolumn{3}{c}{\textbf{(b) Influence of $\boldsymbol{\alpha}$}} & 
            \multicolumn{3}{c}{\textbf{(c) Influence of $\boldsymbol{t_{\text{gate}}}$}} \\
            \cmidrule(r){1-3} \cmidrule(lr){4-6} \cmidrule(l){7-9} 
            
            $\boldsymbol{\lambda}$ & \textbf{Brisque} $\downarrow$ & \textbf{$\mathrm{VS}_{p}$} $\uparrow$ &
            $\boldsymbol{\alpha}$  & \textbf{Brisque} $\downarrow$ & \textbf{$\mathrm{VS}_{p}$} $\uparrow$ &
            $\boldsymbol{t_{\text{gate}}}$ & \textbf{Brisque} $\downarrow$ & \textbf{$\mathrm{VS}_{p}$} $\uparrow$ \\ 
            \midrule
            
            0.9 & 21.56 & 2.61 & 1.0 & 40.39 & 2.50 & 0.15 & 23.13 & 2.60 \\
            0.6 & 21.69 & 2.30 & 0.5 & 20.63 & 2.66 & 0.70 & 33.67 & 2.48 \\
            0.3 & 23.64 & 2.34 & 0.1 & 44.28 & 2.49 & 1.00 & 53.88 & 2.21 \\ 
            \bottomrule
        \end{tabular}%
    }
\end{table}
\vspace{-1mm}

\noindent\textbf{Overall robustness to hyperparameters.}
While OSCAR introduces several hyperparameters, an important empirical observation is that the method is not sensitive to their precise tuning. Across all three axes above, we find that performance degradation only occurs under deliberately extreme configurations. In contrast, within a broad and practically reasonable range of values, the diversity--quality balance remains remarkably stable. This behavior reflects the role of these parameters as structural safeguards rather than fine-grained tuning knobs. The orthogonality factor $\lambda$ primarily determines whether diversity guidance interferes with the base flow, the leverage exponent $\alpha$ controls the degree of redundancy suppression among samples, and the noise gate $t_{\text{gate}}$ specifies when stochastic exploration is permitted. Once these mechanisms operate within their intended regimes, their exact magnitudes have only a minimal effect on the outcome. This robustness allows us to use a single, fixed set of hyperparameters across all experiments (see details in Appendix~\ref{sec:hyper_setup}). As a result, OSCAR does not require careful hyperparameter tuning to achieve strong performance, and its gains in diversity are consistently realized without sacrificing fidelity across a wide range of settings.

\section{Conclusion and Broader Discussion}

In this work, we addressed the critical challenge of the fidelity-diversity trade-off in flow-matching models by introducing a training-free, geometry-consistent control framework. Our method combines a deterministic guidance signal with controllable stochastic noise to enhance semantic diversity. The key innovation is that both components are projected to be orthogonal to the model's base velocity field, which decouples the diversity-seeking signal from the quality-generating flow. Our experiments show that this design yields a stronger precision-recall trade-off and improves a suite of diversity metrics while preserving image quality and prompt alignment. Beyond the inference-time setting studied in this paper, our geometry-consistent control principle also suggests a broader perspective on exploration in flow-based generative models.

Recent RL post-training methods for flow-based generative models often introduce stochasticity through ODE-to-SDE conversion or mixed ODE/SDE sampling to improve exploration during online optimization~\citep{liu2025flow,li2025mixgrpo,xue2025dancegrpo}. This shares a similar motivation with OSCAR, but our results highlight a key distinction between unconstrained stochastic exploration and quality-preserving exploration. As shown in Appendix~\ref{sec:appendix_mix_ode_sde}, a simple Mix ODE/SDE sampler can improve certain diversity-related metrics, but without an explicit fidelity-preserving mechanism, it often degrades visual quality and prompt alignment. This suggests that exploration in flow-based generation should not be reduced to injecting randomness into the trajectory.

From this perspective, OSCAR provides a geometry-constrained form of exploration. Its set-level endpoint objective encourages trajectories to spread across semantic modes, while orthogonal projection decouples both deterministic control and stochastic perturbation from the base flow direction. This makes OSCAR complementary to RL post-training: it can serve as a training-free alternative when post-training is unavailable, and its quality-preserving control principle may provide a structured exploration prior for future RL-based optimization. Incorporating such geometry-aware diversity control into RL post-training could help broaden semantic exploration while reducing the risk of fidelity degradation.

\section*{Acknowledgements}
This research/project is supported by the National Research Foundation, Singapore under its National Large Language Models Funding Initiative (AISG Award No: AISG-NMLP-2024-003). Any opinions, findings and conclusions or recommendations expressed in this material are those of the author(s) and do not reflect the views of National Research Foundation, Singapore. 

\section*{Impact Statement}

This paper presents work whose goal is to advance the field of Machine Learning, specifically in the area of inference-time control for generative models. The proposed method aims to improve semantic diversity in conditional generation while preserving alignment and output quality. There are many potential societal consequences of generative models, including both beneficial applications in creative and scientific domains and potential risks related to misuse or biased content generation. These broader considerations are well-studied and are not unique to the techniques introduced in this work. We do not identify any additional ethical concerns or societal impacts that require specific discussion beyond those commonly associated with modern generative modeling research.

\bibliography{main}
\bibliographystyle{icml2026}

\newpage
\appendix
\onecolumn



\section{Use of LLMs}

The use of LLMs in this study was strictly limited to their function as writing and editing assistants. Their role involved improving the manuscript's grammar, spelling, and stylistic consistency. LLMs did not contribute to any core scientific aspects of the work, such as research ideation, experimental design, data analysis, or interpretation of results. All substantive content and intellectual contributions are solely the work of the authors.

\section{Theoretical Properties: Quality-Preserving Diversity via Orthogonal Stochastic Control}
\label{sec:appendix_theory}

\paragraph{Notation \& Setup.}
Let $\hat\psi:\,\mathcal X\times[0,1]\to U$ be an endpoint predictor and
$\phi:\,U\to\mathbb R^{D}$ a semantic feature encoder. For $m$ concurrent trajectories, define the
endpoint features
\[
z_i(t)\;=\;\phi\!\big(\hat\psi(x_i(t),t)\big),\qquad i=1,\dots,m,
\]
stack rows $Z(t)=[z_1(t);\ldots;z_m(t)]\in\mathbb R^{m\times D}$, and let the Gram matrix be
$K(t)=Z(t)Z(t)^\top\in\mathbb R^{m\times m}$. We use the trace–stabilized log–det set energy
\begin{equation}
\label{eq:logdet_energy_app}
\mathcal E_s(Z)\;=\;-\tfrac{1}{2}\,\log\det\!\Big(I+\tau\,K \Big),
\qquad
\tau,\varepsilon>0.
\end{equation}

Let $v_\theta(x,t\,|\,c)$ denote the base drift with optional conditioning $c$ and define the
regularized unit base direction
\[
\hat q(x,t)\;=\;\frac{v_\theta(x,t\,|\,c)}{\|v_\theta(x,t\,|\,c)\|+\delta},\qquad \delta>0,
\]
together with the quality–orthogonal projector
\[
\Pi_\perp\;=\;I-\lambda\,\hat q\,\hat q^\top,\qquad \lambda\in[0,1].
\]
Our feature–space controlled SDE is
\begin{equation}
\label{eq:sde}
dZ_t\;=\;-\Pi_\perp\,\nabla_Z \mathcal E_s(Z_t)\,dt\;+\;\sqrt{\beta(t)}\,\Pi_\perp\,dW_t,
\end{equation}
where $\beta:[0,1]\to\mathbb R_+$ is a bounded, nonincreasing noise schedule and $W_t$ is
standard Brownian motion in $\mathbb R^{m\times D}$.

\begin{assumption}[Local smoothness and bounded curvature on the orthogonal subspace]
\label{asp:A-smooth}
(i) $\phi$ and $\hat\psi$ are $C^2$ with bounded Jacobians/Hessians on the sampling region.
(ii) $\mathcal{E}_s$ is $C^2$ on $\{Z: I+\tau ZZ^\top+\varepsilon_{\mathrm{tr}}I\succ0\}$.
(iii) There exists $L_\perp>0$ such that
$\mathrm{tr}\!\big(\Pi_\perp^\top \nabla^2\mathcal{E}_s(Z)\,\Pi_\perp\big)\le L_\perp$
for all $Z$ encountered during sampling.
\end{assumption}

\begin{lemma}[Pullback identity]
\label{lem:A-pullback}
Let $\Phi(x,t)=\phi(\hat\psi(x,t))$ and $Z=[\Phi(x^{(1)},t);\ldots;\Phi(x^{(m)},t)]$.
Then for each $i$,
\[
\frac{\partial}{\partial x^{(i)}} \mathcal{E}_s(Z)
=\Big(J_x\hat\psi(x^{(i)},t)\Big)^\top \Big(J_u\phi(u)\Big)^\top
\Big[\nabla_Z \mathcal{E}_s(Z)\Big]_i\Big|_{u=\hat\psi(x^{(i)},t)} .
\]
\end{lemma}

\begin{proof}[Proof of Lemma~\ref{lem:A-pullback}]
Let $f(Z)\coloneqq \mathcal{E}_s(Z)$ and define the row-stacking map
$G(x^{(1)},\ldots,x^{(m)};t)=Z=[z_1;\ldots;z_m]$ with
$z_i=\Phi(x^{(i)},t)=\phi(\hat\psi(x^{(i)},t))$.
Fix $i$ and a direction $v\in\mathbb{R}^{\dim(x)}$. By the chain rule for
directional derivatives (Frobenius inner product), we have
\[
\frac{d}{d\epsilon}\, f\!\big(G(x^{(1)},\ldots,x^{(i)}+\epsilon v,\ldots,x^{(m)};t)\big)\Big|_{\epsilon=0}
= \big\langle \nabla_Z f(Z),\, \dot Z \big\rangle_F ,
\]
where $\dot Z$ is the first‐order variation of $Z$ induced by the perturbation
$x^{(i)}\mapsto x^{(i)}+\epsilon v$. Only the $i$‐th row varies, hence
\[
\dot z_i
= \frac{d}{d\epsilon}\, \Phi(x^{(i)}+\epsilon v,t)\Big|_{\epsilon=0}
= J_u\phi(u_i)\, J_x\hat\psi(x^{(i)},t)\, v, \qquad u_i=\hat\psi(x^{(i)},t).
\]
Using $\langle A,B\rangle_F=\sum_{j}\langle A_j,B_j\rangle$ (rowwise),
\[
\frac{d}{d\epsilon} f(\cdot)\Big|_{\epsilon=0}
= \big\langle [\nabla_Z f(Z)]_i,\; \dot z_i \big\rangle
= v^\top \Big(J_x\hat\psi(x^{(i)},t)\Big)^\top
          \Big(J_u\phi(u_i)\Big)^\top
          \big[\nabla_Z \mathcal{E}_s(Z)\big]_i .
\]
Since this holds for every $v$, we obtain
\[
\frac{\partial}{\partial x^{(i)}} \mathcal{E}_s(Z)
= \Big(J_x\hat\psi(x^{(i)},t)\Big)^\top \Big(J_u\phi(u)\Big)^\top
\Big[\nabla_Z \mathcal{E}_s(Z)\Big]_i\Big|_{u=\hat\psi(x^{(i)},t)} ,
\]
which is the desired identity. The $C^2$ conditions in
Assumption~\ref{asp:A-smooth} ensure differentiability and justify exchanging
differentiation with the Frobenius inner product. \qedhere
\end{proof}

\begin{remark}[Scope of assumptions]
We do not assume global boundedness for deep networks. Our analysis is restricted to the
operational domain $\mathcal D$ traced by the sampler (finite horizon, bounded step sizes,
regularization), where neural networks are empirically locally smooth. Such local assumptions
are standard in analyses of learned dynamics and control. In practice, weight decay/normalization
and gradient clipping further bound the effective Jacobians on $\mathcal D$.
\end{remark}

\begin{theorem}[Expected energy descent \& marginal quality preservation]
\label{thm:A-main1}
Under Assumption~\ref{asp:A-smooth}, along the SDE \eqref{eq:controlled_sde} we have
\begin{equation}
\label{eq:A-descent}
\frac{d}{dt}\,\mathbb{E}\,\mathcal{E}_s(Z_t)
\ \le\
-\,\mathbb{E}\,\big\|\Pi_\perp\nabla_Z\mathcal{E}_s(Z_t)\big\|_F^2
\ +\ \beta(t)\,L_\perp .
\end{equation}
Consequently, since $\beta(t)\!\downarrow\!0$ as $t\uparrow 1$, there exists $t^\star\!\in(0,1)$ such that
for all $t\ge t^\star$, $\tfrac{d}{dt}\mathbb{E}\,\mathcal{E}_s(Z_t)<0$ (late-stage monotonicity).
Moreover, because both the control drift and diffusion act in $\mathrm{range}(\Pi_\perp)$,
the alignment coordinate $y=\langle x,\hat q\rangle$ has the same marginal evolution as the base FM ODE
to first order in the step size: no diffusion and no control drift along $y$.
\end{theorem}

\begin{proof}
Apply Itô's lemma to the twice-differentiable energy $\mathcal{E}_s(Z_t)$ under
\eqref{eq:sde}. Writing $G(Z)=\nabla_Z\mathcal{E}_s(Z)$ and $H(Z)=\nabla_Z^2\mathcal{E}_s(Z)$,
with drift $b(Z_t)=-\Pi_\perp G(Z_t)$ and diffusion matrix
$\Sigma_t=\sqrt{\beta(t)}\,\Pi_\perp$, we obtain
\[
\frac{d}{dt}\mathbb{E}\,\mathcal{E}_s(Z_t)
= \mathbb{E}\big[\langle G(Z_t),\, b(Z_t)\rangle_F\big]
  + \tfrac{1}{2}\,\mathbb{E}\!\big[\!\langle H(Z_t),\, \Sigma_t\Sigma_t^\top\rangle_F\big].
\]
The first term equals $-\mathbb{E}\,\|\Pi_\perp G(Z_t)\|_F^2$. For the second term, using
$\Sigma_t\Sigma_t^\top=\beta(t)\,\Pi_\perp$ and Assumption~\ref{asp:A-smooth}(iii),
\[
\tfrac{1}{2}\,\mathbb{E}\!\big[\!\langle H(Z_t),\, \Sigma_t\Sigma_t^\top\rangle_F\big]
= \tfrac{\beta(t)}{2}\,\mathbb{E}\!\big[\mathrm{tr}(\Pi_\perp^\top H(Z_t)\Pi_\perp)\big]
\ \le\ \beta(t)\,L_\perp ,
\]
absorbing the factor $\tfrac12$ into $L_\perp$ if desired. Combining the two gives
\eqref{eq:A-descent}. Since $\beta(t)$ is bounded and $\beta(t)\!\downarrow\!0$ as $t\uparrow 1$,
there exists $t^\star$ such that the negative first term dominates for all $t\ge t^\star$,
hence $\tfrac{d}{dt}\mathbb{E}\,\mathcal{E}_s(Z_t)<0$.

For the alignment coordinate $y=\langle x,\hat q\rangle$, note that both the control drift
$-\Pi_\perp G$ and the diffusion $\sqrt{\beta(t)}\,\Pi_\perp dW_t$ lie in $\mathrm{range}(\Pi_\perp)$,
which is orthogonal to $\hat q$. Therefore, to first order in the step size, there is neither
additional drift nor diffusion along $y$ beyond that of the base FM ODE, establishing marginal
quality preservation. \qedhere
\end{proof}

\begin{remark}[Discretization and redundancy-aware reweighting]
\label{rem:disc-rew}
\leavevmode
\begin{enumerate}[label=\textnormal{(\roman*)}, leftmargin=1.5em, itemsep=2pt, topsep=2pt]

\item \textbf{Heun-style discretization.}
With step size $\Delta t$, the update reads
\begin{equation}
\label{eq:heun-step}
x_{k+1}
= x_k + \Delta t\!\left[\tfrac12\!\left(v_k + v_{k+1}\right) + u_k\right],
\qquad
u_k = \Pi_\perp g(x_k,t_k)\;+\;\sqrt{\beta(t_k)\,\Delta t}\;\Pi_\perp \xi_k,
\ \ \xi_k \sim \mathcal N(0,I).
\end{equation}
This inherits the descent certificate \eqref{eq:A-descent} up to an $O(\Delta t^2)$ truncation error.
In practice, we use $\Delta t \le 0.05$, making the discretization error negligible relative to the stochastic variability.

\item \textbf{Redundancy-aware reweighting.}
Let $s_i = \big[(K+\varepsilon I)^{-1}\big]_{ii}$ and set $w_i \propto s_i^{\alpha}$ ($\alpha>0$).
With $W=\mathrm{diag}(w_1,\dots,w_m)$, replace $K$ by $\tilde K = W^{1/2} K W^{1/2}$ in the energy
\eqref{eq:logdet_energy_app}. This preconditioning boosts underrepresented samples and attenuates
redundant ones, and it preserves the descent inequality \eqref{eq:A-descent} with a possibly different constant $L_\perp$, thereby stabilizing the set gradient without an explicit trust region.
\end{enumerate}
\end{remark}

\noindent\emph{Link to Theorem~\ref{thm:A-main2}.}
Replacing $K$ by $\tilde K$ keeps the identity
$-\mathcal E_s(Z)=\tfrac12\log\det\!\big(I+\tau\,\tilde K+\varepsilon_{\mathrm{tr}}I\big)$,
so Theorem~\ref{thm:A-main2} applies verbatim to the \emph{weighted} set as a weighted
volume/diversity measure.


\begin{assumption}[Feature regularity for geometric interpretation]
\label{asp:feat-regular}
(\emph{Approx.~centering}) $\sum_{i=1}^m z_i \approx 0$;\quad
(\emph{Norm stabilization}) $\|z_i\|^2 \in [1-\rho,\,1+\rho]$ for a small $\rho$.
\end{assumption}

\begin{remark}[Practical scope]
CLIP-style encoders output $\ell_2$-normalized features and we further apply mini-batch centering.
The objective \eqref{eq:A-volume} does \emph{not} require strict centering; Assumption~\ref{asp:feat-regular}
is only used to interpret log-det as an isotropic ``volume'' and to link it to diversity measures.
\end{remark}

\begin{theorem}[Energy $\downarrow$ $\Longleftrightarrow$ Volume $\uparrow$ $\Longrightarrow$ Diversity $\uparrow$]
\label{thm:A-main2}
Let $K=ZZ^\top$ and define the set volume
\[
\mathcal V(Z)\;=\;\det\!\big(I+\tau K+\varepsilon_{\mathrm{tr}}I\big)^{1/2}.
\]
Then
\begin{equation}
\label{eq:A-volume}
-\mathcal E_s(Z)\;=\;\tfrac12\log\det\!\big(I+\tau K+\varepsilon_{\mathrm{tr}}I\big)\;=\;\log \mathcal V(Z).
\end{equation}
Consequently, by Theorem~\ref{thm:A-main1}, $\;\mathbb{E}\,\log\mathcal V(Z_t)$ is (late-stage) increasing.
Moreover, under Assumption~\ref{asp:feat-regular}:
\begin{enumerate}\itemsep4pt
\item[\textnormal{(i)}] (\emph{Angles / correlations}) With approximately unit-norm rows, increasing $\mathcal V(Z)$
reduces average pairwise correlations and increases pairwise angles among $\{z_i\}$.
\item[\textnormal{(ii)}] (\emph{Spectrum uniformity}) Since $\mathrm{tr}(K)=\sum_i\|z_i\|^2\approx m$ is (approximately) fixed,
the concavity of $\log\det(\cdot)$ implies that a larger $\log\det(I+\tau K+\varepsilon_{\mathrm{tr}}I)$ corresponds to a more
uniform spectrum of $K$ (equivalently, more balanced singular values of $Z$), indicating higher coverage/diversity.
\end{enumerate}
\end{theorem}

\begin{proof}[Proof]
The identity \eqref{eq:A-volume} is immediate from the definition of $\mathcal E_s$.
Late-stage monotonicity of $\mathbb{E}\log\mathcal V(Z_t)$ follows from Theorem~\ref{thm:A-main1}.

For (ii) (\emph{spectrum uniformity}), write $A=\alpha I+\tau K$ with $\alpha=1+\varepsilon_{\mathrm{tr}}>0$.
Let $(\lambda_j)$ be the eigenvalues of $K$. Then
\[
\log\det A=\sum_{j=1}^m \log(\alpha+\tau\lambda_j).
\]
As a symmetric concave function of $(\lambda_j)$, this sum is \emph{Schur-concave}; under the (approximate)
trace constraint $\sum_j\lambda_j=\mathrm{tr}(K)\approx m$, it is maximized when $(\lambda_j)$ is more balanced
(Jensen/Karamata). Hence larger $\log\det A$ implies a more uniform spectrum of $K$ (and thus more balanced singular
values of $Z$).

For (i) (\emph{correlations/angles}), under Assumption~\ref{asp:feat-regular} we have $\operatorname{diag}(K)\approx\mathbf 1$,
so the diagonals are (approximately) fixed. Using
\[
\|K\|_F^2=\sum_{j=1}^m\lambda_j^2=\sum_{i=1}^m K_{ii}^2 + 2\sum_{i<j}K_{ij}^2\;\approx\; m + 2\sum_{i<j}K_{ij}^2,
\]
we see that, for fixed diagonals, the average squared off-diagonal magnitude is monotone in $\sum_j\lambda_j^2$.
Among spectra with fixed trace, $\sum_j\lambda_j^2$ is minimized when the spectrum is most uniform; thus the same
majorization argument used in (ii) implies $\sum_{i<j}K_{ij}^2$ decreases as $\log\det A$ increases. Since, with
(unit-norm) rows, $K_{ij}=\langle z_i,z_j\rangle$ are cosines, smaller off-diagonals mean reduced average correlations
and hence larger pairwise angles. This establishes (i). \qedhere
\end{proof}

\begin{theorem}[Noise robustness via budget and step size]
\label{thm:noise-robust}
Assume (i) $\hat q$ is $L_q$-Lipschitz on the operational domain; (ii) along the trajectory
$\|\nabla_Z \mathcal{E}_s(Z)\|_F \le G$; and (iii) Assumption~\ref{asp:A-smooth} holds so that
\eqref{eq:A-descent} is valid with curvature constant $L_\perp$. Let
$y_t=\langle z_t,\hat q(z_t,t)\rangle$ and $y_t^{\mathrm{base}}$ be the alignment coordinate
under the base flow (no control, no orthogonal diffusion). For an integrator with step
$\Delta t$ and any bounded, nonincreasing schedule $\beta:[0,1]\to\mathbb{R}_+$,
\begin{equation}
\label{eq:noise-bound}
\Big|\mathbb{E}\,y_1-\mathbb{E}\,y_1^{\mathrm{base}}\Big|
\ \le\ C_1\,L_q\,G\,\Delta t \;+\; C_2\,L_\perp \int_0^1 \beta(t)\,dt,
\end{equation}
for universal constants $C_1,C_2=O(1)$. Hence the deviation is $O(\Delta t)+O(B)$ with
$B:=\int_0^1 \beta(t)\,dt$.
\end{theorem}

\begin{proof}
Work with the feature-space SDE and its discretization.
Because both the control drift $-\Pi_\perp \nabla_Z \mathcal{E}_s$ and the diffusion
$\sqrt{\beta(t)}\,\Pi_\perp dW_t$ lie in $\mathrm{range}(\Pi_\perp)$, the alignment
coordinate $y=\langle z,\hat q\rangle$ receives no first-order contribution from the control
when $\hat q$ is locally frozen; the residual change is controlled by the variation of $\hat q$
and by the stochastic term. The Lipschitz property of $\hat q$ and the bound
$\|\nabla_Z \mathcal{E}_s\|_F\!\le\!G$ yield a discretization remainder bounded by
$C_1 L_q G\,\Delta t$. For the stochastic contribution,
apply Itô's isometry together with the descent certificate \eqref{eq:A-descent}:
the quadratic variation seen through the curvature bound in
Assumption~\ref{asp:A-smooth}(iii) accumulates as
$\tfrac12\,\mathbb{E}\langle \nabla_Z^2 \mathcal{E}_s,\, \beta(t)\Pi_\perp\rangle_F
\le \beta(t)\,L_\perp$, which integrates to $C_2 L_\perp \!\int_0^1\!\beta(t)\,dt$.
Combining the two parts gives \eqref{eq:noise-bound}. \qedhere
\end{proof}

\begin{proposition}[Schedule generality: budget controls the bound]
\label{prop:schedule}
Let $\beta_1,\beta_2$ be bounded, nonincreasing schedules with equal budget
$B=\int_0^1 \beta_j(t)\,dt$. Under the assumptions of Theorem~\ref{thm:noise-robust},
the bound \eqref{eq:noise-bound} is the same for $\beta_1$ and $\beta_2$ (up to constants),
and for any $a>0$, replacing $\beta$ by $\beta_a(t)=a\,\beta(t)$ scales the bound linearly
with $a$ via $B_a=a\,B$.
\end{proposition}

\begin{proof}
Immediate from \eqref{eq:noise-bound}, which depends on $\beta$ only through the integral
$\int_0^1 \beta(t)\,dt$. \qedhere
\end{proof}

\begin{remark}[Schedules and matching a target budget]
\label{rem:schedules}
A convenient closed form is
$\displaystyle \beta(t)=\frac{1-t}{1-t+\varepsilon_\beta}$ with $\varepsilon_\beta\in(0,1)$,
which decreases from $\tfrac{1}{1+\varepsilon_\beta}\!\approx\!1$ at $t{=}0$ to $0$ at $t{=}1$.
Common monotone decay choices include:
\[
\beta_0^{\text{cos}}(t)=\cos^2\!\Big(\frac{\pi t}{2}\Big),\quad
\beta_0^{\text{poly}}(t)=(1-t)^p\ (p\!\ge\!1),\quad
\beta_0^{\exp}(t)=\frac{e^{-\kappa t}-e^{-\kappa}}{1-e^{-\kappa}}\ (\kappa\!>\!0).
\]
To match a desired noise budget $B$, set $\tilde\beta(t)=a\,\beta_0(t)$ with
$a=B\big/\!\int_0^1 \beta_0(t)\,dt$.
We provide empirical comparisons of these schedules in
Appendix~\ref{sec:appendix_noise_schedules}.
\end{remark}

\begin{theorem}[Girsanov Representation of OSCAR Path Measure]
Let $\mathbb{P}$ be the path measure induced by the reference process (with orthogonal noise but no guidance) and $\mathbb{Q}$ be the path measure induced by the full OSCAR process on the time interval $[0, 1]$. Assuming the guidance signal $g(x, t)$ and the noise schedule satisfy Novikov's condition, the Radon-Nikodym derivative (likelihood ratio) of the generated trajectories is given by:

$$
\frac{d\mathbb{Q}}{d\mathbb{P}}(X_{[0,1]}) = \exp \left( \int_{0}^{1} \frac{1}{\sqrt{\beta(t)}} \langle g(X_{t}, t), dW_{t} \rangle - \frac{1}{2} \int_{0}^{1} \frac{\|g(X_{t}, t)\|^{2}}{\beta(t)} dt \right)
$$

where the inner products are defined in the subspace defined by $\Pi_{\perp}$.

Furthermore, the KL divergence between the OSCAR trajectory distribution and the reference distribution is exactly the energy cost of the guidance:

$$
\mathcal{D}_{KL}(\mathbb{Q} \| \mathbb{P}) = \frac{1}{2} \mathbb{E}_{\mathbb{Q}} \left[ \int_{0}^{1} \frac{\|g(X_{t}, t)\|^{2}}{\beta(t)} dt \right]
$$
\end{theorem}

\begin{theorem}[Process-level coupling bound]
\label{thm:proc-coupling}
Let the baseline process follow the probability-flow ODE
$\,dX_t=v_\theta(X_t,t)\,dt\,$ and the controlled process follow
\[
d\tilde X_t=\Big(v_\theta(\tilde X_t,t)-\Pi_\perp(\tilde X_t,t)\,g(\tilde X_t,t)\Big)\,dt
            +\sqrt{\beta(t)}\,\Pi_\perp(\tilde X_t,t)\,dW_t,
\]
driven by the \emph{same} Brownian motion $W_t$ (synchronous coupling).
Assume:
(i) $v_\theta(\cdot,t)$ is $L_v$-Lipschitz for all $t\in[0,1]$;
(ii) $\Pi_\perp(\cdot,t)$ is $L_\Pi$-Lipschitz and $\|\Pi_\perp\|\le 1$;
(iii) $\beta$ is bounded and nonincreasing; and
(iv) the drift budget is finite:
$M_g:=\int_0^1\|\Pi_\perp(\cdot,t)\,g(\cdot,t)\|_\infty\,dt<\infty$.
Let $D_t:=\tilde X_t-X_t$. Then there exists $C=C(L_v,L_\Pi)$ such that
\begin{equation}
\label{eq:proc-ode-ineq}
\frac{d}{dt}\,\mathbb{E}\|D_t\|^2\ \le\ C\,\mathbb{E}\|D_t\|^2\ +\ C\,\|\Pi_\perp g(\cdot,t)\|_\infty^2\ +\ C\,\beta(t),
\end{equation}
and hence, by Grönwall,
\begin{equation}
\label{eq:proc-coupling-final}
\mathbb{E}\|D_1\|^2\ \le\ C\Bigg(\Big(\int_0^1\!\|\Pi_\perp g(\cdot,t)\|_\infty dt\Big)^2\ +\ \int_0^1\!\beta(t)\,dt\Bigg).
\end{equation}
\end{theorem}

\begin{proof}
By Itô for $D_t=\tilde X_t-X_t$,
\[
dD_t=\underbrace{\big(v_\theta(\tilde X_t,t)-v_\theta(X_t,t)\big)}_{\text{Lipschitz in }D_t}\,dt
      -\Pi_\perp(\tilde X_t,t)g(\tilde X_t,t)\,dt
      +\sqrt{\beta(t)}\,\Pi_\perp(\tilde X_t,t)\,dW_t .
\]
Applying Itô to $\|D_t\|^2$ and taking expectations,
\[
\frac{d}{dt}\mathbb{E}\|D_t\|^2
=2\,\mathbb{E}\!\left\langle D_t,\ v_\theta(\tilde X_t,t)-v_\theta(X_t,t)\right\rangle
 -2\,\mathbb{E}\!\left\langle D_t,\ \Pi_\perp g(\tilde X_t,t)\right\rangle
 +\beta(t)\,\mathbb{E}\|\Pi_\perp(\tilde X_t,t)\|_F^2 .
\]
The first term is bounded by $2L_v\,\mathbb{E}\|D_t\|^2$. For the second, use Young’s inequality:
for any $\varepsilon>0$,
\(
-2\langle D,\Pi_\perp g\rangle \le \varepsilon\|D\|^2 + \varepsilon^{-1}\|\Pi_\perp g\|^2.
\)
Choosing $\varepsilon$ to absorb into $2L_v$ yields
\[
\frac{d}{dt}\mathbb{E}\|D_t\|^2 \ \le\ C\,\mathbb{E}\|D_t\|^2\ +\ C\,\|\Pi_\perp g(\cdot,t)\|_\infty^2\ +\ C\,\beta(t),
\]
since $\|\Pi_\perp\|_F^2\le C$ by (ii). Grönwall then gives
\eqref{eq:proc-coupling-final}. \qedhere
\end{proof}

\section{Implementation Details}
\label{sec:appendix_details}

\subsection{Finite-difference endpoint extrapolation.}
\label{sec:extrapolation}
Let $z_k$ denote the latent at step $k$ (after applying OSCAR at step $k-1$), and let $\Delta z^{\mathrm{ctrl}}_{k-1}$ be the control displacement added at step $k-1$ with step size $\Delta t_{k-1}$. We estimate a local velocity by a finite difference
\[
v_k \;\approx\; \frac{z_k - z_{k-1} - \Delta z^{\mathrm{ctrl}}_{k-1}}{\Delta t_{k-1}},
\]
and extrapolate a local endpoint as
\[
z_{\mathrm{ep}} \;=\; z_k + \alpha(t_k)\, v_k,
\]
where $\alpha(t_k)$ is a scalar schedule. This extrapolation reuses already computed latents and does not require an extra evaluation of $v_\theta$, so the predictor NFE is unchanged. We initially experimented with a classical Heun-style predictor that evaluates $v_\theta$ twice per step, but found that, under matched compute, it did not yield consistent improvements while doubling the predictor NFE \citep{karras2022elucidating}. For this reason, all reported results use the finite-difference extrapolation above, which captures a similar average velocity at no additional predictor cost.

\subsection{General Experimental Setup}
All experiments are conducted on the frozen, pretrained Stable Diffusion v-3.5 Medium model \citep{esser2024scaling}, with all generations performed at a resolution of 512x512 pixels and using bfloat16 precision. All runs use the model's default flow-matching Euler scheduler
(\texttt{FlowMatchEulerDiscreteScheduler}) with 30 sampling steps. A consistent negative prompt, ``low quality, blurry", is used across all experiments. The main prompts used for generation are derived from the captions of the COCO dataset. To ensure robustness, all reported metrics are aggregated over at least four distinct random seeds. For each unique setting (i.e., a specific prompt and CFG scale), we generate a set of 64 images for general class-conditional tasks, and a set of 20 images for the DIM/CIM evaluations. All experiments were performed on a system with two NVIDIA A40 GPUs.


\subsection{Framework and Baseline Implementation Details}
Our flow-matching backbone, which serves as the foundation for our method and all baselines, is adapted from the official implementation at \url{https://github.com/facebookresearch/flow_matching}. For \textbf{Particle Guidance} \citep{corso2023particle}, which was originally designed for diffusion models, we adapted the official author implementation, publicly available at \url{https://github.com/gcorso/particle-guidance}, to the flow-matching framework. For \textbf{CADS} \citep{sadat2023cads} and \textbf{Diverse Flow} \citep{morshed2025diverseflow}, official code was not provided by the authors. Therefore, we carefully re-implemented their methodologies, strictly adhering to the descriptions and formulations presented in their respective papers. For all baselines, we performed a thorough hyperparameter search for their key parameters to ensure each method was evaluated at its strongest possible performance, guaranteeing a fair and rigorous comparison.


\subsection{Baseline Implementation Details}
All baselines are implemented as inference-time methods on the same frozen pretrained backbone, using the same sampling configuration, prompts, random seeds, and evaluation protocol as OSCAR whenever applicable. For the base model, we directly use the original flow-matching sampling procedure without additional diversity guidance. For CADS, we set $\tau_1=0.6$, $\tau_2=0.9$, $\texttt{cads\_s}=0.10$, and $\texttt{noise\_scale}=1.0$. We align its noise-related settings with those of OSCAR as much as possible for a fair comparison. In practice, we found CADS to be sensitive to the noise magnitude; for example, increasing $\texttt{noise\_scale}$ to $2.0$ often produced colored spots and visible artifacts. We therefore use the above configuration as a stable and faithful implementation. For DiverseFlow, we use $\texttt{gamma\_sched}=\texttt{sqrt}$, $\texttt{kernel\_spread}=3.0$, and $\texttt{gamma\_max}=0.12$. This follows the intended design of applying stronger diversification early in the sampling process and weaker guidance near the end, with a moderate kernel width and conservative guidance strength to avoid fidelity degradation. These settings are used consistently across experiments unless otherwise specified.

\subsection{Evaluation Setup}
Our evaluation metrics are computed using standard pretrained models. Specifically, we use OpenAI's ViT-B/32 model for calculating CLIP Scores~\citep{radford2021learning}, and a standard InceptionV3 model with ImageNet weights for all Inception-based metrics~\citep{DBLP:journals/corr/SzegedyVISW15}. For perceptual quality and human-preference alignment, we additionally report ImageReward scores using the public ImageReward model~\citep{xu2023imagereward}, and CLIP-IQA, a CLIP-based no-reference image quality metric. All these models are kept fixed and are never finetuned on our generated data. For distributional metrics that require a reference set of real images, such as FID and KID, we construct a concept-specific reference dataset to ensure a fair comparison. Since our prompts are based on concepts from the COCO dataset~\citep{lin2014microsoft}, our reference set for a given concept is composed exclusively of all images corresponding to that concept's class from the COCO training set. For example, when evaluating generated ``truck'' images, the real reference dataset consists solely of the ``truck'' images from the COCO training data. This domain-aligned setup provides a more accurate measure of distributional fidelity. For the additional complex-prompt evaluation on T2I-CompBench~\citep{huang2023t2i}, we use 300 prompts sampled from the \textit{spatial color} and \textit{complex} subsets, and generate 32 images per prompt for each method. For KID computation in this setting, we use 5,000 COCO training images as the reference set to obtain a stable estimate.

\subsection{Hyperparameters Setup}
\label{sec:hyper_setup}
Table~\ref {tab:sampler_hparams} details the key hyperparameters for our method, which govern the core components of our diversity guidance and orthogonal noise injection.

\begin{table}[t]
\centering
\begin{threeparttable}
\caption{Hyperparameters used by our sampler, with default values and brief descriptions.}
\label{tab:sampler_hparams}
\setlength{\tabcolsep}{5pt} 
\renewcommand{\arraystretch}{1.15}
\begin{tabularx}{\textwidth}{l c l X}
\toprule
Name & Symbol & Value & Meaning \\
\midrule
gamma0 & $\gamma_0$ & 0.12 & Global strength for the diversity/volume term. \\
gamma-max-ratio & $\gamma_{\max}$ & 0.3 & Trust-region ratio: caps diversity displacement by a fraction of the base flow displacement norm. \\
partial-ortho & $p_{\perp}$ & 0.95 & Proportion of projection orthogonal to the base velocity. \\
noise-partial-ortho & $p_{\perp}^{\text{noise}}$ & 0.95 & Proportion of orthogonalization for the noise with respect to the base velocity. \\
t-gate & ${t}_{\text{gate}}$ & 0.4 & Time gate $[0.05,t_{gate}]$ where the diversity term \& noise is active. \\
sched-shape & $noise(t)$ & cos2 & Time schedule shape for $\gamma$ (cos2 or t1mt).\tnote{*} \\
tau & $\tau$ & 1.0 & Scale for the volume/log-det related term. \\
eta-sde & $\eta$ & 1.0 & Global scale for SDE/Brownian noise. \\
vnorm-threshold & $\delta_{v}$ & $1\times10^{-4}$ & Skip velocity-based projections when the base velocity norm is below this threshold. \\
\bottomrule
\end{tabularx}
\begin{tablenotes}[flushleft]
\footnotesize
\item[*] The cosine-squared (cos2) schedule yields a smooth, bell-shaped profile that ramps up from zero, peaks at the midpoint, and then ramps down. 
The parabolic (t1mt) schedule exhibits a similar rise–fall pattern shaped like an inverted parabola.
\end{tablenotes}
\end{threeparttable}
\end{table}

\section{Additional Main Experiments Results}
\FloatBarrier
\label{sec:appendix_main_results}

\subsection{Additional Precision-Recall Distribution Curves}
\label{sec:appendix_prd}

To demonstrate that the superior fidelity-coverage trade-off of our method is not limited to a single class, this section provides additional PRD curves for two more concepts from the COCO dataset: ``truck", ``apple", ``pizza" ``bus" and ``bicycle". As shown in Figure~\ref{fig:appendix_prd_curves}, the results are consistent with the findings for the ``truck" concept presented in the main paper. Across all concepts and guidance scales, our method's PRD curve consistently lies to the top-right of the baselines, achieving a higher recall for any given level of precision and thus a superior AUC.

\begin{figure}[h]
    \centering

    \begin{subfigure}{0.9\textwidth}
        \centering
        \includegraphics[width=\linewidth]{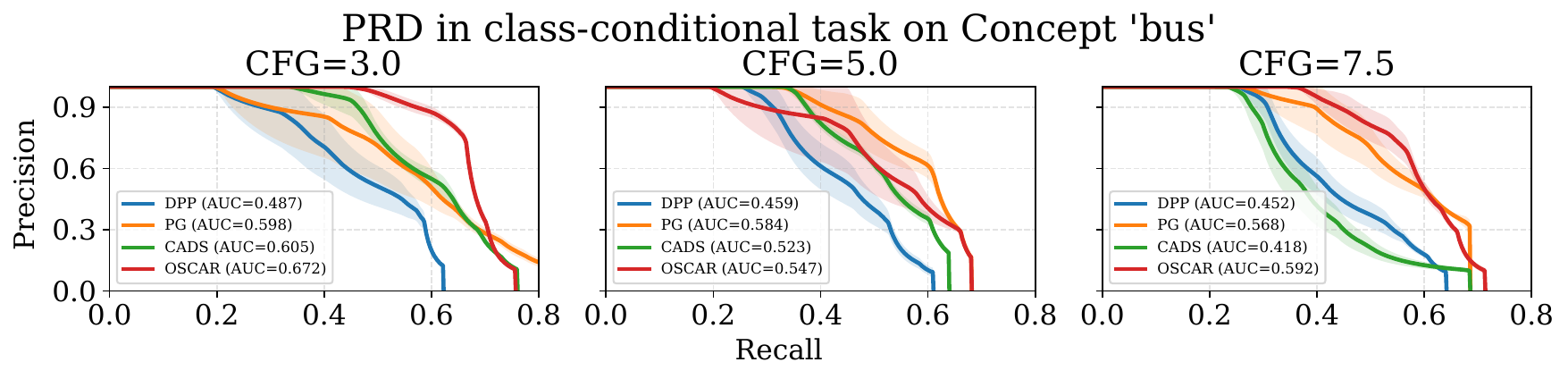}
        \caption{PRD curves for the bus concept}
        \label{fig:prd_bus}
    \end{subfigure}

    \vspace{0.8em} 

    \begin{subfigure}{0.9\textwidth}
        \centering
        \includegraphics[width=\linewidth]{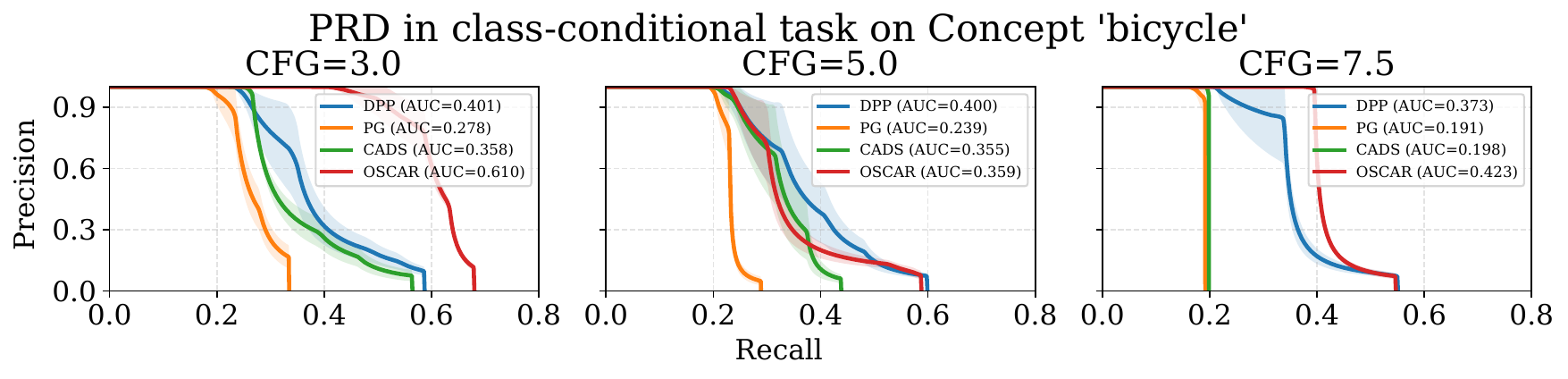}
        \caption{PRD curves for the bicycle concept}
        \label{fig:prd_bicycle}
    \end{subfigure}

    \begin{subfigure}{0.9\textwidth}
        \centering
        \includegraphics[width=\linewidth]{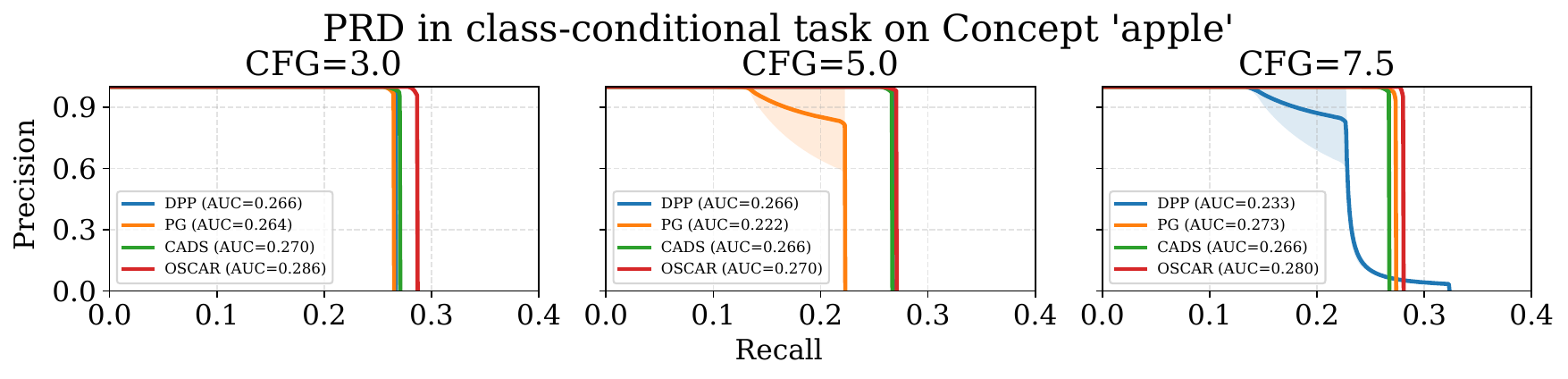}
        \caption{PRD curves for the apple concept}
        \label{fig:prd_apple}
    \end{subfigure}

    \begin{subfigure}{0.9\textwidth}
        \centering
        \includegraphics[width=\linewidth]{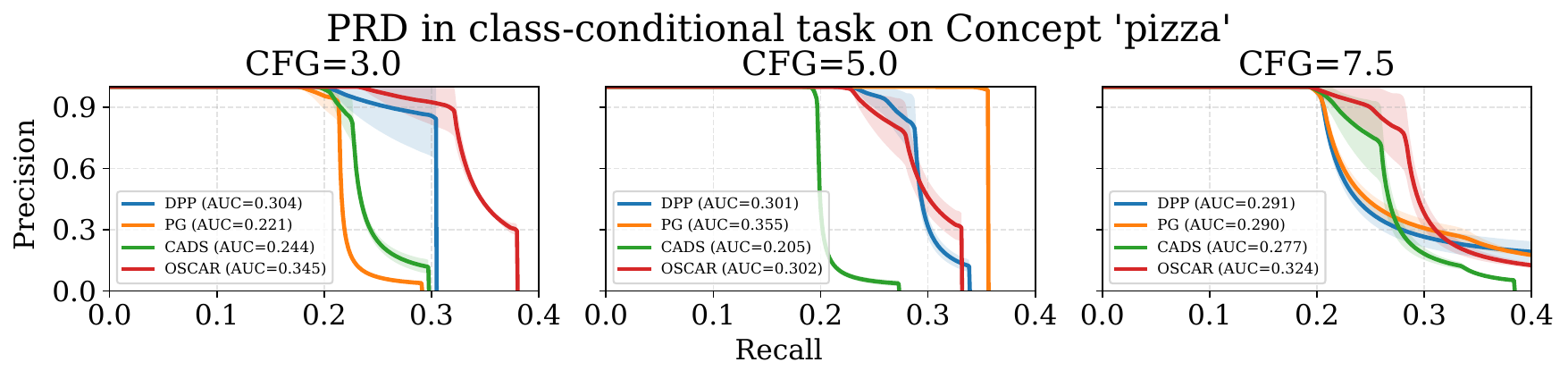}
        \caption{PRD curves for the pizza concept}
        \label{fig:prd_pizza}
    \end{subfigure}

    \begin{subfigure}{0.9\textwidth}
        \centering
        \includegraphics[width=\linewidth]{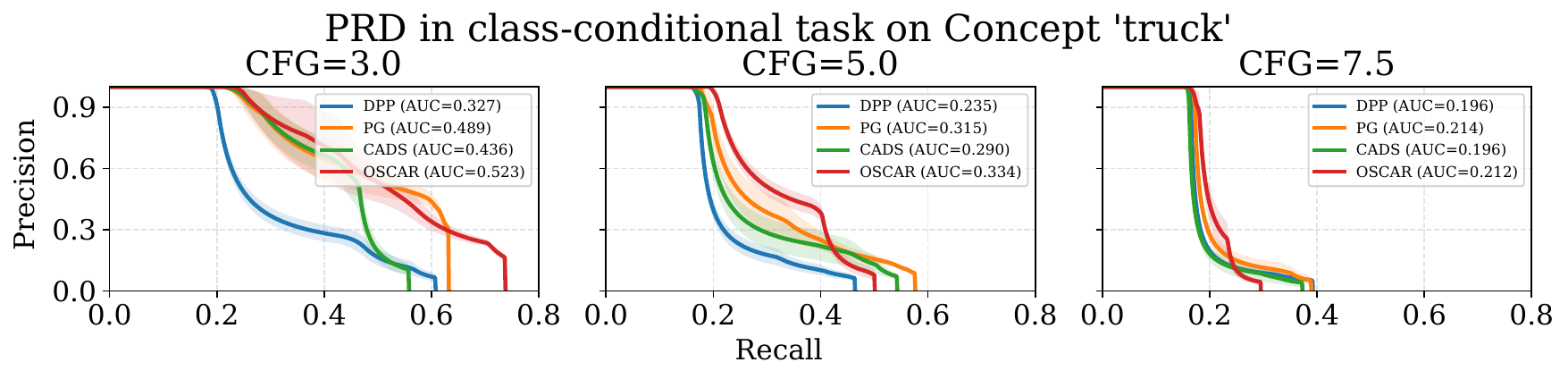}
        \caption{PRD curves for the truck concept}
        \label{fig:prd_truck}
    \end{subfigure}
        \caption{Additional PRD curves for the (a) bus, (b) bicycle, (c) apple, (d) pizza and (e) truck concepts. 
    The results confirm that our method consistently achieves a better fidelity–coverage trade-off compared to baselines across various semantic classes.}
    \label{fig:appendix_prd_curves}
\end{figure}

\subsection{Extended Mode Coverage and Entropy Analysis}
\label{sec:appendix_coverage_entropy}

To further demonstrate our method's ability to discover and evenly represent fine-grained, intra-class modes, we extend the mode coverage and normalized entropy analysis to the ``bus" and ``bicycle" concepts. The results, shown in Figure~\ref{fig:appendix_coverage_bus} and Figure~\ref{fig:appendix_coverage_bicycle}, are fully consistent with the findings for the ``truck" concept presented in the main paper.

For both additional concepts, our method consistently achieves a higher final cluster coverage plateau, indicating that it successfully identifies a broader set of unique sub-types compared to the baselines. Furthermore, the normalized entropy scores for our method are substantially and consistently higher. This provides strong evidence that our generated samples are more uniformly distributed across the discovered modes, leading to a less redundant and more semantically rich output for the end-user.

\begin{figure}[h!]
    \centering

    \begin{subfigure}{0.8\textwidth}
        \centering
        \includegraphics[width=\linewidth]{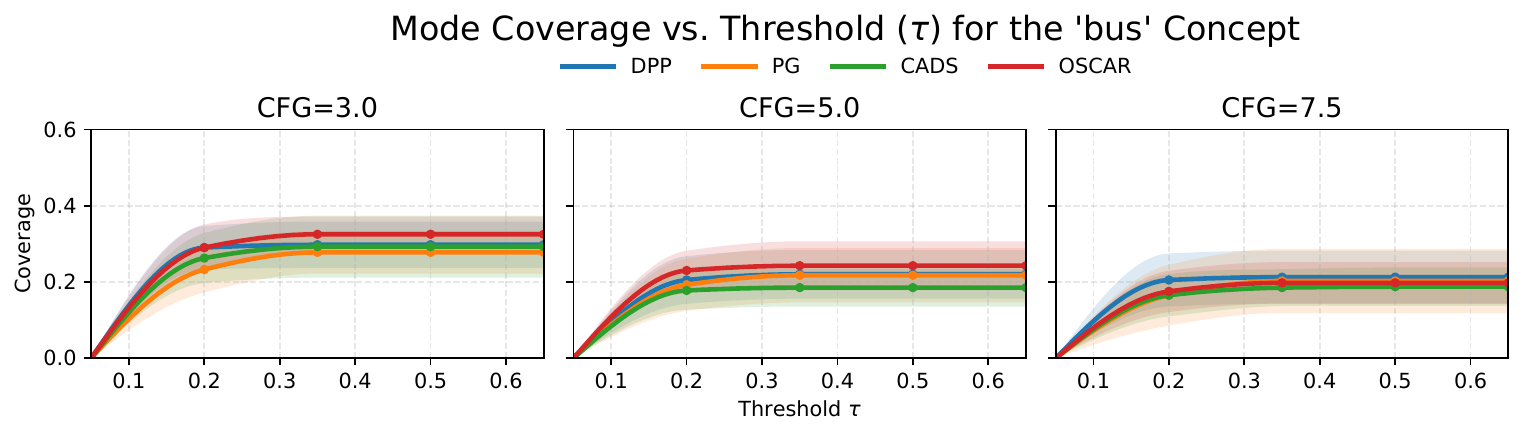}
        \caption{Mode Coverage vs. Threshold ($\tau$) for `bus'}
        \label{fig:coverage_bus}
    \end{subfigure}
    \hfill
    \begin{subfigure}{0.8\textwidth}
        \centering
        \includegraphics[width=\linewidth]{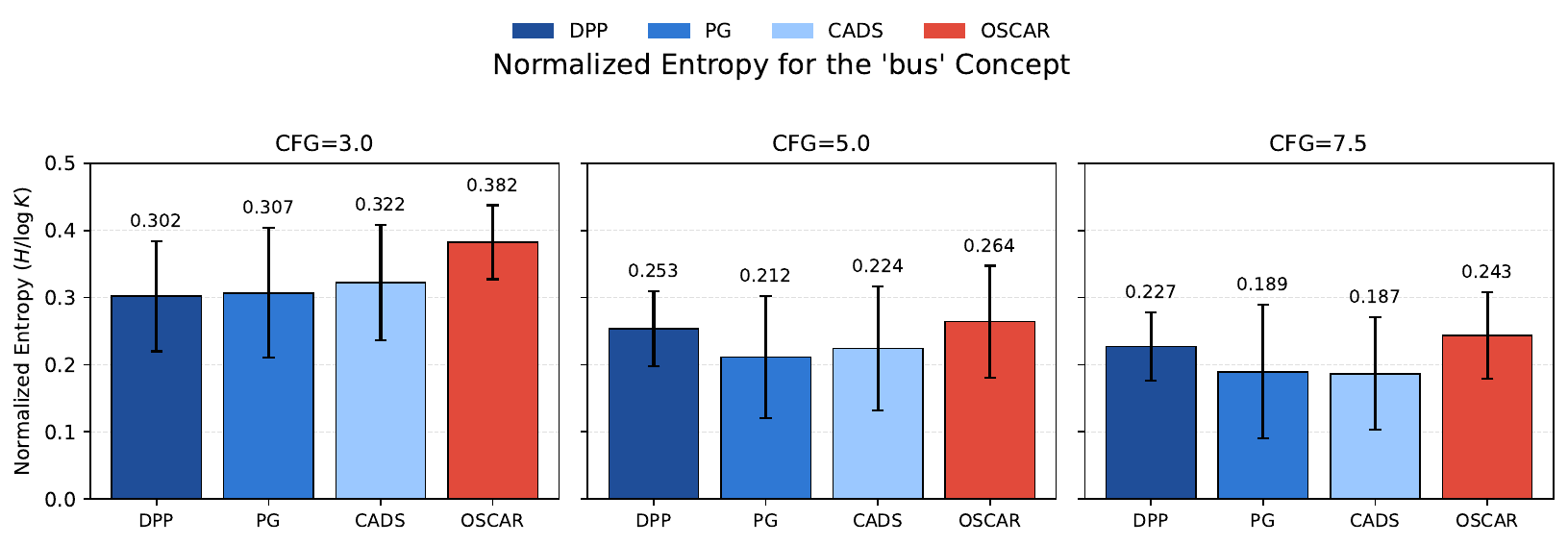}
        \caption{Normalized Entropy for `bus'}
        \label{fig:entropy_bus}
    \end{subfigure}
        \caption{Mode Coverage and Normalized Entropy for the `bus' concept. The results confirm our method's superior ability to both discover more intra-class modes and distribute samples more uniformly among them.}
        \label{fig:appendix_coverage_bus}
\end{figure}

\begin{figure}[h!]
    \centering

    \begin{subfigure}{0.8\textwidth}
        \centering
        \includegraphics[width=\linewidth]{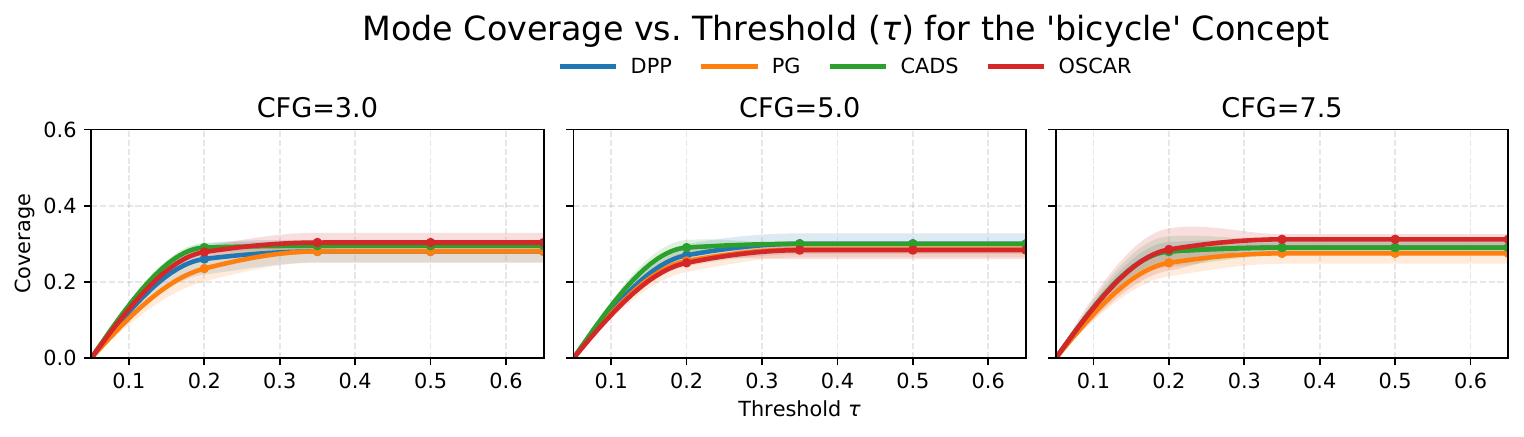}
        \caption{Mode Coverage vs. Threshold ($\tau$) for `bicycle'}
        \label{fig:coverage_bicycle}
    \end{subfigure}
    \hfill
    \begin{subfigure}{0.8\textwidth}
        \centering
        \includegraphics[width=\linewidth]{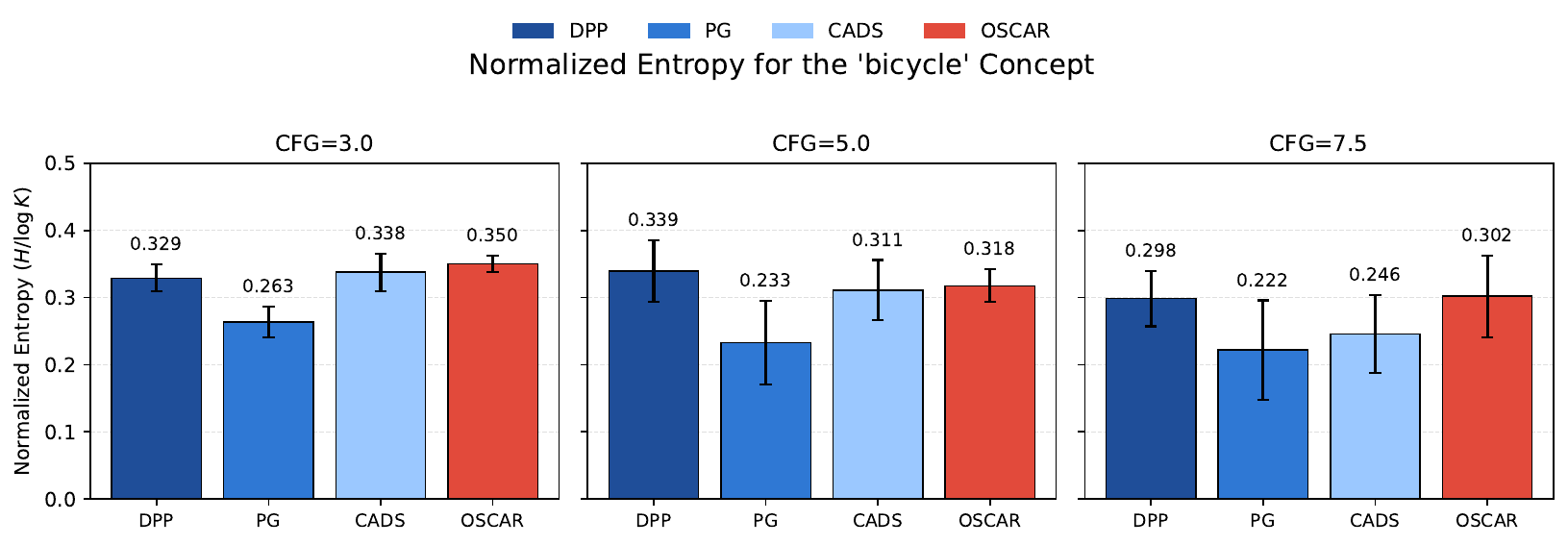}
        \caption{Normalized Entropy for `bicycle'}
        \label{fig:entropy_bicycle}
    \end{subfigure}
        \caption{Mode Coverage and Normalized Entropy for the `bicycle' concept, reinforcing the consistent outperformance of our method.}
        \label{fig:appendix_coverage_bicycle}
\end{figure}

\subsection{Quantitative Results for Additional Concepts}
\label{sec:add_quantitative}

To further validate the generalizability of our method's performance, this section provides a comprehensive extension of the main experimental results. For the sake of brevity and clarity, the quantitative and qualitative comparisons presented in the main body of the paper focused primarily on the ``truck'' concept. Here, we present the same set of comparisons for several additional concepts from the COCO dataset, with detailed quantitative results for the ``bus" and ``bicycle" concepts shown in Table~\ref{tab:bus_all_metrics_comparison} and Table~\ref{tab:bicycle_all_metrics_comparison}. These extended results demonstrate that the conclusions drawn in the main paper are not specific to a single class, and confirm that our method's advantages in enhancing diversity while preserving quality hold across a wider range of semantic categories.

\subsubsection{Classification Prompts Example}
To ensure our evaluation is robust and not biased by specific prompt phrasing, we use a set of synonymous prompts for each tested concept. The following listing shows an excerpt from our prompt file, illustrating the structure used for the ``truck" and ``bus" concepts discussed in the main paper.

\begin{tcolorbox}[
    enhanced,
    title={Example of the JSON structure used to store prompts for different concepts},
    label={lst:calss-truck-bus-prompts},
    colback=gray!5,
    colframe=gray!75,
    fonttitle=\bfseries,
    colbacktitle=gray!20,
    coltitle=black
]

\begin{lstlisting}[
    language=json,
    basicstyle=\small\ttfamily, % 使用等宽字体 (关键!)
    columns=flexible,           % 调整字符间距，防止过宽
    keepspaces=true,            % 保持缩进
    aboveskip=0pt,              % 移除顶部多余空白
    belowskip=0pt               % 移除底部多余空白
]
"truck": [
    "a truck",
    "a photo of a truck",
    "a photo of one truck"
],
"bus": [
    "a bus",
    "a photo of a bus",
    "a photo of one bus"
]
\end{lstlisting}
\end{tcolorbox}


\begin{table*}[h!]
\centering
\caption{
    Comprehensive quantitative comparison of our method against all baselines for the ``bus" concept. The results demonstrate the performance of different methods across various guidance scales and evaluation metrics.
}
\label{tab:bus_all_metrics_comparison}
\footnotesize
\setlength{\tabcolsep}{5pt}
\renewcommand{\arraystretch}{0.8}
\begin{tabularx}{\textwidth}{l@{\extracolsep{\fill}}cccccc}
\toprule
& \multicolumn{3}{c}{\textbf{Guidance Scale (CFG)}} & \multicolumn{3}{c}{\textbf{Guidance Scale (CFG)}} \\
\cmidrule(lr){2-4}\cmidrule(lr){5-7}
\textbf{Method} & $3.0$ & $5.0$ & $7.5$ & $3.0$ & $5.0$ & $7.5$ \\
\midrule
& \multicolumn{3}{c}{\textbf{Brisque} $\downarrow$} & \multicolumn{3}{c}{$\mathbf{1-\text{MS-SSIM}}$(\%) $\uparrow$} \\
\cmidrule(lr){2-4}\cmidrule(lr){5-7}
FM-SD3.5  & 24.92 $\pm$ 2.41 & 27.93 $\pm$ 2.18 & 32.07 $\pm$ 2.03 & 91.05 $\pm$ 0.58 & 91.21 $\pm$ 0.67 & 90.85 $\pm$ 0.99 \\
PG & 38.02 $\pm$ 6.00 & 34.76 $\pm$ 3.92 & 36.18 $\pm$ 2.45 & 90.89 $\pm$ 1.57 & 90.96 $\pm$ 2.07 & 90.96 $\pm$ 2.94 \\
CADS         & 23.49 $\pm$ 2.17 & 27.91 $\pm$ 2.51 & 34.93 $\pm$ 3.37 & 93.11 $\pm$ 0.77 & 92.80 $\pm$ 1.13 & 91.79 $\pm$ 1.89 \\
DPP      & 23.00 $\pm$ 3.12 & 28.31 $\pm$ 2.68 & 34.86 $\pm$ 2.52 & 90.82 $\pm$ 1.24 & 90.96 $\pm$ 1.93 & 90.10 $\pm$ 2.75 \\
\rowcolor{lightblue}\textbf{Ours}      & 22.71 $\pm$ 5.03 & 26.55 $\pm$ 4.72 & 33.55 $\pm$ 4.78 & 92.38 $\pm$ 1.23 & 91.91 $\pm$ 1.77 & 91.12 $\pm$ 2.26 \\
\midrule
& \multicolumn{3}{c}{\textbf{Vendi Score Pixel} $\uparrow$} & \multicolumn{3}{c}{\textbf{Vendi Score Inception} $\uparrow$} \\
\cmidrule(lr){2-4}\cmidrule(lr){5-7}
FM-SD3.5  & 3.49 $\pm$ 0.33 & 3.60 $\pm$ 0.44 & 3.62 $\pm$ 0.46 & 3.74 $\pm$ 0.51 & 3.28 $\pm$ 0.44 & 3.07 $\pm$ 0.42 \\
PG & 3.53 $\pm$ 0.74 & 3.68 $\pm$ 1.06 & 3.63 $\pm$ 1.12 & 4.06 $\pm$ 1.08 & 3.35 $\pm$ 0.89 & 3.21 $\pm$ 0.81 \\
CADS         & 3.95 $\pm$ 0.81 & 3.93 $\pm$ 1.13 & 3.87 $\pm$ 1.24 & 3.91 $\pm$ 1.11 & 3.40 $\pm$ 0.96 & 3.27 $\pm$ 0.85 \\
DPP      & 3.59 $\pm$ 0.80 & 3.58 $\pm$ 1.09 & 3.46 $\pm$ 1.10 & 3.78 $\pm$ 1.40 & 3.31 $\pm$ 1.31 & 3.13 $\pm$ 1.15 \\
\rowcolor{lightblue}\textbf{Ours}      & 4.08 $\pm$ 0.80 & 3.98 $\pm$ 1.02 & 3.78 $\pm$ 1.05 & 4.03 $\pm$ 1.12 & 3.47 $\pm$ 0.99 & 3.32 $\pm$ 0.92 \\
\midrule
& \multicolumn{3}{c}{\textbf{CLIP-IQA} $\uparrow$} & \multicolumn{3}{c}{\textbf{CLIP Score} $\uparrow$} \\
\cmidrule(lr){2-4}\cmidrule(lr){5-7}
FM-SD3.5  & 6.08 $\pm$ 0.32 & 6.06 $\pm$ 0.37 & 6.14 $\pm$ 0.37 & 28.08 $\pm$ 0.17 & 27.96 $\pm$ 0.18 & 27.92 $\pm$ 0.17 \\
PG & 6.11 $\pm$ 0.59 & 6.20 $\pm$ 0.49 & 6.20 $\pm$ 0.47 & 28.05 $\pm$ 0.38 & 27.68 $\pm$ 0.36 & 27.72 $\pm$ 0.30 \\
CADS        & 6.23 $\pm$ 0.53 & 6.21 $\pm$ 0.49 & 6.19 $\pm$ 0.49 & 27.81 $\pm$ 0.42 & 27.65 $\pm$ 0.34 & 27.76 $\pm$ 0.25 \\
DPP     & 6.09 $\pm$ 0.58 & 6.09 $\pm$ 0.54 & 6.04 $\pm$ 0.48 & 28.28 $\pm$ 0.40 & 28.12 $\pm$ 0.33 & 28.17 $\pm$ 0.37 \\
\rowcolor{lightblue}\textbf{Ours}        & 6.24 $\pm$ 0.58 & 6.26 $\pm$ 0.52 & 6.27 $\pm$ 0.46 & 27.94 $\pm$ 0.37 & 27.84 $\pm$ 0.36 & 27.91 $\pm$ 0.37 \\
\midrule
& \multicolumn{3}{c}{\textbf{FID} $\downarrow$} & \multicolumn{3}{c}{\textbf{Image Reward} $\uparrow$} \\
\cmidrule(lr){2-4}\cmidrule(lr){5-7}
FM-SD3.5  & 113.25 $\pm$ 1.60 & 111.22 $\pm$ 1.28 & 110.61 $\pm$ 1.13 & 0.26 $\pm$ 0.31 & 0.38 $\pm$ 0.32 & 0.42 $\pm$ 0.26 \\
PG & 116.18 $\pm$ 2.93 & 114.41 $\pm$ 2.35 & 113.71 $\pm$ 2.07 & 0.23 $\pm$ 0.48 & 0.39 $\pm$ 0.31 & 0.44 $\pm$ 0.29 \\
CADS        & 115.37 $\pm$ 2.53 & 114.40 $\pm$ 2.12 & 114.04 $\pm$ 1.75 & 0.32 $\pm$ 0.32 & 0.43 $\pm$ 0.24 & 0.48 $\pm$ 0.23 \\
DPP     & 114.65 $\pm$ 3.87 & 113.39 $\pm$ 3.18 & 112.66 $\pm$ 2.61 & 0.22 $\pm$ 0.44 & 0.31 $\pm$ 0.35 & 0.35 $\pm$ 0.31 \\
\rowcolor{lightblue}\textbf{Ours}        & 114.33 $\pm$ 3.50 & 112.44 $\pm$ 2.38 & 111.98 $\pm$ 2.50 & 0.31 $\pm$ 0.34 & 0.42 $\pm$ 0.30 & 0.46 $\pm$ 0.30 \\
\bottomrule
\end{tabularx}
\end{table*}

\begin{table*}[h!]
\centering
\caption{
    Comprehensive quantitative comparison of our method against all baselines for the ``bicycle" concept. The results demonstrate our method's performance in enhancing diversity while maintaining fidelity and prompt alignment across different guidance scales.
}
\label{tab:bicycle_all_metrics_comparison}
\footnotesize
\setlength{\tabcolsep}{5pt}
\renewcommand{\arraystretch}{0.8}
\begin{tabularx}{\textwidth}{l@{\extracolsep{\fill}}cccccc}
\toprule
& \multicolumn{3}{c}{\textbf{Guidance Scale (CFG)}} & \multicolumn{3}{c}{\textbf{Guidance Scale (CFG)}} \\
\cmidrule(lr){2-4}\cmidrule(lr){5-7}
\textbf{Method} & $3.0$ & $5.0$ & $7.5$ & $3.0$ & $5.0$ & $7.5$ \\
\midrule
& \multicolumn{3}{c}{\textbf{Brisque} $\downarrow$} & \multicolumn{3}{c}{$\mathbf{1-\text{MS-SSIM}}$(\%) $\uparrow$} \\
\cmidrule(lr){2-4}\cmidrule(lr){5-7}
FM-SD3.5  & 25.84 $\pm$ 1.97 & 27.03 $\pm$ 2.04 & 30.76 $\pm$ 1.89 & 88.40 $\pm$ 1.06 & 89.33 $\pm$ 1.34 & 89.57 $\pm$ 1.51 \\
PG & 49.33 $\pm$ 2.20 & 63.71 $\pm$ 1.86 & 69.05 $\pm$ 4.21 & 86.83 $\pm$ 2.12 & 88.93 $\pm$ 2.30 & 89.26 $\pm$ 2.90 \\
CADS         & 28.48 $\pm$ 3.58 & 32.54 $\pm$ 2.67 & 34.00 $\pm$ 2.58 & 91.48 $\pm$ 2.59 & 92.99 $\pm$ 2.24 & 92.51 $\pm$ 2.35 \\
DPP      & 21.19 $\pm$ 2.77 & 23.75 $\pm$ 2.40 & 26.44 $\pm$ 3.99 & 87.02 $\pm$ 4.41 & 89.19 $\pm$ 3.95 & 90.11 $\pm$ 3.67 \\
\rowcolor{lightblue}\textbf{Ours}      & 20.84 $\pm$ 3.25 & 24.84 $\pm$ 3.09 & 26.14 $\pm$ 2.80 & 90.74 $\pm$ 2.25 & 91.81 $\pm$ 2.27 & 91.48 $\pm$ 1.49 \\
\midrule
& \multicolumn{3}{c}{\textbf{Vendi Score Pixel} $\uparrow$} & \multicolumn{3}{c}{\textbf{Vendi Score Inception} $\uparrow$} \\
\cmidrule(lr){2-4}\cmidrule(lr){5-7}
FM-SD3.5  & 2.26 $\pm$ 0.24 & 2.38 $\pm$ 0.25 & 2.46 $\pm$ 0.25 & 3.59 $\pm$ 0.32 & 3.10 $\pm$ 0.30 & 2.90 $\pm$ 0.25 \\
PG & 2.06 $\pm$ 0.11 & 2.18 $\pm$ 0.07 & 2.29 $\pm$ 0.05 & 3.50 $\pm$ 0.07 & 3.06 $\pm$ 0.10 & 2.97 $\pm$ 0.07 \\
CADS         & 2.30 $\pm$ 0.04 & 2.37 $\pm$ 0.06 & 2.48 $\pm$ 0.04 & 3.76 $\pm$ 0.28 & 3.19 $\pm$ 0.10 & 2.94 $\pm$ 0.07 \\
DPP      & 2.28 $\pm$ 0.02 & 2.34 $\pm$ 0.01 & 2.43 $\pm$ 0.03 & 3.83 $\pm$ 0.12 & 3.15 $\pm$ 0.18 & 2.86 $\pm$ 0.06 \\
\rowcolor{lightblue}\textbf{Ours}      & 2.55 $\pm$ 0.17 & 2.51 $\pm$ 0.07 & 2.61 $\pm$ 0.03 & 3.72 $\pm$ 0.13 & 3.24 $\pm$ 0.15 & 2.99 $\pm$ 0.11 \\
\midrule
& \multicolumn{3}{c}{\textbf{CLIP-IQA} $\uparrow$} & \multicolumn{3}{c}{\textbf{CLIP Score} $\uparrow$} \\
\cmidrule(lr){2-4}\cmidrule(lr){5-7}
FM-SD3.5  & 0.43 $\pm$ 0.41 & 0.51 $\pm$ 0.38 & 0.48 $\pm$ 0.38 & 29.53 $\pm$ 0.30 & 29.26 $\pm$ 0.29 & 29.06 $\pm$ 0.26 \\
PG & 0.42 $\pm$ 0.42 & 0.51 $\pm$ 0.33 & 0.49 $\pm$ 0.31 & 29.86 $\pm$ 0.22 & 29.09 $\pm$ 0.24 & 28.77 $\pm$ 0.13 \\
CADS        & 0.47 $\pm$ 0.35 & 0.52 $\pm$ 0.32 & 0.53 $\pm$ 0.34 & 29.43 $\pm$ 0.26 & 29.10 $\pm$ 0.27 & 28.85 $\pm$ 0.15 \\
DPP     & 0.42 $\pm$ 0.38 & 0.55 $\pm$ 0.32 & 0.62 $\pm$ 0.29 & 29.85 $\pm$ 0.16 & 29.56 $\pm$ 0.15 & 29.45 $\pm$ 0.11 \\
\rowcolor{lightblue}\textbf{Ours}        & 0.48 $\pm$ 0.38 & 0.55 $\pm$ 0.33 & 0.58 $\pm$ 0.31 & 29.81 $\pm$ 0.17 & 29.31 $\pm$ 0.22 & 28.97 $\pm$ 0.06 \\
\midrule
& \multicolumn{3}{c}{\textbf{FID} $\downarrow$} & \multicolumn{3}{c}{\textbf{Image Reward} $\uparrow$} \\
\cmidrule(lr){2-4}\cmidrule(lr){5-7}
FM-SD3.5  & 149.21 $\pm$ 1.85 & 148.28 $\pm$ 1.45 & 147.06 $\pm$ 3.90 & 6.25 $\pm$ 0.70 & 6.35 $\pm$ 0.58 & 6.29 $\pm$ 0.66 \\
PG & 149.34 $\pm$ 3.05 & 151.45 $\pm$ 1.68 & 149.44 $\pm$ 4.26 & 6.18 $\pm$ 0.70 & 6.27 $\pm$ 0.68 & 6.22 $\pm$ 0.68 \\
CADS        & 149.27 $\pm$ 1.59 & 149.06 $\pm$ 1.20 & 149.24 $\pm$ 3.90 & 6.28 $\pm$ 0.67 & 6.31 $\pm$ 0.27 & 6.25 $\pm$ 0.70 \\
DPP     & 146.59 $\pm$ 2.06 & 148.14 $\pm$ 0.93 & 145.49 $\pm$ 3.87 & 6.14 $\pm$ 0.69 & 6.27 $\pm$ 0.69 & 6.29 $\pm$ 0.68 \\
\rowcolor{lightblue}\textbf{Ours}        & 146.36 $\pm$ 3.48 & 148.50 $\pm$ 1.33 & 146.99 $\pm$ 4.80 & 6.30 $\pm$ 0.54 & 6.44 $\pm$ 0.51 & 6.42 $\pm$ 0.63 \\
\bottomrule
\end{tabularx}
\end{table*}

\subsubsection{Prompt-level analysis within a single concept}
We also notice that, even within a fixed semantic concept, different textual prompts can induce large variations in standard evaluation metrics. When aggregating over all prompts belonging to the same concept, this high intra-concept variance tends to blur the performance differences between methods and makes the results harder to interpret. To better disentangle these effects, in this part we therefore report metrics for two representative prompts under the ``truck'' concept, namely ``a truck'' and ``a photo of a truck'' (see Tab.~\ref{tab:all_metrics_comparison_a_truck} and Tab.~\ref{tab:all_metrics_comparison_photo_truck}). The former is intentionally short and under-specified, while the latter provides a slightly richer semantic context. Comparing methods under these two prompts side by side allows us to more clearly separate the impact of our diversity-enhancing control from the prompt-induced variability, and we observe similar behaviors for other concepts.

\begin{table*}[h]
\centering
\caption{\small
Quantitative comparison of our method against baselines across different CFG levels for the ``a photo of a truck'' prompt under the ``truck'' concept. Our method consistently improves diversity metrics while maintaining competitive or better image quality and alignment.
}
\label{tab:all_metrics_comparison_photo_truck}
\footnotesize
\setlength{\tabcolsep}{5pt}
\renewcommand{\arraystretch}{0.8}
\begin{tabularx}{\textwidth}{l@{\extracolsep{\fill}}cccccc}
\toprule
& \multicolumn{3}{c}{\textbf{Guidance Scale (CFG)}} & \multicolumn{3}{c}{\textbf{Guidance Scale (CFG)}} \\
\cmidrule(lr){2-4}\cmidrule(lr){5-7}
\textbf{Method} & $3.0$ & $5.0$ & $7.5$ & $3.0$ & $5.0$ & $7.5$ \\
\midrule
& \multicolumn{3}{c}{\textbf{Brisque} $\downarrow$} & \multicolumn{3}{c}{$\mathbf{1-\text{MS-SSIM}}$(\%) $\uparrow$} \\
\cmidrule(lr){2-4}\cmidrule(lr){5-7}
PG & 34.45 $\pm$ 5.69 & 34.18 $\pm$ 1.05 & 35.77 $\pm$ 3.94 & 89.29 $\pm$ 1.48 & 91.64 $\pm$ 1.31 & 91.60 $\pm$ 1.34 \\
CADS         & 31.10 $\pm$ 2.25 & 33.65 $\pm$ 1.12 & 37.59 $\pm$ 1.42 & 88.25 $\pm$ 0.77 & 90.19 $\pm$ 0.55 & 92.01 $\pm$ 0.55 \\
DPP      & 23.84 $\pm$ 1.24 & 29.56 $\pm$ 1.63 & 36.76 $\pm$ 1.11 & 89.82 $\pm$ 1.19 & 91.11 $\pm$ 1.34 & 91.26 $\pm$ 1.15 \\
\rowcolor{lightblue}\textbf{Ours}       & 26.98 $\pm$ 0.95 & 29.85 $\pm$ 0.85 & 37.57 $\pm$ 1.51 & 90.75 $\pm$ 2.15 & 91.80 $\pm$ 2.51 & 91.80 $\pm$ 1.79 \\
\midrule
& \multicolumn{3}{c}{\textbf{Vendi Score Pixel} $\uparrow$} & \multicolumn{3}{c}{\textbf{Vendi Score Inception} $\uparrow$} \\
\cmidrule(lr){2-4}\cmidrule(lr){5-7}
PG & 3.33 $\pm$ 0.25 & 3.56 $\pm$ 0.27 & 3.64 $\pm$ 0.24 & 5.53 $\pm$ 0.13 & 4.92 $\pm$ 0.20 & 4.68 $\pm$ 0.12 \\
CADS         & 3.71 $\pm$ 0.24 & 3.72 $\pm$ 0.17 & 3.73 $\pm$ 0.11 & 5.90 $\pm$ 0.17 & 4.72 $\pm$ 0.24 & 4.43 $\pm$ 0.18 \\
DPP      & 3.65 $\pm$ 0.14 & 3.83 $\pm$ 0.17 & 3.79 $\pm$ 0.13 & 5.86 $\pm$ 0.34 & 4.93 $\pm$ 0.16 & 4.30 $\pm$ 0.29 \\
\rowcolor{lightblue}\textbf{Ours}       & 3.79 $\pm$ 0.11 & 3.86 $\pm$ 0.15 & 4.01 $\pm$ 0.15 & 5.89 $\pm$ 0.34 & 5.12 $\pm$ 0.35 & 4.77 $\pm$ 0.19 \\
\midrule
& \multicolumn{3}{c}{\textbf{FID} $\downarrow$} & \multicolumn{3}{c}{\textbf{CLIP Score} $\uparrow$} \\
\cmidrule(lr){2-4}\cmidrule(lr){5-7}
PG & 143.51 $\pm$ 3.35 & 141.83 $\pm$ 1.83 & 143.63 $\pm$ 2.18 & 27.82 $\pm$ 0.21 & 27.45 $\pm$ 0.20 & 27.46 $\pm$ 0.13 \\
CADS        & 130.65 $\pm$ 1.24 & 129.41 $\pm$ 1.58 & 129.38 $\pm$ 0.88 & 27.54 $\pm$ 0.15 & 27.22 $\pm$ 0.09 & 27.04 $\pm$ 0.11 \\
DPP     & 143.87 $\pm$ 3.69 & 139.43 $\pm$ 1.63 & 141.05 $\pm$ 3.98 & 27.63 $\pm$ 0.09 & 27.14 $\pm$ 0.10 & 27.01 $\pm$ 0.12 \\
\rowcolor{lightblue}\textbf{Ours}        & 130.13 $\pm$ 0.96 & 128.18 $\pm$ 0.20 & 126.89 $\pm$ 1.58 & 27.75 $\pm$ 0.04 & 27.50 $\pm$ 0.11 & 27.49 $\pm$ 0.12 \\
\bottomrule
\end{tabularx}
\end{table*}

\begin{table*}[h]
\centering
\caption{\small
Quantitative comparison of our method against baselines across different CFG levels for the ``a truck'' prompt under the ``truck'' concept. Our method consistently improves diversity metrics while maintaining competitive or better image quality and alignment.
}
\label{tab:all_metrics_comparison_a_truck}
\footnotesize
\setlength{\tabcolsep}{5pt}
\renewcommand{\arraystretch}{0.8}
\begin{tabularx}{\textwidth}{l@{\extracolsep{\fill}}cccccc}
\toprule
& \multicolumn{3}{c}{\textbf{Guidance Scale (CFG)}} & \multicolumn{3}{c}{\textbf{Guidance Scale (CFG)}} \\
\cmidrule(lr){2-4}\cmidrule(lr){5-7}
\textbf{Method} & $3.0$ & $5.0$ & $7.5$ & $3.0$ & $5.0$ & $7.5$ \\
\midrule
& \multicolumn{3}{c}{\textbf{Brisque} $\downarrow$} & \multicolumn{3}{c}{$\mathbf{1-\text{MS-SSIM}}$(\%) $\uparrow$} \\
\cmidrule(lr){2-4}\cmidrule(lr){5-7}
PG & 43.88 $\pm$ 3.47 & 36.40 $\pm$ 3.20 & 40.30 $\pm$ 2.48 & 86.44 $\pm$ 2.63 & 86.08 $\pm$ 1.96 & 83.81 $\pm$ 1.97 \\
CADS & 22.52 $\pm$ 1.61 & 23.59 $\pm$ 1.61 & 28.37 $\pm$ 1.48 & 86.29 $\pm$ 2.15 & 85.71 $\pm$ 2.26 & 84.46 $\pm$ 2.65 \\
DPP & 22.18 $\pm$ 1.57 & 25.77 $\pm$ 1.33 & 32.05 $\pm$ 0.92 & 88.33 $\pm$ 0.91 & 87.47 $\pm$ 1.86 & 85.58 $\pm$ 1.32 \\
\rowcolor{lightblue}\textbf{Ours} & 19.50 $\pm$ 2.04 & 22.31 $\pm$ 1.62 & 27.65 $\pm$ 1.72 & 90.20 $\pm$ 1.94 & 88.60 $\pm$ 1.24 & 87.73 $\pm$ 1.38 \\
\midrule
& \multicolumn{3}{c}{\textbf{Vendi Score Pixel} $\uparrow$} & \multicolumn{3}{c}{\textbf{Vendi Score Inception} $\uparrow$} \\
\cmidrule(lr){2-4}\cmidrule(lr){5-7}
PG & 4.63 $\pm$ 0.27 & 4.21 $\pm$ 0.13 & 4.09 $\pm$ 0.07 & 2.49 $\pm$ 0.17 & 2.42 $\pm$ 0.14 & 2.38 $\pm$ 0.11 \\
CADS & 4.63 $\pm$ 0.41 & 3.95 $\pm$ 0.24 & 3.63 $\pm$ 0.11 & 2.69 $\pm$ 0.16 & 2.53 $\pm$ 0.09 & 2.47 $\pm$ 0.08 \\
DPP & 4.61 $\pm$ 0.26 & 3.97 $\pm$ 0.20 & 3.72 $\pm$ 0.07 & 2.59 $\pm$ 0.07 & 2.42 $\pm$ 0.06 & 2.31 $\pm$ 0.06 \\
\rowcolor{lightblue}\textbf{Ours} & 4.79 $\pm$ 0.28 & 4.29 $\pm$ 0.16 & 4.14 $\pm$ 0.19 & 2.82 $\pm$ 0.11 & 2.66 $\pm$ 0.11 & 2.51 $\pm$ 0.06 \\
\midrule
& \multicolumn{3}{c}{\textbf{FID} $\downarrow$} & \multicolumn{3}{c}{\textbf{CLIP Score} $\uparrow$} \\
\cmidrule(lr){2-4}\cmidrule(lr){5-7}
PG & 142.82 $\pm$ 3.12 & 142.88 $\pm$ 2.71 & 142.07 $\pm$ 5.60 & 27.65 $\pm$ 0.15 & 26.70 $\pm$ 0.12 & 26.48 $\pm$ 0.17 \\
CADS & 126.75 $\pm$ 1.23 & 125.96 $\pm$ 0.52 & 125.17 $\pm$ 0.64 & 27.19 $\pm$ 0.09 & 26.57 $\pm$ 0.07 & 26.57 $\pm$ 0.14 \\
DPP & 139.64 $\pm$ 4.07 & 138.31 $\pm$ 2.76 & 138.62 $\pm$ 4.79 & 27.82 $\pm$ 0.10 & 27.30 $\pm$ 0.10 & 26.88 $\pm$ 0.09 \\
\rowcolor{lightblue}\textbf{Ours} & 128.80 $\pm$ 1.27 & 127.64 $\pm$ 1.30 & 126.17 $\pm$ 1.45 & 27.65 $\pm$ 0.12 & 27.18 $\pm$ 0.11 & 26.80 $\pm$ 0.10 \\
\bottomrule
\end{tabularx}
\end{table*}

In addition, to further demonstrate that our conclusions are not specific to a single concept, we report prompt-level results on three other concepts,``apple'', ``suitcase'', and ``pizza'', using a single representative prompt for each concept (see Tab.~\ref{tab:all_metrics_comparison_apple}, Tab.~\ref{tab:all_metrics_comparison_suitcase}, and Tab.~\ref{tab:all_metrics_comparison_pizza}). On these concepts, our method achieves consistently larger improvements on the diversity metrics (Vendi Score Pixel and Vendi Score Inception) compared to all baselines, while preserving comparable image quality.

\begin{table*}[h]
\centering
\caption{\small
Quantitative comparison of our method against baselines across different CFG levels for the ``a photo of a apple'' prompt under the ``apple'' concept. 
}
\label{tab:all_metrics_comparison_apple}
\footnotesize
\setlength{\tabcolsep}{5pt}
\renewcommand{\arraystretch}{0.8}
\begin{tabularx}{\textwidth}{l@{\extracolsep{\fill}}cccccc}
\toprule
& \multicolumn{3}{c}{\textbf{Guidance Scale (CFG)}} & \multicolumn{3}{c}{\textbf{Guidance Scale (CFG)}} \\
\cmidrule(lr){2-4}\cmidrule(lr){5-7}
\textbf{Method} & $3.0$ & $5.0$ & $7.5$ & $3.0$ & $5.0$ & $7.5$ \\
\midrule
& \multicolumn{3}{c}{\textbf{Brisque} $\downarrow$} & \multicolumn{3}{c}{$\mathbf{1-\text{MS-SSIM}}$(\%) $\uparrow$} \\
\cmidrule(lr){2-4}\cmidrule(lr){5-7}
PG & 20.50 $\pm$ 0.65 & 25.12 $\pm$ 1.63 & 30.74 $\pm$ 3.28 & 56.00 $\pm$ 2.08 & 55.55 $\pm$ 2.53 & 56.60 $\pm$ 1.89 \\
CADS & 23.76 $\pm$ 4.87 & 25.64 $\pm$ 5.68 & 30.82 $\pm$ 2.27 & 58.16 $\pm$ 3.29 & 58.92 $\pm$ 3.01 & 59.89 $\pm$ 3.40 \\
DPP & 19.69 $\pm$ 2.54 & 26.36 $\pm$ 2.72 & 31.08 $\pm$ 4.26 & 56.97 $\pm$ 3.76 & 57.45 $\pm$ 4.96 & 56.51 $\pm$ 5.69 \\
\rowcolor{lightblue}\textbf{Ours} & 17.71 $\pm$ 1.81 & 25.53 $\pm$ 0.51 & 27.96 $\pm$ 1.89 & 57.97 $\pm$ 3.56 & 55.79 $\pm$ 3.00 & 57.56 $\pm$ 2.90 \\
\midrule
& \multicolumn{3}{c}{\textbf{Vendi Score Pixel} $\uparrow$} & \multicolumn{3}{c}{\textbf{Vendi Score Inception} $\uparrow$} \\
\cmidrule(lr){2-4}\cmidrule(lr){5-7}
PG & 2.06 $\pm$ 0.06 & 2.04 $\pm$ 0.09 & 2.12 $\pm$ 0.11 & 2.12 $\pm$ 0.08 & 1.91 $\pm$ 0.02 & 1.99 $\pm$ 0.18 \\
CADS & 1.87 $\pm$ 0.30 & 1.78 $\pm$ 0.22 & 1.77 $\pm$ 0.22 & 2.19 $\pm$ 0.17 & 1.96 $\pm$ 0.03 & 1.93 $\pm$ 0.12 \\
DPP & 1.87 $\pm$ 0.27 & 1.85 $\pm$ 0.22 & 1.86 $\pm$ 0.25 & 2.06 $\pm$ 0.09 & 2.03 $\pm$ 0.03 & 1.99 $\pm$ 0.13 \\
\rowcolor{lightblue}\textbf{Ours} & 2.19 $\pm$ 0.06 & 2.13 $\pm$ 0.08 & 2.22 $\pm$ 0.08 & 2.27 $\pm$ 0.04 & 2.07 $\pm$ 0.08 & 2.07 $\pm$ 0.13 \\
\midrule
& \multicolumn{3}{c}{\textbf{FID} $\downarrow$} & \multicolumn{3}{c}{\textbf{CLIP Score} $\uparrow$} \\
\cmidrule(lr){2-4}\cmidrule(lr){5-7}
PG & 152.83 $\pm$ 0.83 & 149.68 $\pm$ 0.12 & 151.91 $\pm$ 0.70 & 31.84 $\pm$ 0.09 & 31.79 $\pm$ 0.07 & 31.83 $\pm$ 0.23 \\
CADS & 154.05 $\pm$ 3.25 & 151.76 $\pm$ 1.39 & 150.88 $\pm$ 2.44 & 31.79 $\pm$ 0.01 & 31.81 $\pm$ 0.15 & 31.69 $\pm$ 0.07 \\
DPP & 151.97 $\pm$ 1.72 & 149.80 $\pm$ 1.30 & 149.99 $\pm$ 2.25 & 31.87 $\pm$ 0.09 & 31.84 $\pm$ 0.12 & 31.88 $\pm$ 0.14 \\
\rowcolor{lightblue}\textbf{Ours} & 152.95 $\pm$ 1.00 & 150.35 $\pm$ 2.84 & 151.53 $\pm$ 1.48 & 31.88 $\pm$ 0.12 & 31.81 $\pm$ 0.14 & 31.87 $\pm$ 0.19 \\
\bottomrule
\end{tabularx}
\end{table*}

\begin{table*}[h]
\centering
\caption{\small
Quantitative comparison of our method against baselines across different CFG levels for the ``a photo of a suitcase'' prompt under the ``suitcase'' concept. 
}
\label{tab:all_metrics_comparison_suitcase}
\footnotesize
\setlength{\tabcolsep}{5pt}
\renewcommand{\arraystretch}{0.8}
\begin{tabularx}{\textwidth}{l@{\extracolsep{\fill}}cccccc}
\toprule
& \multicolumn{3}{c}{\textbf{Guidance Scale (CFG)}} & \multicolumn{3}{c}{\textbf{Guidance Scale (CFG)}} \\
\cmidrule(lr){2-4}\cmidrule(lr){5-7}
\textbf{Method} & $3.0$ & $5.0$ & $7.5$ & $3.0$ & $5.0$ & $7.5$ \\
\midrule
& \multicolumn{3}{c}{\textbf{Brisque} $\downarrow$} & \multicolumn{3}{c}{$\mathbf{1-\text{MS-SSIM}}$(\%) $\uparrow$} \\
\cmidrule(lr){2-4}\cmidrule(lr){5-7}
PG & 58.07 $\pm$ 4.75 & 43.32 $\pm$ 2.48 & 36.44 $\pm$ 3.20 & 74.11 $\pm$ 4.74 & 75.69 $\pm$ 1.39 & 77.33 $\pm$ 1.73 \\
CADS & 54.98 $\pm$ 8.24 & 83.71 $\pm$ 2.57 & 100.60 $\pm$ 7.59 & 84.36 $\pm$ 1.97 & 87.02 $\pm$ 5.40 & 86.12 $\pm$ 3.97 \\
DPP & 27.66 $\pm$ 6.52 & 29.07 $\pm$ 1.85 & 33.48 $\pm$ 5.31 & 76.76 $\pm$ 5.98 & 78.10 $\pm$ 5.06 & 80.45 $\pm$ 2.46 \\
\rowcolor{lightblue}\textbf{Ours} & 20.38 $\pm$ 1.83 & 17.35 $\pm$ 1.51 & 14.97 $\pm$ 1.55 & 82.99 $\pm$ 3.49 & 83.36 $\pm$ 2.41 & 84.77 $\pm$ 1.85 \\
\midrule
& \multicolumn{3}{c}{\textbf{Vendi Score Pixel} $\uparrow$} & \multicolumn{3}{c}{\textbf{Vendi Score Inception} $\uparrow$} \\
\cmidrule(lr){2-4}\cmidrule(lr){5-7}
PG & 2.72 $\pm$ 0.20 & 2.83 $\pm$ 0.08 & 2.84 $\pm$ 0.14 & 5.53 $\pm$ 0.28 & 5.06 $\pm$ 0.30 & 4.79 $\pm$ 0.49 \\
CADS & 2.75 $\pm$ 0.29 & 2.93 $\pm$ 0.12 & 3.01 $\pm$ 0.11 & 5.50 $\pm$ 0.13 & 5.04 $\pm$ 0.51 & 4.46 $\pm$ 0.39 \\
DPP & 3.01 $\pm$ 0.29 & 3.26 $\pm$ 0.23 & 3.18 $\pm$ 0.10 & 5.17 $\pm$ 0.13 & 4.57 $\pm$ 0.13 & 4.24 $\pm$ 0.29 \\
\rowcolor{lightblue}\textbf{Ours} & 3.18 $\pm$ 0.36 & 3.43 $\pm$ 0.06 & 3.61 $\pm$ 0.04 & 5.73 $\pm$ 0.13 & 5.45 $\pm$ 0.27 & 4.87 $\pm$ 0.42 \\
\midrule
& \multicolumn{3}{c}{\textbf{FID} $\downarrow$} & \multicolumn{3}{c}{\textbf{CLIP Score} $\uparrow$} \\
\cmidrule(lr){2-4}\cmidrule(lr){5-7}
PG & 158.69 $\pm$ 3.64 & 156.19 $\pm$ 2.06 & 154.75 $\pm$ 1.88 & 30.12 $\pm$ 0.11 & 29.85 $\pm$ 0.34 & 29.76 $\pm$ 0.14 \\
CADS & 166.04 $\pm$ 4.10 & 161.58 $\pm$ 0.62 & 158.54 $\pm$ 4.39 & 29.37 $\pm$ 0.82 & 29.07 $\pm$ 0.36 & 29.10 $\pm$ 0.20 \\
DPP & 156.42 $\pm$ 2.06 & 156.26 $\pm$ 1.94 & 155.99 $\pm$ 3.09 & 30.33 $\pm$ 0.28 & 29.84 $\pm$ 0.15 & 29.72 $\pm$ 0.12 \\
\rowcolor{lightblue}\textbf{Ours} & 159.20 $\pm$ 2.90 & 153.91 $\pm$ 3.99 & 151.90 $\pm$ 1.63 & 29.54 $\pm$ 0.62 & 29.51 $\pm$ 0.19 & 29.59 $\pm$ 0.05 \\
\bottomrule
\end{tabularx}
\end{table*}

\begin{table*}[h]
\centering
\caption{\small
Quantitative comparison of our method against baselines across different CFG levels for the ``a photo of a pizza'' prompt under the ``pizza'' concept. 
}
\label{tab:all_metrics_comparison_pizza}
\footnotesize
\setlength{\tabcolsep}{5pt}
\renewcommand{\arraystretch}{0.8}
\begin{tabularx}{\textwidth}{l@{\extracolsep{\fill}}cccccc}
\toprule
& \multicolumn{3}{c}{\textbf{Guidance Scale (CFG)}} & \multicolumn{3}{c}{\textbf{Guidance Scale (CFG)}} \\
\cmidrule(lr){2-4}\cmidrule(lr){5-7}
\textbf{Method} & $3.0$ & $5.0$ & $7.5$ & $3.0$ & $5.0$ & $7.5$ \\
\midrule
& \multicolumn{3}{c}{\textbf{Brisque} $\downarrow$} & \multicolumn{3}{c}{$\mathbf{1-\text{MS-SSIM}}$(\%) $\uparrow$} \\
\cmidrule(lr){2-4}\cmidrule(lr){5-7}
PG & 31.07 $\pm$ 4.46 & 19.28 $\pm$ 3.16 & 21.19 $\pm$ 1.74 & 90.90 $\pm$ 0.95 & 91.04 $\pm$ 1.75 & 92.11 $\pm$ 1.14 \\
CADS & 18.23 $\pm$ 0.35 & 18.24 $\pm$ 1.18 & 25.06 $\pm$ 8.83 & 93.29 $\pm$ 0.73 & 91.18 $\pm$ 1.74 & 90.22 $\pm$ 2.64 \\
DPP & 20.36 $\pm$ 0.76 & 20.46 $\pm$ 0.41 & 32.35 $\pm$ 7.69 & 93.61 $\pm$ 0.50 & 92.77 $\pm$ 0.39 & 93.54 $\pm$ 1.04 \\
\rowcolor{lightblue}\textbf{Ours} & 18.13 $\pm$ 0.81 & 17.39 $\pm$ 0.20 & 27.25 $\pm$ 1.58 & 94.45 $\pm$ 0.69 & 93.76 $\pm$ 1.37 & 94.05 $\pm$ 1.69 \\
\midrule
& \multicolumn{3}{c}{\textbf{Vendi Score Pixel} $\uparrow$} & \multicolumn{3}{c}{\textbf{Vendi Score Inception} $\uparrow$} \\
\cmidrule(lr){2-4}\cmidrule(lr){5-7}
PG & 1.79 $\pm$ 0.01 & 1.99 $\pm$ 0.04 & 2.44 $\pm$ 0.10 & 2.31 $\pm$ 0.05 & 2.18 $\pm$ 0.11 & 2.66 $\pm$ 0.11 \\
CADS & 1.84 $\pm$ 0.04 & 1.86 $\pm$ 0.05 & 1.92 $\pm$ 0.08 & 2.30 $\pm$ 0.04 & 2.16 $\pm$ 0.08 & 2.33 $\pm$ 0.46 \\
DPP & 1.91 $\pm$ 0.04 & 2.11 $\pm$ 0.08 & 2.35 $\pm$ 0.13 & 2.36 $\pm$ 0.08 & 2.19 $\pm$ 0.11 & 2.60 $\pm$ 0.31 \\
\rowcolor{lightblue}\textbf{Ours} & 2.05 $\pm$ 0.02 & 2.24 $\pm$ 0.05 & 2.60 $\pm$ 0.06 & 2.88 $\pm$ 0.09 & 2.41 $\pm$ 0.10 & 2.70 $\pm$ 0.22  \\
\midrule
& \multicolumn{3}{c}{\textbf{FID} $\downarrow$} & \multicolumn{3}{c}{\textbf{CLIP Score} $\uparrow$} \\
\cmidrule(lr){2-4}\cmidrule(lr){5-7}
PG & 136.71 $\pm$ 2.75 & 136.36 $\pm$ 2.40 & 141.19 $\pm$ 2.80 & 30.32 $\pm$ 0.16 & 30.29 $\pm$ 0.07 & 29.80 $\pm$ 0.17 \\
CADS & 134.44 $\pm$ 0.66 & 135.62 $\pm$ 0.49 & 138.73 $\pm$ 3.63 & 29.69 $\pm$ 0.09 & 29.58 $\pm$ 0.34 & 29.32 $\pm$ 0.07 \\
DPP & 134.22 $\pm$ 3.54 & 135.69 $\pm$ 2.56 & 140.73 $\pm$ 4.90 & 29.86 $\pm$ 0.18 & 29.86 $\pm$ 0.11 & 29.27 $\pm$ 0.27 \\
\rowcolor{lightblue}\textbf{Ours} & 131.93 $\pm$ 1.74 & 132.46 $\pm$ 3.04 & 139.88 $\pm$ 2.13 & 29.89 $\pm$ 0.25 & 29.81 $\pm$ 0.16 & 29.44 $\pm$ 0.23 \\
\bottomrule
\end{tabularx}
\end{table*}

\subsection{Detailed Analysis of Attribute-Level Diversity}
\label{sec:appendix_dim_cim_details}

While global metrics like FID and Vendi Score are useful, they do not fully capture two critical aspects of a generative model's behavior: its inherent biases and its fine-grained controllability. To address this, we adopt the DIM/CIM framework proposed by \citet{teotia2025dimcim}. This framework is designed to evaluate the crucial trade-off between a model's spontaneous creativity and its ability to follow specific instructions.

To analyze the trade-off between spontaneous diversity and explicit control, we employ two complementary metrics: Default-mode Diversity (DIM) and Conditional Generalization (CIM). The DIM metric quantifies a model's default-mode bias when given a coarse, underspecified prompt (e.g., ``a photo of a bus"), where a score closer to zero indicates a more balanced and diverse output. In contrast, the CIM metric evaluates controllability by measuring whether a model can reliably produce a specific, requested attribute from a dense, explicit prompt (e.g., ``a photo of a red bus"), where a high score is better. By evaluating both metrics, we can comprehensively assess whether our method's improvement in diversity comes at the cost of controllability, directly addressing the core trade-off.

\paragraph{Prompt Generation Methodology} To evaluate attribute-level performance using the DIM and CIM metrics, we created a structured set of prompts for each concept. The process begins with a real-world caption from the COCO dataset, which we term the \texttt{coco\_seed\_caption}. From this, we derive a generic \texttt{coarse\_prompt} by removing specific attributes (e.g., ``yellow''). This coarse prompt is used to evaluate the model's default-mode bias. Subsequently, we programmatically insert a variety of attributes (e.g., different colors and body types) back into the coarse prompt template to generate a list of \texttt{dense\_prompts}. This set of specific prompts is used to evaluate the model's conditional generalization and instruction-following ability. Listing below shows one complete example of this structure for the 'bus' concept, demonstrating how one seed caption yields an entire set of evaluation prompts.

\begin{tcolorbox}[
    enhanced,
    title={Example of the JSON structure used to store prompts for different concepts},
    label={lst:calss-truck-bus-prompts},
    colback=gray!5,
    colframe=gray!75,
    fonttitle=\bfseries,
    colbacktitle=gray!20,
    coltitle=black
]
\begin{lstlisting}[
    language=json,
    basicstyle=\small\ttfamily, % 使用等宽字体 (关键!)
    columns=flexible,           % 调整字符间距，防止过宽
    keepspaces=true,            % 保持缩进
    aboveskip=0pt,              % 移除顶部多余空白
    belowskip=0pt               % 移除底部多余空白
]
{
    "coco_seed_caption": "A yellow bus stopping at a bus stop in a city.",
    "coarse_prompt": "A bus stopping at a bus stop in a city.",
    "dense_prompts": [
        {
            "attribute_type": "color",
            "attribute": "red",
            "dense_prompt": "A red bus stopping at a bus stop in a city."
        },
        {
            "attribute_type": "body_type",
            "attribute": "double-decker",
            "dense_prompt": "A double-decker bus stopping at a bus stop in a city."
        }
    ]
\end{lstlisting}
\end{tcolorbox}

\paragraph{Qualitative Comparison.}
To provide a qualitative sense of the performance differences, we present a visual comparison against all baselines in Figure~\ref{fig:visualize_DIMCIM}. The images in this figure were generated using the dense prompt ``A single-decked bus is parked on the street" for the first row, and other representative prompts for the subsequent rows. The results visually confirm our method's superior ability to generate a more diverse set of outputs while maintaining high fidelity, in line with the quantitative findings. You can see more results in Sec~\ref{sec:appendix_more_visuals}.

\begin{figure*}[htbp]
    \centering

    \begin{subfigure}[t]{0.24\linewidth}
        \centering
        \caption{CADS}
        \includegraphics[width=\linewidth]{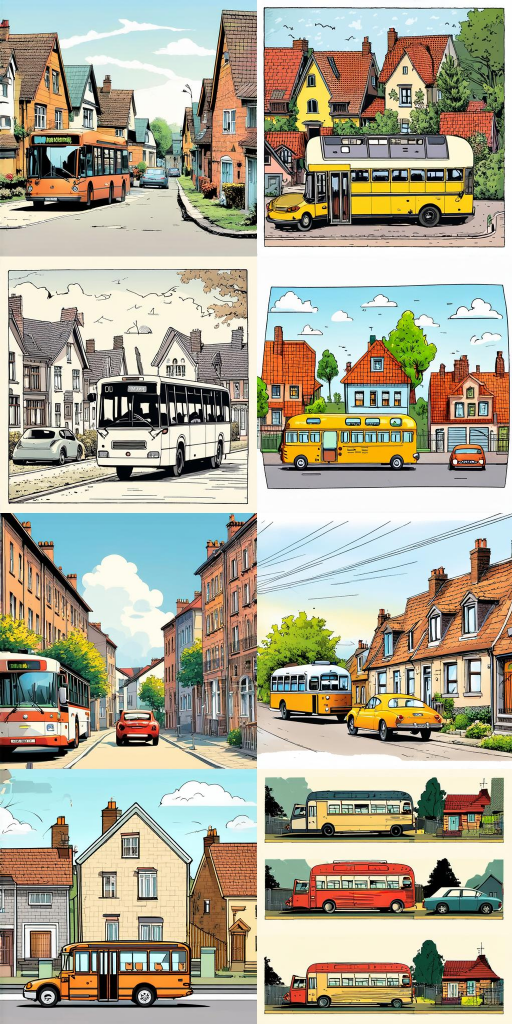}
    \end{subfigure}
    \hfill
    \begin{subfigure}[t]{0.24\linewidth}
        \centering
        \caption{DPP}
        \includegraphics[width=\linewidth]{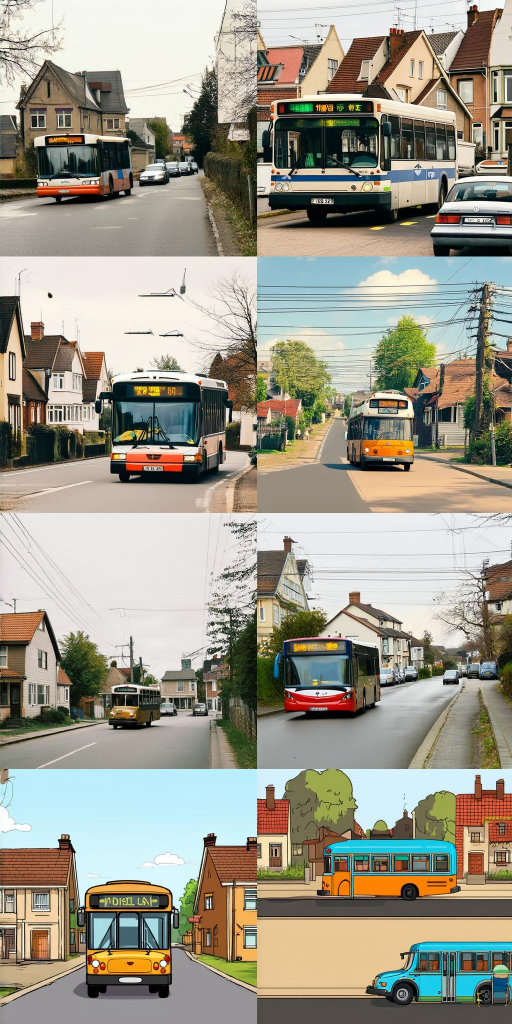}
        \label{fig:grid_dpp}
    \end{subfigure}
    \hfill
    \begin{subfigure}[t]{0.24\linewidth}
        \centering
        \caption{PG}
        \includegraphics[width=\linewidth]{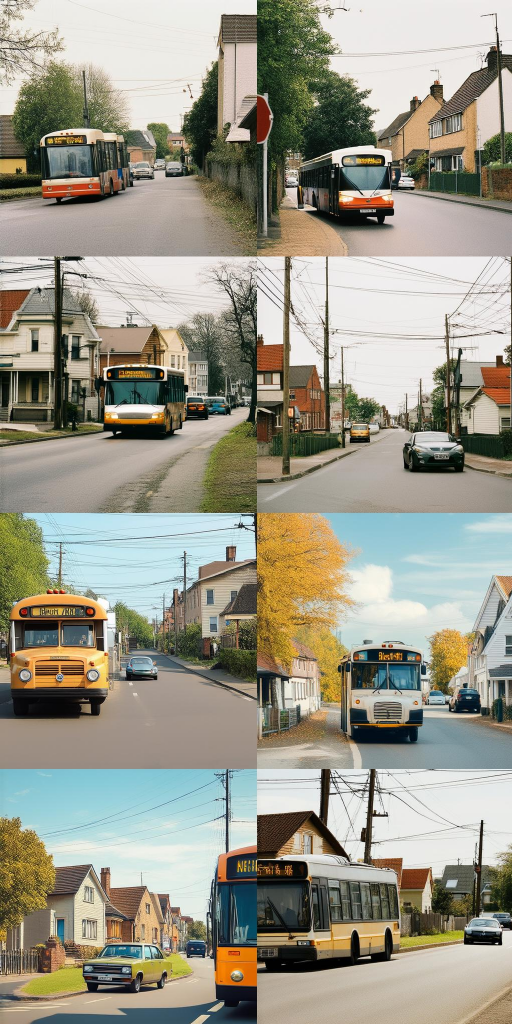}
        \label{fig:grid_pg}
    \end{subfigure}
    \hfill
    \begin{subfigure}[t]{0.24\linewidth}
        \centering
        \caption{Our Method}
        \includegraphics[width=\linewidth]{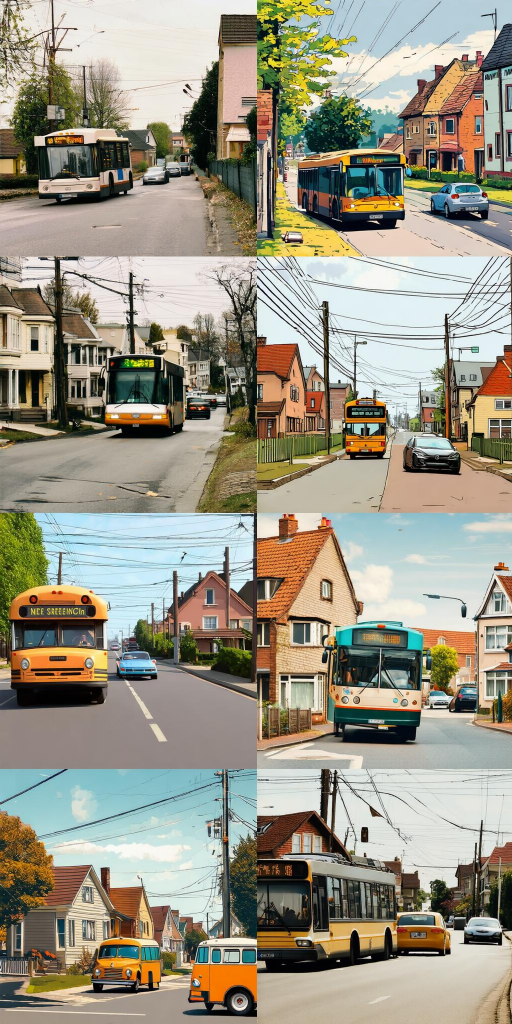}
        \label{fig:grid_ourmethod}
    \end{subfigure}
    \caption{
        Qualitative comparison of default-mode diversity using \textbf{coarse prompts}.
    }
    \label{fig:visualize_DIMCIM}
\end{figure*}

\section{Additional Experiments}
\FloatBarrier
\label{sec:appendix_additional_exp}

\subsection{Toy Example: Behavior in High-Particle Regimes}
\label{subsec:appendix_high_particle}

\subsubsection{Analysis of Particle Density and Coverage Trade-offs}
\label{sec:toy_analysis}

To characterize the behavior of OSCAR, we analyze its performance across different particle regimes using a 2D $3 \times 3$ Gaussian Mixture Model (GMM) toy example. As illustrated in Figure~\ref{fig:appendix_toy_example2} and Figure~\ref{fig:appendix_toy_example1}, we compare standard Flow Matching with OSCAR under moderate ($N=200$) and high-density ($N=2000$) conditions.

In the moderate regime, standard Flow Matching successfully converges to the GMM means but suffers from over-concentration, failing to capture the full variance of the components. In contrast, OSCAR maintains alignment with all nine modes while significantly improving within-mode coverage. By actively discouraging trajectory collapsing, our method yields a sample distribution that better represents the underlying support while preserving the clear cluster structure.

\begin{figure}[t]
    \centering
    \includegraphics[width=0.6\textwidth]{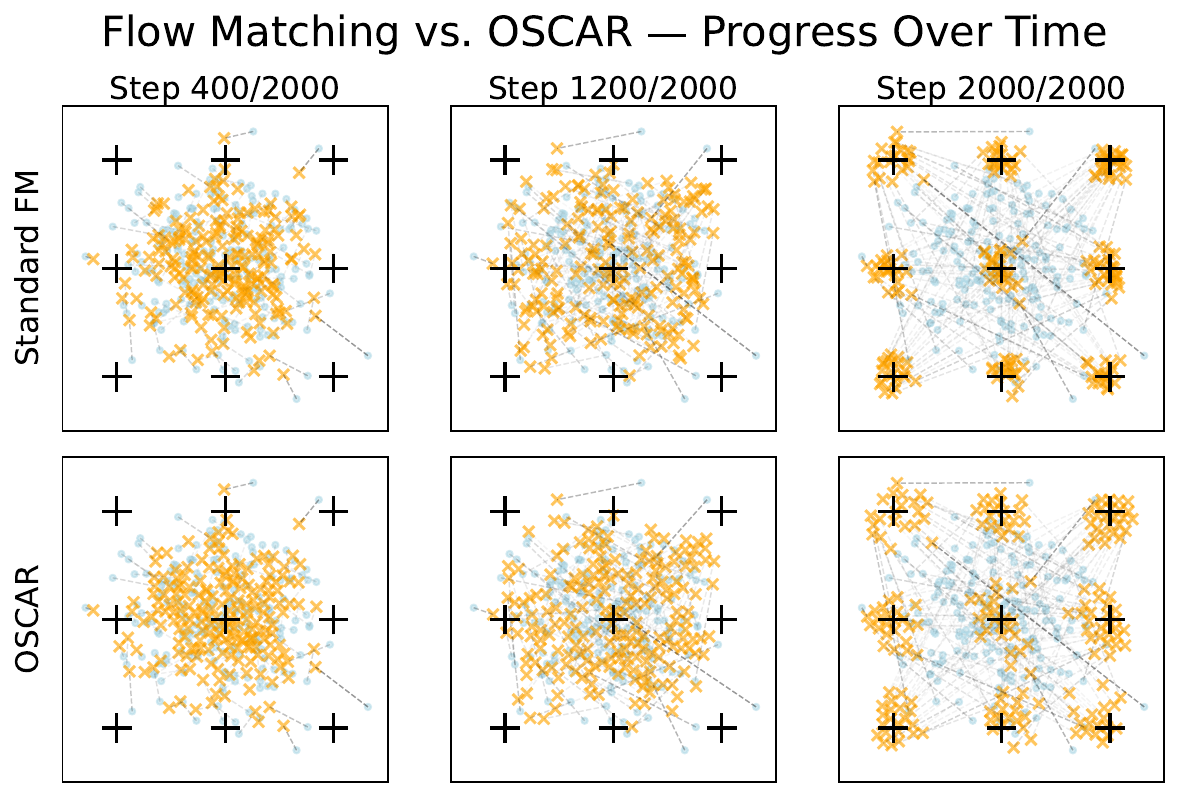}
    \captionof{figure}{
      \textbf{2D toy visualization of diversity enhancement on a 3x3 GMM.}
      The top row shows the standard Flow Matching baseline, while the bottom row shows our method. Columns are snapshots at early, middle, and final generation steps. Black ``+" markers indicate the GMM means, blue dots are the shared initial particles, and orange ``x" markers are the particle locations at the given step. Our method yields more uniform within-mode coverage and better-separated modes.
    }
    \label{fig:appendix_toy_example2}
\end{figure}

As the number of particles increases to $N=2000$  without reducing the guidance strength $\gamma(t)$, we observe a phase transition in the particle configuration. The cumulative repulsive force begins to dominate the base flow, causing samples to arrange themselves into a near-uniform distribution over the support. In this regime, the modes appear as ``square-like'' patches that prioritize coverage over local density peaks.

This phenomenon is a direct consequence of our core objective: maximizing the Feature Volume. The control signal $g(x,t)$ acts effectively as a repulsive force between trajectories. From a physics perspective, when a large number of mutually repulsive particles are confined within a potential well, the configuration that minimizes total potential energy is a uniform distribution rather than a concentrated cluster. Consequently, in high-density regimes, OSCAR prioritizes covering the support, exploring the tails and boundaries of the valid region, over matching the exact probability density of the mode centers.

Rigorously, this behavior is predicted by the Girsanov Representation derived in Appendix~\ref{sec:appendix_theory} Theorem 4. The KL divergence between the OSCAR sampling distribution $\mathbb{Q}$ and the baseline distribution $\mathbb{P}$ is governed by the energy cost of the guidance: $\mathcal{D}_{KL}(\mathbb{Q} \| \mathbb{P}) = \frac{1}{2} \mathbb{E}_{\mathbb{Q}} [ \int \|g\|^2 / \beta dt ]$. As $N$ increases, the magnitude of the collective repulsive gradients $\|g\|^2$ grows significantly. This implies a larger upper bound on the divergence, theoretically predicting that the sampled distribution will deviate more significantly from the baseline density to achieve maximum entropy.

\begin{figure}[t]
    \centering
    \includegraphics[width=0.6\textwidth]{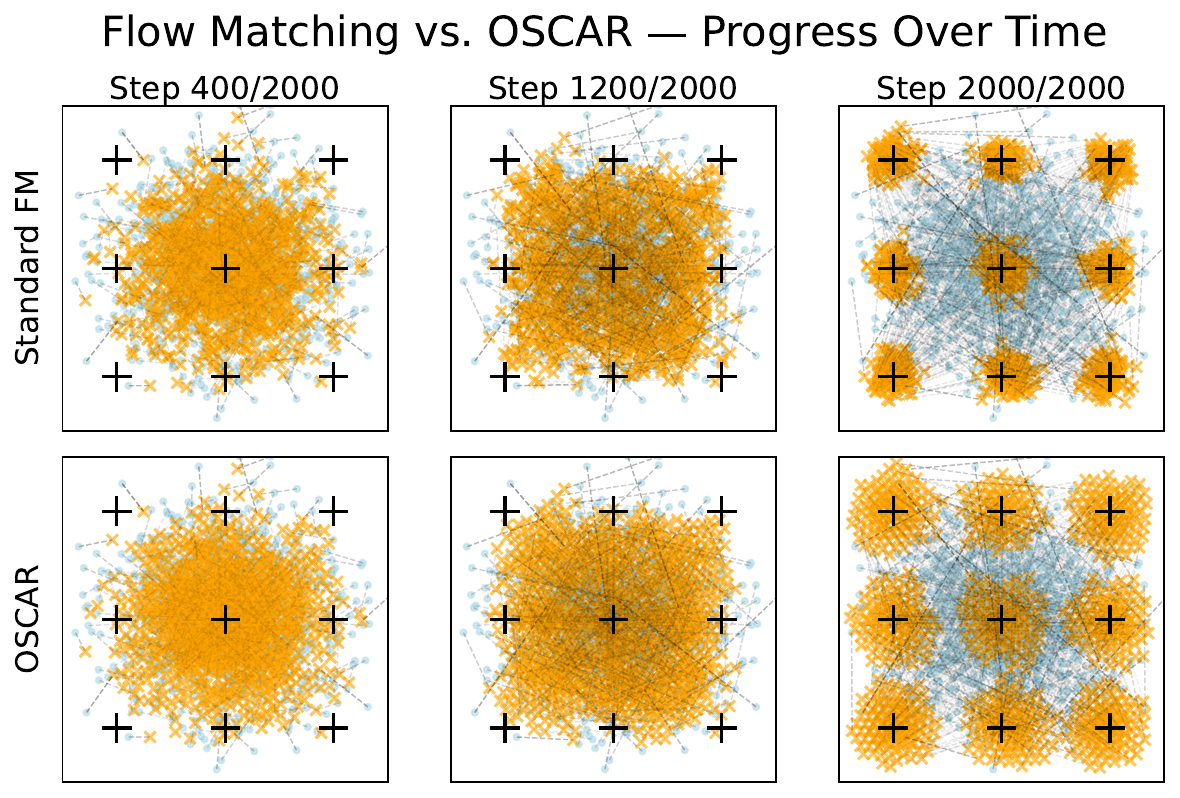}
        \caption{
        When the number of particles is significantly increased without a corresponding reduction in guidance strength, the cumulative repulsive force dominates the base flow. This causes the samples to arrange themselves into a uniform distribution to minimize system energy, effectively prioritizing maximum support coverage over precise density matching.
        }
    \label{fig:appendix_toy_example1}
\end{figure}

\subsection{ImageNet-400 Single-sample Evaluation}
As an additional sanity check that our diversity-aware control does not degrade per-sample fidelity, we conduct a class-conditional experiment on ImageNet. We randomly sample 400 classes from ImageNet-1K and, for each class, construct a single class-name prompt and generate one image per method. We then evaluate all methods on the resulting set of 400 images. As reported in Tab.~\ref{tab:imagenet_evaluation}, our method achieves the best performance on both variants of the Vendi Score (pixel- and Inception-based), indicating that the induced stochastic control does not collapse to a narrow subset of visual patterns. At the same time, OSCAR attains the lowest FID among all competitors, and matches baselines on CLIP score and other fidelity-oriented metrics. Taken together, these results confirm that our method preserves, and can even slightly improve, overall image quality while maintaining strong diversity characteristics even in this challenging single-sample-per-class regime.

\begin{table}[h]
\centering
\caption{\small
Evaluation on 400 single-image ImageNet-style prompts (1 per class).
This unbiased setting tests for any systematic degradation of single-image fidelity.
}
\label{tab:imagenet_evaluation}
\begin{tabular}{lcccc}
\toprule
\textbf{Method} &
\makecell{\textbf{FID} $\downarrow$} &
\makecell{\textbf{Vendi}\\\textbf{(Pixel)} $\uparrow$} &
\makecell{\textbf{Vendi}\\\textbf{(Inception)} $\uparrow$} &
\makecell{\textbf{CLIP Score} $\uparrow$} \\

\specialrule{\heavyrulewidth}{\aboverulesep}{\belowrulesep}
FM-SD3.5
& 100.4 & 3.38 & 35.70 & 18.14 $\pm$ 0.25  \\

\specialrule{\heavyrulewidth}{\aboverulesep}{\belowrulesep}
DPP
& 100.3 & 3.34 & 36.03 & 18.11 $\pm$ 0.25 \\
CADS
& 101.1 & 3.52 & 35.91 & 17.96 $\pm$ 0.25 \\
PG
& 99.7 & 3.66 & 34.34 & 18.13 $\pm$ 0.24 \\

\specialrule{\heavyrulewidth}{\aboverulesep}{\belowrulesep}
\rowcolor{lightblue}
\textbf{Oscar}
& \textbf{99.3} & \textbf{3.85} & \textbf{36.40} & \textbf{18.08 $\pm$ 0.24} \\
\bottomrule
\end{tabular}
\end{table}

\subsection{Fair-Budget Comparison}
\label{subsec:fair_budget}

We first compare the computational cost of Oscar against various baselines under the same generation configuration, as shown in Table~\ref{tab:cost_comparison_total}. In terms of FLOPs and runtime, Oscar introduces only a modest computational overhead compared to the baseline FM-SD3.5. However, this overhead is significantly lower than that of DPP, which requires more than double the computation and time of the baseline. This result empirically validates our claim in Section~\ref{sec:related_work} that our method significantly reduces the computational complexity inherent in DPP. Regarding peak GPU VRAM usage, all methods exhibit roughly comparable memory consumption.

\begin{table}[h]
\centering
\caption{\small
Computational cost comparison under fixed generation settings (NFE=30, CFG=5.0, batch size=32). We report total theoretical computation, wall-clock time, and peak VRAM usage. 
}
\label{tab:cost_comparison_total} 
\begin{tabular}{lccc}
\toprule
\textbf{Variant} &
\makecell{\textbf{FLOPs}\textbf{(G)} $\downarrow$} &
\makecell{\textbf{Time}\\\textbf{(seconds/run)} $\downarrow$} & 
\makecell{\textbf{Peak VRAM}\textbf{(GB)} $\downarrow$} \\

\specialrule{\heavyrulewidth}{\aboverulesep}{\belowrulesep} 
FM-SD3.5
& 4093.4 & 237.8 & 19.2 \\

\specialrule{\heavyrulewidth}{\aboverulesep}{\belowrulesep} 
DPP
& 9045.1 & 990.2 & 19.5 \\
CADS
& 4093.4 & 231.2 & 20.0 \\
PG
& 4093.4 & 229.6 & 26.4 \\

\specialrule{\heavyrulewidth}{\aboverulesep}{\belowrulesep} 
\rowcolor{lightblue}
Oscar (Ours)
& 5534.6 & 359.7 & 18.2 \\
\bottomrule

\multicolumn{4}{l}{\footnotesize Costs measured on an NVIDIA A6000 GPU, 512x512 resolution, batch size=32.}
\\[-2pt]
\end{tabular}
\end{table}

To conduct a fairer comparison, we further evaluate Oscar against FM-SD3.5 under a matched computational budget, as detailed in Table~\ref{tab:appendix_compute_matched}. We match the cost of Oscar against two enhanced FM-SD3.5 variants: one with increased NFE to 40 and another with an increased number of particles to 48. The results show that even at a matched computational cost, our method comprehensively outperforms the baselines on most metrics. Although the FM-SD3.5 variant with increased particles achieves a better FID score, we note that the FID metric is known to be highly sensitive to the number of samples when the evaluation set is small.

\begin{table*}[h]
\centering
\caption{\small
Compute-matched comparison with standard Flow Matching model under matched FLOPs.
}
\label{tab:appendix_compute_matched}
\begin{tabular}{lcccccccc}
\toprule
 & \multicolumn{3}{c}{\textbf{Compute budget}} & \multicolumn{3}{c}{\textbf{Quality \& diversity metrics}} \\
\cmidrule(lr){2-4} \cmidrule(lr){5-8}
\textbf{Method} &
\makecell{\textbf{NFE}} &
\makecell{\textbf{Particles}} &
\makecell{\textbf{FLOPs}\\\textbf{(T/run)}} &
\makecell{\textbf{CLIP}$\uparrow$} &
\makecell{\textbf{Vendi}\\\textbf{(Pixel)} $\uparrow$} &
\makecell{\textbf{FID}$\downarrow$} &
\makecell{\textbf{BRISQUE}$\downarrow$} \\
\midrule
FM-SD3.5 & 30 & 32 & 4.09 & 28.24 $\pm$ 0.18 & 2.45 $\pm$ 0.13 & 164.4 $\pm$ 1.8 & 23.4 $\pm$ 1.4 \\
+K particles   & 30 & 48 & 6.14 & 28.15 $\pm$ 0.23 & 2.60 $\pm$ 0.21 & \textbf{149.7 $\pm$ 1.3} & 23.5 $\pm$ 1.7 \\
+N NFE  & 40 & 32 & 5.43 & 28.15 $\pm$ 0.21 & 2.40 $\pm$ 0.15 & 165.3 $\pm$ 1.8 & 24.6 $\pm$ 1.7 \\
\midrule
\rowcolor{lightblue}
\textbf{OSCAR} & 30 & 32 & 5.53 & \textbf{28.26 $\pm$ 0.22} & \textbf{2.86 $\pm$ 0.05} & 163.3 $\pm$ 1.6 & \textbf{21.2 $\pm$ 1.5} \\
\bottomrule
\end{tabular}
\end{table*}

\subsection{Performance Across Different Sampling Steps}
\label{subsec:appendix_nfe}

To demonstrate the robustness and efficiency of our method across various computational budgets, we first analyze its performance as a function of the NFE. The quantitative results, presented in Table~\ref{tab:NFE_metrics_comparison}, confirm that our method consistently outperforms all baselines across the tested NFE settings of 10, 20, and 40. Specifically, our method achieves a significant and stable advantage in diversity, leading in both Vendi Score Pixel and Vendi Score Inception scores at every step count. This is accomplished without sacrificing quality; our method attains the best Brisque score at every NFE level and maintains a highly competitive FID score, confirming a high degree of overall distributional fidelity.

\begin{table*}[h!]
\centering
\caption{
    Quantitative comparison of our method against baselines across different NFE. Our method consistently improves diversity metrics while achieving state-of-the-art performance on quality and alignment scores.
}
\label{tab:NFE_metrics_comparison}
\footnotesize
\setlength{\tabcolsep}{5pt}
\renewcommand{\arraystretch}{0.8}
\begin{tabularx}{\textwidth}{l@{\extracolsep{\fill}}cccccc}
\toprule
& \multicolumn{3}{c}{\textbf{Number of Function Evaluations (NFE)}} & \multicolumn{3}{c}{\textbf{Number of Function Evaluations (NFE)}} \\
\cmidrule(lr){2-4}\cmidrule(lr){5-7}
\textbf{Method} & $40$ & $20$ & $10$ & $40$ & $20$ & $10$ \\
\midrule
& \multicolumn{3}{c}{\textbf{Brisque} $\downarrow$} & \multicolumn{3}{c}{\textbf{1 - MS-SSIM}(\%) $\uparrow$} \\
\cmidrule(lr){2-4}\cmidrule(lr){5-7}
PG              & 48.75 $\pm$ 2.85 & 71.26 $\pm$ 3.24 & 105.04 $\pm$ 3.94 & 79.40 $\pm$ 0.46 & 79.27 $\pm$ 0.73 & 79.07 $\pm$ 0.66 \\
CADS            & 15.74 $\pm$ 0.73 & 27.51 $\pm$ 0.95 & 56.84 $\pm$ 1.81 & 81.61 $\pm$ 0.65 & 80.81 $\pm$ 0.61 & 80.49 $\pm$ 0.83 \\
DPP             & 13.80 $\pm$ 1.41 & 26.51 $\pm$ 1.98 & 61.20 $\pm$ 2.23 & 81.05 $\pm$ 0.58 & 80.62 $\pm$ 0.71 & 80.00 $\pm$ 0.41 \\
\rowcolor{lightblue}\textbf{Ours}     & \textbf{12.30 $\pm$ 1.25} & \textbf{17.73 $\pm$ 1.21} & \textbf{44.04 $\pm$ 1.38} & \textbf{83.47 $\pm$ 0.42} & \textbf{82.02 $\pm$ 0.72} & \textbf{82.00 $\pm$ 0.73} \\
\midrule
& \multicolumn{3}{c}{\textbf{Vendi Score Pixel} $\uparrow$} & \multicolumn{3}{c}{\textbf{Vendi Score Inception} $\uparrow$} \\
\cmidrule(lr){2-4}\cmidrule(lr){5-7}
PG              & 1.74 $\pm$ 0.06 & 1.72 $\pm$ 0.06 & 1.64 $\pm$ 0.05 & 2.78 $\pm$ 0.15 & 2.77 $\pm$ 0.07 & 2.73 $\pm$ 0.05 \\
CADS            & 1.97 $\pm$ 0.04 & 1.94 $\pm$ 0.05 & 1.77 $\pm$ 0.06 & 2.84 $\pm$ 0.40 & 2.81 $\pm$ 0.07 & 2.83 $\pm$ 0.06 \\
DPP             & 1.88 $\pm$ 0.06 & 1.85 $\pm$ 0.06 & 1.75 $\pm$ 0.05 & 2.87 $\pm$ 0.07 & 2.86 $\pm$ 0.12 & 2.77 $\pm$ 0.10 \\
\rowcolor{lightblue}\textbf{Ours}     & \textbf{2.07 $\pm$ 0.07} & \textbf{2.06 $\pm$ 0.06} & \textbf{1.93 $\pm$ 0.05} & \textbf{2.93 $\pm$ 0.08} & \textbf{2.88 $\pm$ 0.05} & \textbf{2.87 $\pm$ 0.12} \\
\midrule
& \multicolumn{3}{c}{\textbf{FID} $\downarrow$} & \multicolumn{3}{c}{\textbf{CLIP Score} $\uparrow$} \\
\cmidrule(lr){2-4}\cmidrule(lr){5-7}
PG              & 167.34 $\pm$ 2.12 & 166.50 $\pm$ 2.39 & 165.39 $\pm$ 2.03 & \textbf{19.11 $\pm$ 0.40} & \textbf{18.93 $\pm$ 0.36} & 18.23 $\pm$ 0.76 \\
CADS            & 165.84 $\pm$ 0.95 & 166.13 $\pm$ 1.22 & 167.34 $\pm$ 1.03 & 18.97 $\pm$ 0.38 & 18.90 $\pm$ 0.46 & \textbf{18.35 $\pm$ 0.63} \\
DPP             & 166.51 $\pm$ 1.52 & 165.89 $\pm$ 0.72 & 166.21 $\pm$ 1.28 & 18.97 $\pm$ 0.42 & 18.63 $\pm$ 0.46 & 17.91 $\pm$ 3.94 \\
\rowcolor{lightblue}\textbf{Ours}     & \textbf{165.40 $\pm$ 0.94} & \textbf{164.75 $\pm$ 0.66} & \textbf{165.31 $\pm$ 0.79} & 18.77 $\pm$ 0.46 & 18.58 $\pm$ 0.58 & 18.11 $\pm$ 0.69 \\
\midrule
& \multicolumn{3}{c}{\textbf{CLIP-IQA} $\uparrow$} & \multicolumn{3}{c}{\textbf{Image Reward} $\uparrow$} \\
\cmidrule(lr){2-4}\cmidrule(lr){5-7}
PG              & 7.57 $\pm$ 0.24 & 7.52 $\pm$ 0.23 & 7.41 $\pm$ 0.25 & 1.26 $\pm$ 0.20 & 1.17 $\pm$ 0.23 & 1.03 $\pm$ 0.25 \\
CADS            & 7.67 $\pm$ 0.23 & 7.65 $\pm$ 0.23 & 7.59 $\pm$ 0.21 & 1.42 $\pm$ 0.12 & 1.37 $\pm$ 0.13 & 1.29 $\pm$ 0.16 \\
DPP             & 7.61 $\pm$ 0.25 & 7.64 $\pm$ 0.24 & 7.61 $\pm$ 0.24 & 1.42 $\pm$ 0.17 & 1.36 $\pm$ 0.15 & 1.23 $\pm$ 0.18 \\
\rowcolor{lightblue}\textbf{Ours}     & \textbf{7.66 $\pm$ 0.22} & \textbf{7.65 $\pm$ 0.24} & \textbf{7.63 $\pm$ 0.23} & \textbf{1.49 $\pm$ 0.11} & \textbf{1.43 $\pm$ 0.12} & \textbf{1.36 $\pm$ 0.14} \\
\bottomrule
\end{tabularx}
\end{table*}


This quantitative superiority is particularly evident in low-step scenarios. For a direct visual comparison, Figure~\ref{fig:nfe10_comparison} showcases the qualitative differences against all baselines at a fixed, low computational budget of NFE=20 for the prompt ``A photo of a truck''. The visual results corroborate our quantitative findings, highlighting that our method produces a significantly more diverse and visually coherent set of images compared to the baselines in efficient generation scenarios.

Having established our method's advantage over baselines, we also analyze its internal trade-off between computational cost and image fidelity in Figure~\ref{fig:visual_cozy}. While our approach is effective even at NFE=10, the generations can exhibit minor artifacts like blurry edges and localized distortions. Quality progressively improves with a larger budget, with NFE=20 yielding sharper results and NFE=40 producing the most natural and artifact-free images. Taken together, these results confirm that our method is superior to baselines at all tested budgets, and its output quality can be further enhanced with a higher NFE.

\begin{figure*}[htbp]
    \centering

    \begin{subfigure}[b]{0.32\textwidth}
        \centering
        \includegraphics[width=\linewidth]{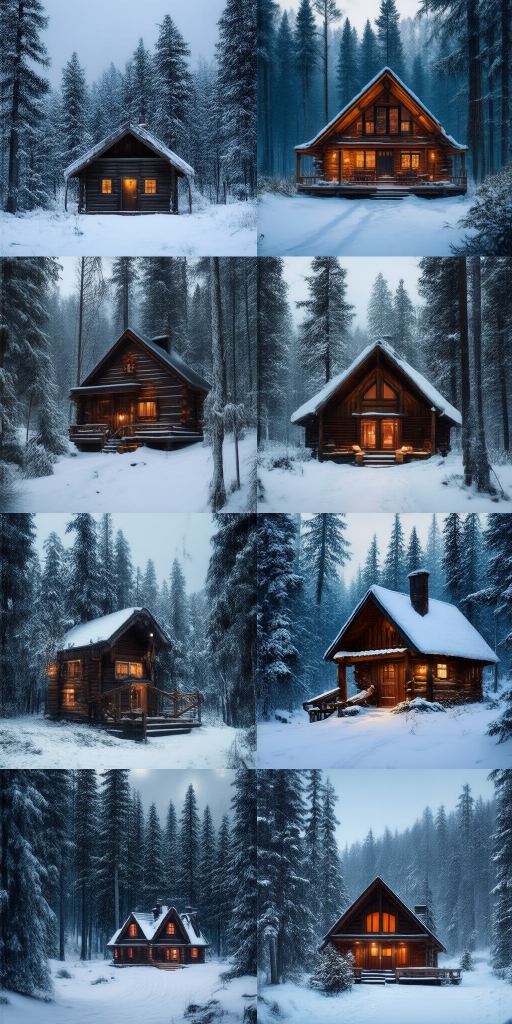}
        \caption{NFE = 10}
        \label{fig:nfe10}
    \end{subfigure}
    \hfill 
    \begin{subfigure}[b]{0.32\textwidth}
        \centering
        \includegraphics[width=\linewidth]{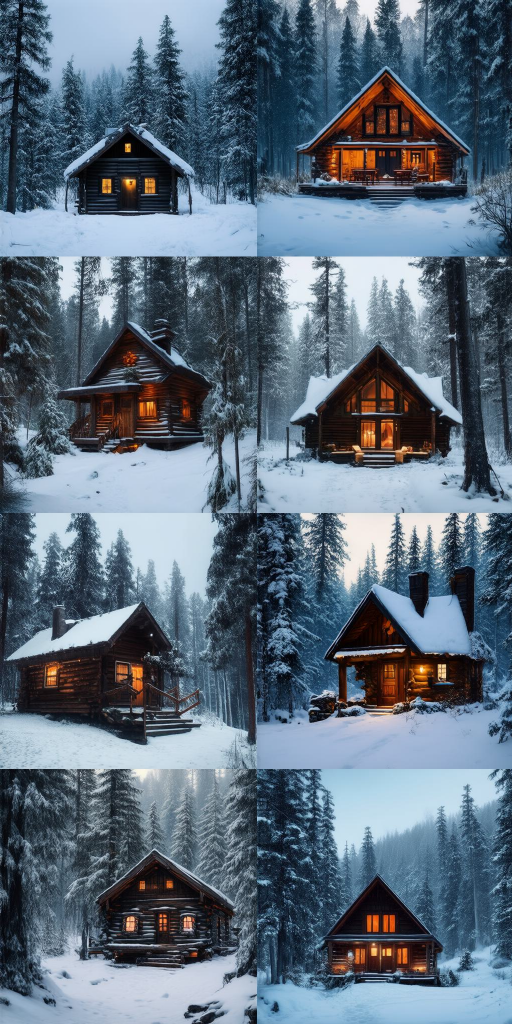}
        \caption{NFE = 20}
        \label{fig:nfe20}
    \end{subfigure}
    \hfill 
    \begin{subfigure}[b]{0.32\textwidth}
        \centering
        \includegraphics[width=\linewidth]{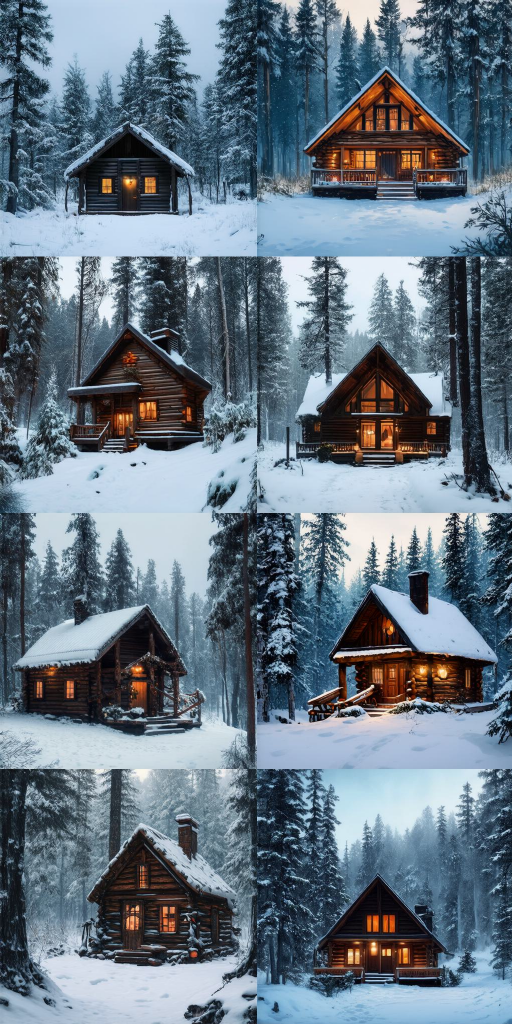}
        \caption{NFE = 40}
        \label{fig:nfe40}
    \end{subfigure}
    \caption{
        Visual analysis of our method's sensitivity to the NFE for the prompt ``A cozy cabin in the snowy forest." A low NFE of 10 results in images with blurry edges and localized artifacts, while a higher NFE of 40 yields the most natural and artifact-free results.
    }
    \label{fig:visual_cozy}
\end{figure*}

\begin{figure*}[htbp]
    \centering

    \begin{subfigure}[b]{0.24\linewidth}
        \centering
        \includegraphics[width=\linewidth]{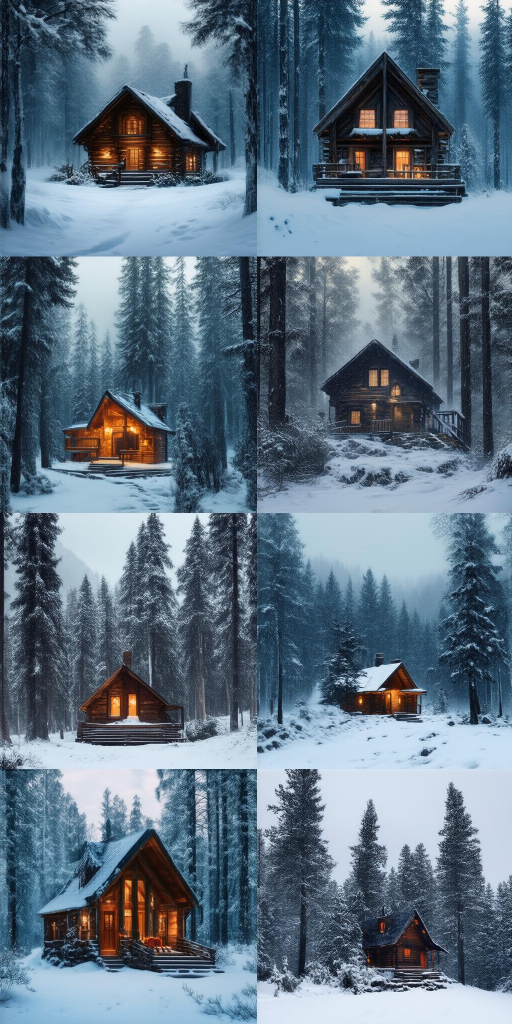}
        \caption{CADS}
        \label{fig:nfe10_cads}
    \end{subfigure}
    \hfill
    \begin{subfigure}[b]{0.24\linewidth}
        \centering
        \includegraphics[width=\linewidth]{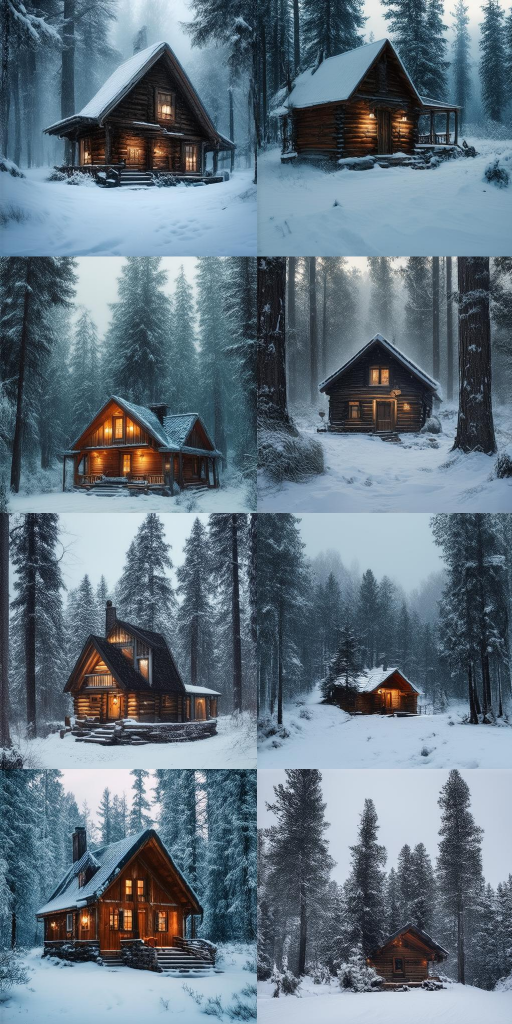}
        \caption{DPP}
        \label{fig:nfe10_dpp}
    \end{subfigure}
    \hfill
    \begin{subfigure}[b]{0.24\linewidth}
        \centering
        \includegraphics[width=\linewidth]{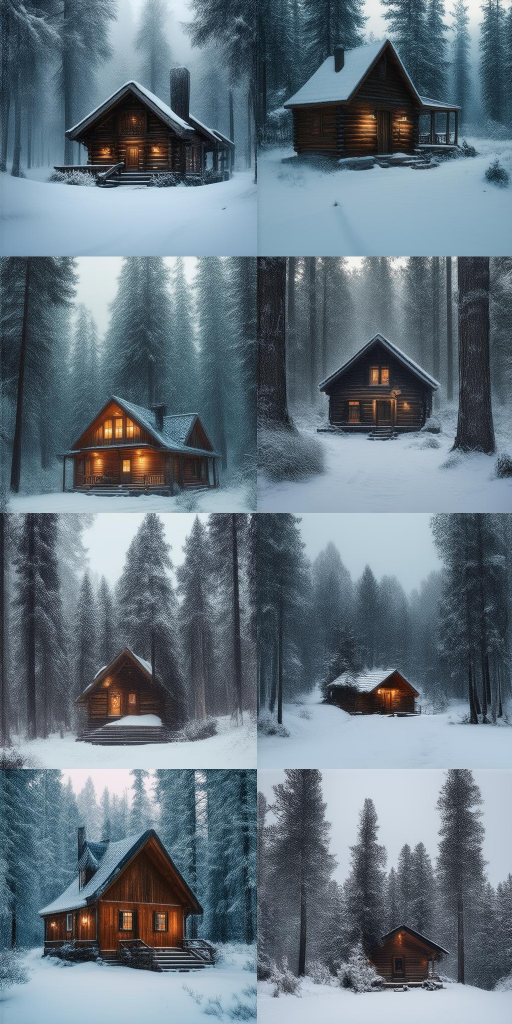}
        \caption{PG}
        \label{fig:nfe10_pg}
    \end{subfigure}
    \hfill
    \begin{subfigure}[b]{0.24\linewidth}
        \centering
        \includegraphics[width=\linewidth]{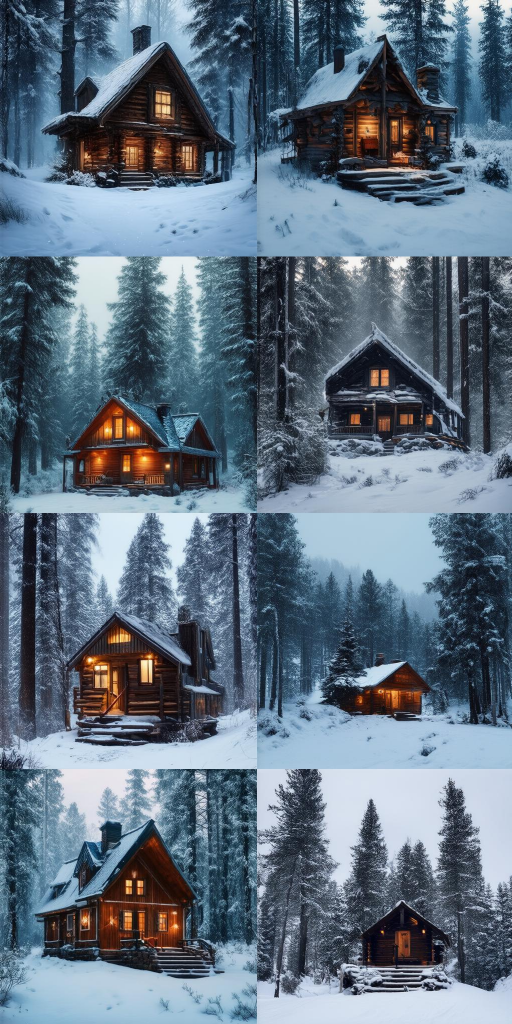}
        \caption{Our Method}
        \label{fig:nfe10_ourmethod}
    \end{subfigure}
    \caption{
        Visual comparison of all methods at a fixed low computational budget (NFE=20) for the prompt ``A cozy cabin in the snowy forest". Our method generates a visibly more diverse set of images.
    }
    \label{fig:nfe10_comparison}
\end{figure*}

\subsection{Diagnostic Results for Mixed ODE/SDE Sampling}
\label{sec:appendix_mix_ode_sde}

We additionally evaluate a simple Mix ODE/SDE sampler as a diagnostic comparison on T2I-CompBench. 
This sampler follows a hybrid trajectory: stochastic SDE-style updates are used in the early stage to increase sample variation, while the later stage returns to the deterministic ODE sampler. 
Unlike OSCAR, this diagnostic sampler does not include an explicit geometric constraint for preserving the base flow direction.

Table~\ref{tab:mix_ode_sde_compbench} shows the results on T2I-CompBench. 
Mix ODE/SDE improves Vendi Score Pixel over the base model and other baselines, indicating that early stochasticity can indeed increase low-level sample variation. 
However, this comes with weaker alignment and perceptual quality indicators: its CLIP Score and BLIPVQA are lower than those of OSCAR and several other baselines. 
By contrast, OSCAR achieves the best Vendi Score Inception, $1-\text{MS-SSIM}$, and BLIPVQA, while remaining competitive on CLIP Score. 
These results suggest that unconstrained stochasticity can increase diversity, but does not necessarily yield a favorable diversity--quality trade-off.

\begin{table}[t]
    \centering
    \caption{\small
    Diagnostic comparison of Mix ODE/SDE on T2I-CompBench.
    }
    \label{tab:mix_ode_sde_compbench}
    \footnotesize
    \setlength{\tabcolsep}{4pt}
    \renewcommand{\arraystretch}{0.95}
    \begin{tabular}{lccccc}
        \toprule
        \textbf{Method} &
        \textbf{Vendi Inc.} $\uparrow$ &
        \textbf{Vendi Pix.} $\uparrow$ &
        \textbf{CLIP} $\uparrow$ &
        $\mathbf{1-\textbf{MS-SSIM}}$ $\uparrow$ &
        \textbf{BLIPVQA} $\uparrow$ \\
        \midrule
        FM-SD3.5    & 2.95 & 1.89 & \textbf{36.03} & 0.8092 & 0.7384 \\
        CADS        & 3.04 & 1.86 & 35.20 & 0.8231 & 0.7245 \\
        PG          & 3.02 & 1.94 & 35.37 & 0.7835 & 0.7811 \\
        DPP         & 2.93 & 1.99 & 35.55 & 0.7984 & 0.7914 \\
        Mix ODE/SDE & 3.02 & \textbf{2.16} & 34.04 & 0.8201 & 0.7353 \\
        \rowcolor{lightblue}
        \textbf{OSCAR} & \textbf{3.15} & 2.13 & 35.53 & \textbf{0.8285} & \textbf{0.7952} \\
        \bottomrule
    \end{tabular}
\end{table}

Figure~\ref{fig:mix_ode_sde_visuals} further illustrates this behavior qualitatively. 
Although Mix ODE/SDE produces visually varied samples, many outputs contain noticeable artifacts, distorted structures, or unnatural textures. 
This supports the quantitative observation that simply injecting stochasticity is insufficient for preserving image quality during diversity enhancement.

\begin{figure}[t]
    \centering
    \includegraphics[width=\linewidth]{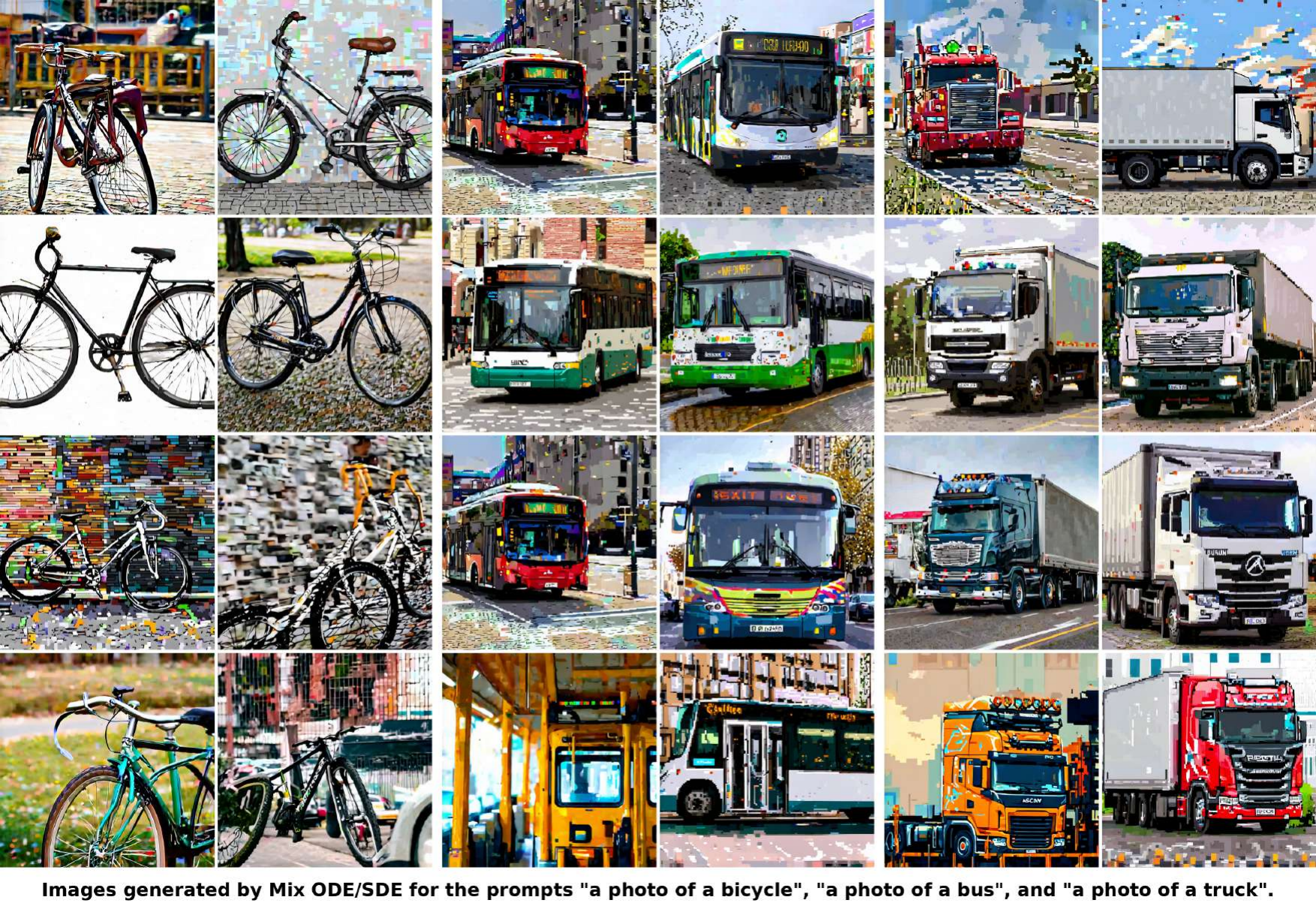}
    \caption{\small
    Visual results of the Mix ODE/SDE diagnostic sampler on T2I-CompBench prompts. 
    The samples show increased variation but also visible artifacts and degraded visual fidelity.
    }
    \label{fig:mix_ode_sde_visuals}
\end{figure}

\section{Robust Studies}
\FloatBarrier
\label{sec:appendix_robust}

\subsection{Robustness to the Number of Generated Samples}
\label{sec:appendix_sample_number_robustness}

A potential concern is that OSCAR may rely on generating a large number of images jointly, since the Gram-matrix-based repulsive term is defined over a jointly active set of trajectories. When the batch size $m$ is small, the method remains applicable, but the diversity signal becomes more local. To examine this setting, we conduct an additional small-batch evaluation with $m=4$. As shown in Table~\ref{tab:small_batch_robustness}, OSCAR still consistently improves over the base model under this constrained generation setting. Although the diversity gain is weaker than in the larger-batch regime, OSCAR achieves the best Vendi Score Pixel and CLIP Score, and remains competitive on Vendi Score Inception and $1-\text{MS-SSIM}$. These results indicate that the advantage of OSCAR does not rely on a large candidate set and can still be realized when only a few images are generated per prompt.

\begin{table}[t]
    \centering
    \caption{\small Small-batch robustness evaluation with $m=4$ generated images per prompt.}
    \label{tab:small_batch_robustness}
    \small
    \setlength{\tabcolsep}{6pt}
    \renewcommand{\arraystretch}{1.05}
    \begin{tabular}{lcccc}
        \toprule
        \textbf{Method} &
        \textbf{Vendi Score Inc.} $\uparrow$ &
        \textbf{Vendi Score Pixel} $\uparrow$ &
        \textbf{CLIP Score} $\uparrow$ &
        $\mathbf{1-\text{MS-SSIM}}$ $\uparrow$ \\
        \midrule
        FM-SD3.5    & 1.84 & 1.64 & 37.75 & 0.8686 \\
        CADS        & 1.83 & 1.64 & 37.82 & 0.8686 \\
        PG          & 1.83 & 1.58 & 37.81 & 0.8391 \\
        DPP         & 1.87 & 1.67 & 37.13 & 0.8654 \\
        Mix ODE/SDE & \textbf{1.88} & 1.74 & 36.64 & \textbf{0.9033} \\
        \rowcolor{lightblue}
        \textbf{OSCAR} & \textbf{1.88} & \textbf{1.76} & \textbf{37.95} & 0.8825 \\
        \bottomrule
    \end{tabular}
\end{table}

More generally, a natural offline extension is to maintain a memory bank of cached endpoint features from previous generations. The repulsive term can then be computed not only against the currently active set, but also against these cached features, providing a broader semantic reference under low-VRAM or asynchronous generation settings. We view this as a promising generalization of the current online formulation.

\subsection{Numerical Robustness: Our Extrapolation vs. Alternative Integrators}

The final step of our framework utilizes a predictor to estimate the final vector. To evaluate the robustness of this component, we compared several methods, with the results detailed in Table 8. These methods include Euler, Midpoint, no predictor (w/o Predictor), and our method.

Table \ref{tab:predictor_ablation} reveals a clear trade-off between performance and stability. Completely removing the predictor, while yielding the lowest volatility, significantly sacrifices image quality, leading to notably worse performance. In contrast, the Euler and Midpoint methods show improved average performance, but at the cost of greater instability across multiple runs. Our method achieves the joint-best average performance across all metrics. More importantly, our method accomplishes this superior performance while maintaining low volatility, striking the optimal balance between high-quality output and high stability. It was therefore selected as our default configuration.


\begin{table}[h]
\centering
\caption{\small
Robustness and ablation study on the predictor component.
Different predictors are used to estimate the final vector with the same NFE.
Our default choice demonstrates superior performance.
}
\label{tab:predictor_ablation}
\begin{tabular}{lccccc}
\toprule
\textbf{Predictor} &
\makecell{\textbf{CLIP}$\uparrow$} &
\makecell{\textbf{Vendi}\\\textbf{(Pixel)} $\uparrow$} &
\makecell{\textbf{Vendi}\\\textbf{(Inception)} $\uparrow$} &
\makecell{\textbf{FID}$\downarrow$} &
\makecell{\textbf{BRISQUE}$\downarrow$} \\
\midrule
w/o Predictor
& 27.77 $\pm$ 0.16 & 2.78 $\pm$ 0.04 & 5.37 $\pm$ 0.11 & 165.5 $\pm$ 1.2 & 24.9 $\pm$ 1.5 \\
Euler
& 28.19 $\pm$ 0.30 & \textbf{2.87 $\pm$ 0.16} & 5.52 $\pm$ 0.28 & 164.6 $\pm$ 1.6 & 21.9 $\pm$ 1.8 \\
Midpoint
& 28.21 $\pm$ 0.32 & 2.86 $\pm$ 0.12 & 5.43 $\pm$ 0.30 & 164.4 $\pm$ 1.6 & 22.4 $\pm$ 2.2 \\
\rowcolor{lightblue}
\textbf{Ours}
& \textbf{28.26 $\pm$ 0.22} & 2.86 $\pm$ 0.05 & \textbf{5.63 $\pm$ 0.20} & \textbf{163.3 $\pm$ 1.6} & \textbf{21.2 $\pm$ 1.5} \\
\bottomrule
\multicolumn{6}{l}{\footnotesize 6 seeds; mean$\pm$95\% CI. ``w/o Predictor" directly uses the current vector without a correction step.}
\\[-2pt]
\end{tabular}
\end{table}

\subsection{Cross-Backbone Generalization}

To validate the portability of our OSCAR framework, we extended its application beyond FM-SD3.5 to two other prominent foundational models: SDXL-Turbo \citep{sauer2024adversarial} and SD1.5 \citep{Rombach_2022_CVPR}. SDXL-Turbo is a well-known distilled model designed for high-speed, real-time generation, while SD1.5 is a widely-adopted and influential early latent diffusion model.

As detailed in Table \ref{tab:portability_fixed}, we first transferred the hyperparameters tuned on FM-SD3.5 directly to these new models without modification denoted ``OSCAR (default params)". Notably, even this direct transfer demonstrates strong robustness: compared to the original baselines, this configuration shows slight improvements in diversity metrics while, critically, incurring no performance degradation in key image quality metrics.

Furthermore, OSCAR's comprehensive performance is further enhanced with simple, model-specific hyperparameter tuning. This strongly indicates that OSCAR is not an architecture-specific solution but rather a general and robust ``plug-and-play" framework that can be easily ported to diverse diffusion models.

\begin{table}[h]
\centering
\caption{\small
Portability of our OSCAR framework across different foundational models.
OSCAR consistently improves fidelity and alignment over the respective baselines for FM-SD3.5, SDXL-Turbo, and SD1.5.
Further gains are achieved by tuning OSCAR's hyperparameters for each specific model.
Higher is better for CLIP/Vendi; lower is better for FID/BRISQUE.
}
\label{tab:portability_fixed} 
\begin{tabular}{lcccccc}
\toprule
\textbf{Model} & \textbf{Variant} &
\makecell{\textbf{CLIP}$\uparrow$} &
\makecell{\textbf{Vendi}\\\textbf{(Pixel)} $\uparrow$} &
\makecell{\textbf{Vendi}\\\textbf{(Inception)} $\uparrow$} &
\makecell{\textbf{FID}$\downarrow$} &
\makecell{\textbf{BRISQUE}$\downarrow$} \\

\specialrule{\heavyrulewidth}{\aboverulesep}{\belowrulesep}
\multirow{2}{*}{\textbf{FM-SD3.5}}
& Baseline & 28.24 $\pm$ 0.18 & 2.45 $\pm$ 0.13 & 5.37 $\pm$ 0.27 & 164.4 $\pm$ 1.8 & 23.4 $\pm$ 1.4 \\
& \cellcolor{lightblue}\textbf{OSCAR} & \cellcolor{lightblue}\textbf{28.26 $\pm$ 0.22} & \cellcolor{lightblue}\textbf{2.86 $\pm$ 0.05} & \cellcolor{lightblue}\textbf{5.63 $\pm$ 0.22} & \cellcolor{lightblue}\textbf{163.3 $\pm$ 1.6} & \cellcolor{lightblue}\textbf{21.2 $\pm$ 1.5} \\

\specialrule{\heavyrulewidth}{\aboverulesep}{\belowrulesep}
\multirow{3}{*}{\textbf{SDXL-Turbo}}
& Baseline & 30.98 $\pm$ 0.15 & 5.31 $\pm$ 0.25 & 4.24 $\pm$ 0.13 & 150.4 $\pm$ 1.0 & 24.6 $\pm$ 0.3 \\
& \textbf{OSCAR} (default params) & 30.77 $\pm$ 0.24 & 5.41 $\pm$ 0.25 & 4.29 $\pm$ 0.23 & 152.8 $\pm$ 0.6 & 25.1 $\pm$ 1.0 \\
& \cellcolor{lightblue}\textbf{OSCAR} (tuned params) & \cellcolor{lightblue}\textbf{30.94 $\pm$ 0.22} & \cellcolor{lightblue}\textbf{5.48 $\pm$ 0.34} & \cellcolor{lightblue}\textbf{4.42 $\pm$ 0.17} & \cellcolor{lightblue}\textbf{151.6 $\pm$ 1.3} & \cellcolor{lightblue}\textbf{25.0 $\pm$ 0.9} \\

\specialrule{\heavyrulewidth}{\aboverulesep}{\belowrulesep}
\multirow{3}{*}{\textbf{SD1.5}}
& Baseline & 29.93 $\pm$ 0.35 & 2.37 $\pm$ 0.12 & 7.02 $\pm$ 0.27 & 174.7 $\pm$ 1.9 & 12.1 $\pm$ 3.3 \\
& \textbf{OSCAR} (default params) & 29.88 $\pm$ 0.35 & 2.45 $\pm$ 0.12 & 7.09 $\pm$ 0.16 & 175.1 $\pm$ 1.6 & 12.8 $\pm$ 3.2 \\
& \cellcolor{lightblue}\textbf{OSCAR} (tuned params) & \cellcolor{lightblue}\textbf{29.93 $\pm$ 0.37} & \cellcolor{lightblue}\textbf{2.46 $\pm$ 0.12} & \cellcolor{lightblue}\textbf{7.11 $\pm$ 0.12} & \cellcolor{lightblue}\textbf{174.6 $\pm$ 1.6} & \cellcolor{lightblue}\textbf{12.5 $\pm$ 3.1} \\
\bottomrule

\multicolumn{7}{l}{\footnotesize 6 seeds; mean$\pm$95\% CI. ``default params" uses hyperparameters from FM-SD3.5; ``tuned params" are optimized for each model.}
\\[-2pt]
\end{tabular}
\end{table}

\subsection{Encoder Choice Robustness}

To test the robustness of our framework and its dependency on different feature extractors, we conducted an ablation study as shown in Table \ref{tab:encoder_ablation}. We replaced our default CLIP encoder with two other prevalent encoders, Inception and DINO, while keeping all other components and parameters unchanged.

The results clearly indicate that our framework is highly robust to the choice of encoder. All three variants demonstrate highly comparable performance across all key metrics; our metrics remained almost entirely unperturbed by the change in encoder. These findings prove that the efficacy of our framework does not stem from an over-reliance on a specific feature space. On the contrary, the framework exhibits strong universality, and its core mechanisms operate stably across different feature spaces. We selected CLIP as our default configuration because it showed a slight advantage in diversity metrics while maintaining highly competitive fidelity.

\begin{table}[h]
\centering
\caption{\small
Ablation study on the feature encoder used within our framework.
Our default model shows the best performance compared to other common encoders like Inception and DINO.
}
\label{tab:encoder_ablation}
\begin{tabular}{lccccc}
\toprule
\textbf{Encoder} &
\makecell{\textbf{CLIP}\\\textbf{Score} $\uparrow$} &
\makecell{\textbf{Vendi}\\\textbf{(Pixel)} $\uparrow$} &
\makecell{\textbf{Vendi}\\\textbf{(Inception)} $\uparrow$} &
\makecell{\textbf{FID}$\downarrow$} &
\makecell{\textbf{BRISQUE}$\downarrow$} \\
\midrule
Inception
& 28.32 $\pm$ 0.22 & 2.85 $\pm$ 0.08 & 5.60 $\pm$ 0.22 & 163.2 $\pm$ 1.2 & 21.7 $\pm$ 1.6 \\
DINO
& 28.35 $\pm$ 0.17 & 2.82 $\pm$ 0.05 & 5.57 $\pm$ 0.24 & 164.2 $\pm$ 2.0 & 21.1 $\pm$ 1.8 \\
\rowcolor{lightblue}
\textbf{CLIP (Ours)}
& \textbf{28.26 $\pm$ 0.22} & \textbf{2.86 $\pm$ 0.05} & \textbf{5.63 $\pm$ 0.20} & \textbf{163.3 $\pm$ 1.6} & \textbf{21.2 $\pm$ 1.5} \\
\bottomrule
\multicolumn{6}{l}{\footnotesize 6 seeds; mean$\pm$95\% CI. All variants use the full OSCAR framework.}
\\[-2pt]
\end{tabular}
\end{table}

\subsection{Noise Robustness}
\label{sec:add_noise}

\subsubsection{Robustness to Noise Gate ($t_{gate}$)}
\label{sec:add_noise_schedule}
We analyze the sensitivity of our method to the noise application schedule, controlled by the \textit{Noise Gate} parameter. This parameter defines the time interval $[0.05, t_{gate}]$ during which stochastic noise is active. The results, presented quantitatively in Table~\ref{tab:noisegate_comparison} and visually in Figure~\ref{fig:noise_gate_visual_ablation}, demonstrate our method's remarkable robustness. Across a wide range of end times, from $t_{gate}=0.5$ to $t_{gate}=0.1$, all quality and diversity metrics remain exceptionally stable, without any sign of performance collapse. This indicates that our method is not sensitive to the precise duration of noise injection, making it easy to use without extensive hyperparameter tuning.

\begin{table}[htbp]
\centering
\caption{
    Robustness analysis for the noise injection schedule, controlled by the end time $t_{gate}$ of the noise gate $[0.05, t_{gate}]$. The performance across all metrics remains highly stable, demonstrating that our method is not sensitive to this hyperparameter.
}
\label{tab:noisegate_comparison}
\resizebox{\textwidth}{!}{%
\begin{tabular}{c ccccc}
\toprule
\textbf{Noise Gate} & \makecell{\textbf{1-MS-SSIM \%}  $\uparrow$} & \makecell{\textbf{Vendi Score Pixel} $\uparrow$} & \makecell{\textbf{Vendi Score Inception} $\uparrow$} & \makecell{\textbf{FID} $\downarrow$} & \makecell{\textbf{Brisque} $\downarrow$} \\
\midrule
0.50 & 87.92 $\pm$ 2.11 & 2.32 $\pm$ 0.15 & 5.12 $\pm$ 0.07 & 131.31 $\pm$ 1.31 & 25.50 $\pm$ 1.71 \\
0.40 & 87.92 $\pm$ 1.38 & 2.36 $\pm$ 0.12 & 5.21 $\pm$ 0.17 & 132.31 $\pm$ 2.10 & 23.78 $\pm$ 1.66 \\
0.30 & 88.73 $\pm$ 1.61 & 2.36 $\pm$ 0.17 & 5.21 $\pm$ 0.11 & 132.80 $\pm$ 1.96 & 23.87 $\pm$ 2.22 \\
0.20 & 88.53 $\pm$ 1.44 & 2.37 $\pm$ 0.18 & 5.24 $\pm$ 0.12 & 131.33 $\pm$ 1.15 & 23.71 $\pm$ 1.87 \\
0.10 & 87.97 $\pm$ 1.49 & 2.34 $\pm$ 0.14 & 5.16 $\pm$ 0.07 & 130.93 $\pm$ 1.42 & 23.54 $\pm$ 2.05 \\
\bottomrule
\end{tabular}%
}
\end{table}

\begin{figure}[htbp]
    \centering

    \begin{subfigure}[b]{0.19\linewidth}
        \centering
        \includegraphics[width=\linewidth]{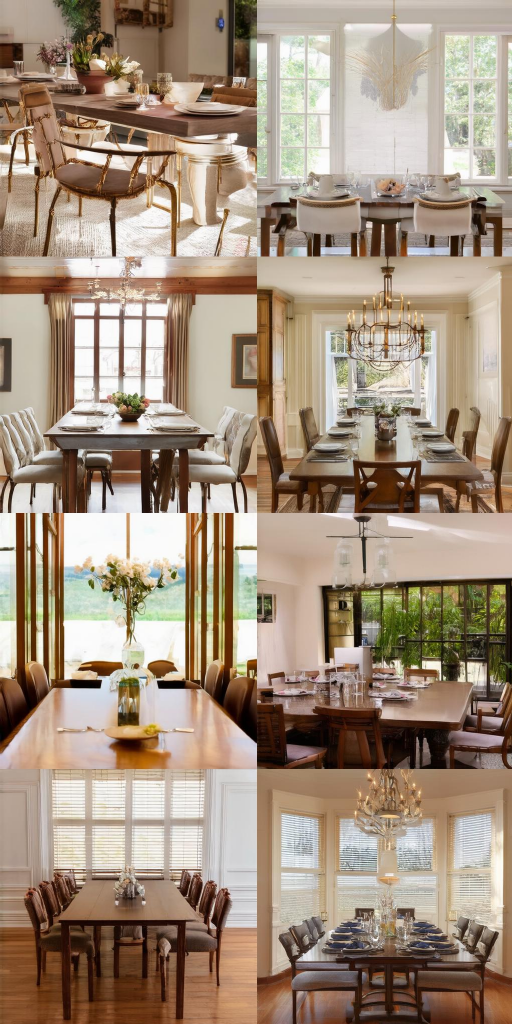}
        \caption{$t_{gate}=0.5$}
        \label{fig:gate_0_5}
    \end{subfigure}
    \hfill
    \begin{subfigure}[b]{0.19\linewidth}
        \centering
        \includegraphics[width=\linewidth]{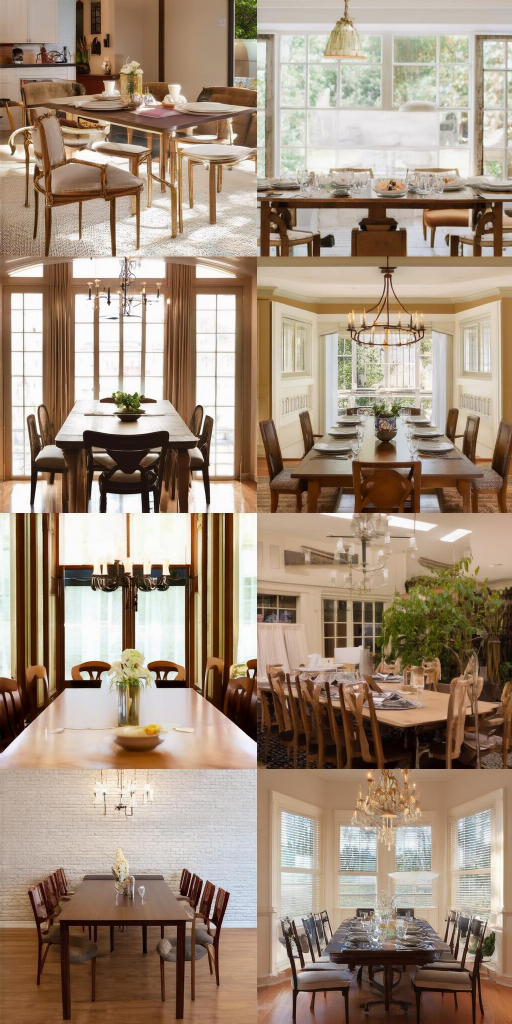}
        \caption{$t_{gate}=0.4$}
        \label{fig:gate_0_6}
    \end{subfigure}
    \hfill
    \begin{subfigure}[b]{0.19\linewidth}
        \centering
        \includegraphics[width=\linewidth]{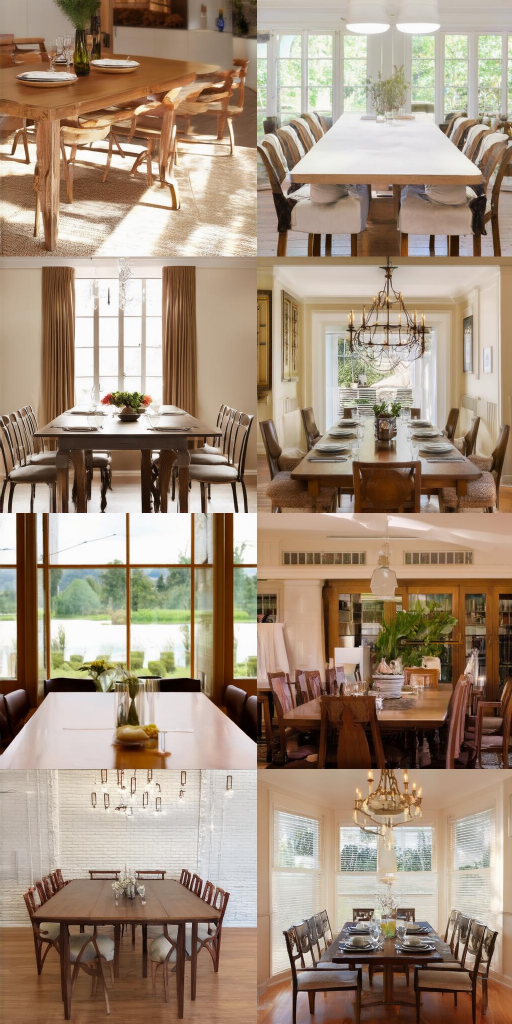}
        \caption{$t_{gate}=0.3$}
        \label{fig:gate_0_7}
    \end{subfigure}
    \hfill
    \begin{subfigure}[b]{0.19\linewidth}
        \centering
        \includegraphics[width=\linewidth]{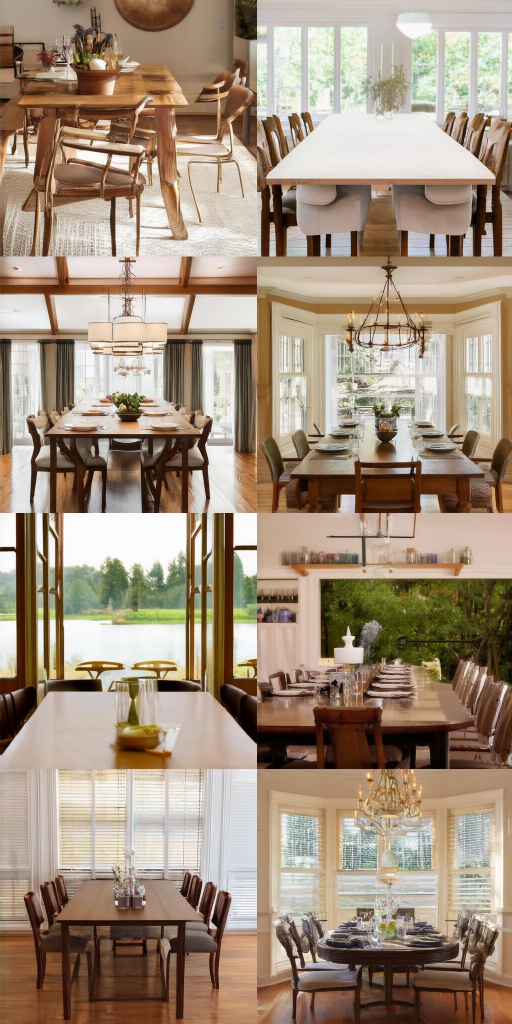}
        \caption{$t_{gate}=0.2$}
        \label{fig:gate_0_8}
    \end{subfigure}
    \hfill
    \begin{subfigure}[b]{0.19\linewidth}
        \centering
        \includegraphics[width=\linewidth]{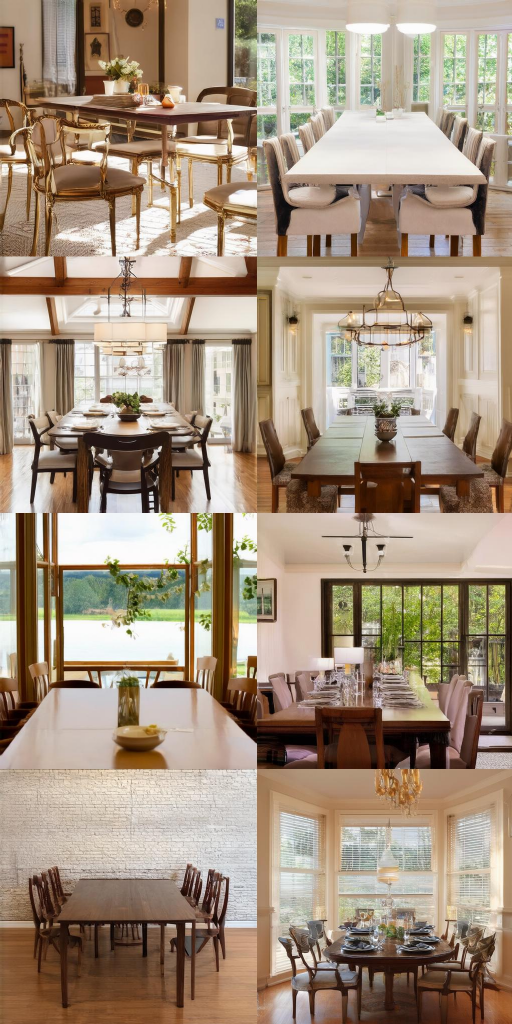}
        \caption{$t_{gate}=0.1$}
        \label{fig:gate_0_9}
    \end{subfigure}
    \caption{Visual comparison of generated images across different noise gates $[0.05, t_{gate}]$.}
    \label{fig:noise_gate_visual_ablation}
\end{figure}

\subsubsection{Robustness to Noise Scale ($s$)}
\label{sec:add_noise_scale}
Similarly, we evaluate the method's robustness to the magnitude of the injected noise, controlled by the \textit{Noise Scale} parameter. As shown quantitatively in Table~\ref{tab:noisescale_comparison} and visually in Figure~\ref{fig:noise_scale_visual_ablation}, our method's performance remains highly consistent across a wide range of noise scales, from $s=0.25$ to $s=5.0$. While an excessively large scale ($s=10.0$) begins to degrade perceptual quality, the overall stability of metrics like FID and Vendi Score Inception across nearly two orders of magnitude highlights the robustness of our approach. The results suggest a broad optimal operating region around $s=2.0$, where multiple metrics are jointly optimized.

\begin{table}[htbp]
\centering
\caption{
    Robustness analysis for the noise scale ($s$). Performance is highly stable for $s \in [0.25, 5.0]$, with a clear optimal region around $s=2.0$ where multiple quality and diversity metrics are maximized. An excessive scale of $s=10.0$ leads to a degradation in perceptual quality.
}
\label{tab:noisescale_comparison}
\resizebox{\textwidth}{!}{%
\begin{tabular}{l ccccc}
\toprule
\textbf{Noise Scale} & \makecell{\textbf{CLIP Score} $\uparrow$} & \makecell{\textbf{Vendi Score Pixel} $\uparrow$} & \makecell{\textbf{Vendi Score Inception}  $\uparrow$} & \makecell{\textbf{FID} $\downarrow$} & \makecell{\textbf{Brisque} $\downarrow$} \\
\midrule
0.25 & 14.69 $\pm$ 2.63 & 2.05 $\pm$ 0.15 & 3.70 $\pm$ 0.17 & 134.54 $\pm$ 2.77 & 23.59 $\pm$ 4.64 \\
0.5 & 14.11 $\pm$ 1.43 & 2.08 $\pm$ 0.19 & 3.74 $\pm$ 0.16 & 135.24 $\pm$ 1.45 & 24.82 $\pm$ 4.88 \\
1.0 & 15.93 $\pm$ 0.91 & 2.07 $\pm$ 0.16 & 3.81 $\pm$ 0.13 & 134.14 $\pm$ 1.77 & 22.64 $\pm$ 4.54 \\
2.0 & 16.11 $\pm$ 0.74 & 2.12 $\pm$ 0.19 & 3.84 $\pm$ 0.14 & 134.38 $\pm$ 2.83 & 22.18 $\pm$ 3.28 \\
5.0 & 15.39 $\pm$ 0.83 & 2.06 $\pm$ 0.18 & 3.81 $\pm$ 0.06 & 133.71 $\pm$ 2.70 & 24.06 $\pm$ 4.95 \\
10.0 & 12.29 $\pm$ 0.60 & 2.01 $\pm$ 0.16 & 3.58 $\pm$ 0.22 & 134.55 $\pm$ 1.75 & 28.51 $\pm$ 5.55 \\
\bottomrule
\end{tabular}%
}
\end{table}

\begin{figure}[htbp]
    \centering

    \begin{subfigure}[b]{0.155\linewidth}
        \centering
        \includegraphics[width=\linewidth]{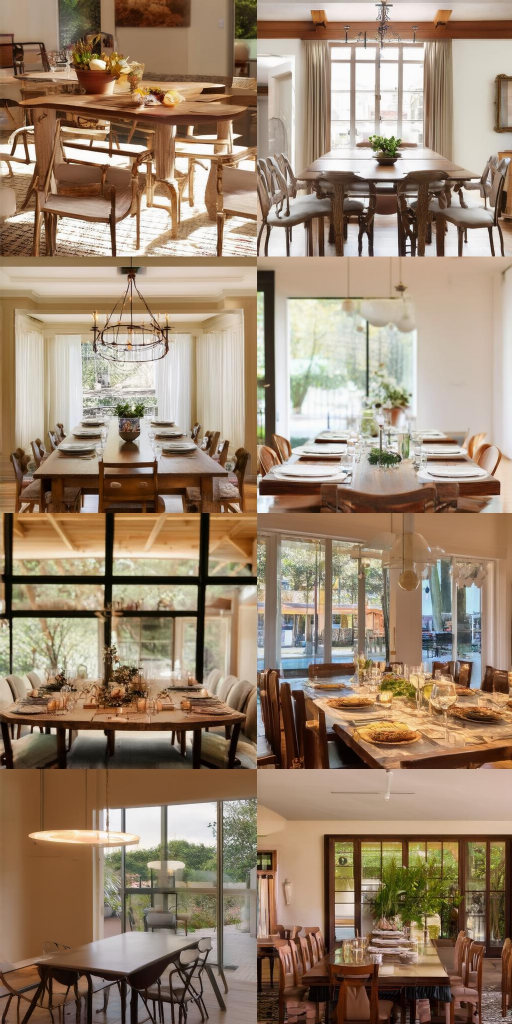}
        \caption{$s=0.25$}
        \label{fig:scale_0_25}
    \end{subfigure}
    \hfill
    \begin{subfigure}[b]{0.155\linewidth}
        \centering
        \includegraphics[width=\linewidth]{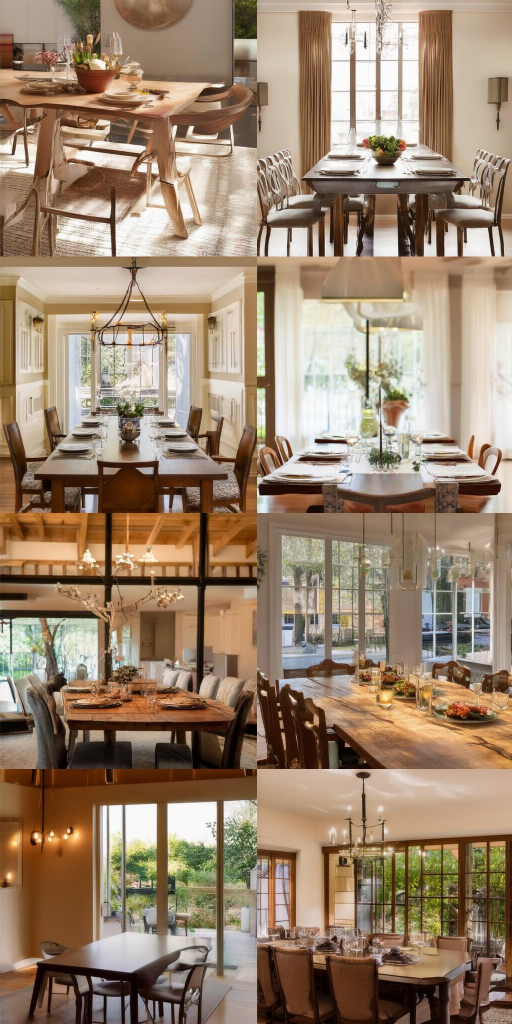}
        \caption{$s=0.5$}
        \label{fig:scale_0_5}
    \end{subfigure}
    \hfill
    \begin{subfigure}[b]{0.155\linewidth}
        \centering
        \includegraphics[width=\linewidth]{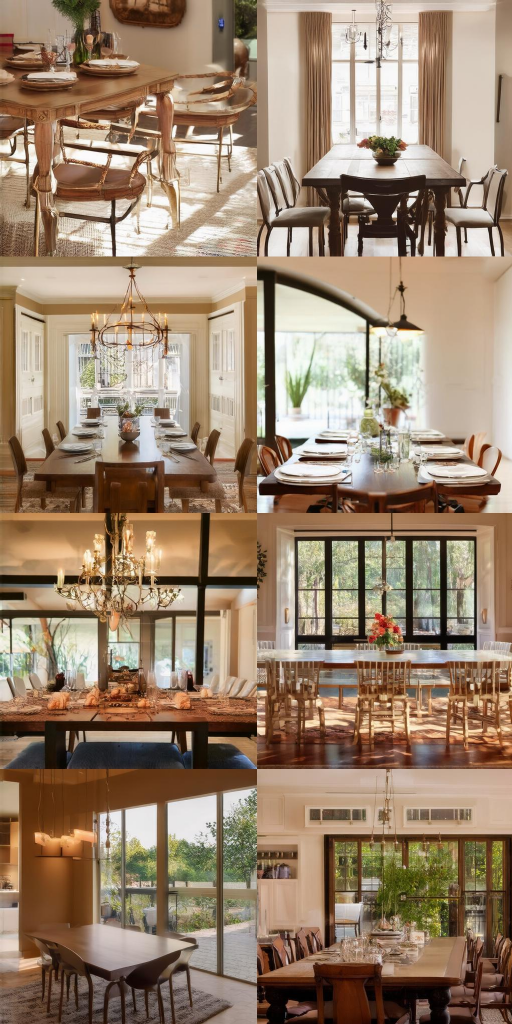}
        \caption{$s=1.0$}
        \label{fig:scale_1_0}
    \end{subfigure}
    \hfill
    \begin{subfigure}[b]{0.155\linewidth}
        \centering
        \includegraphics[width=\linewidth]{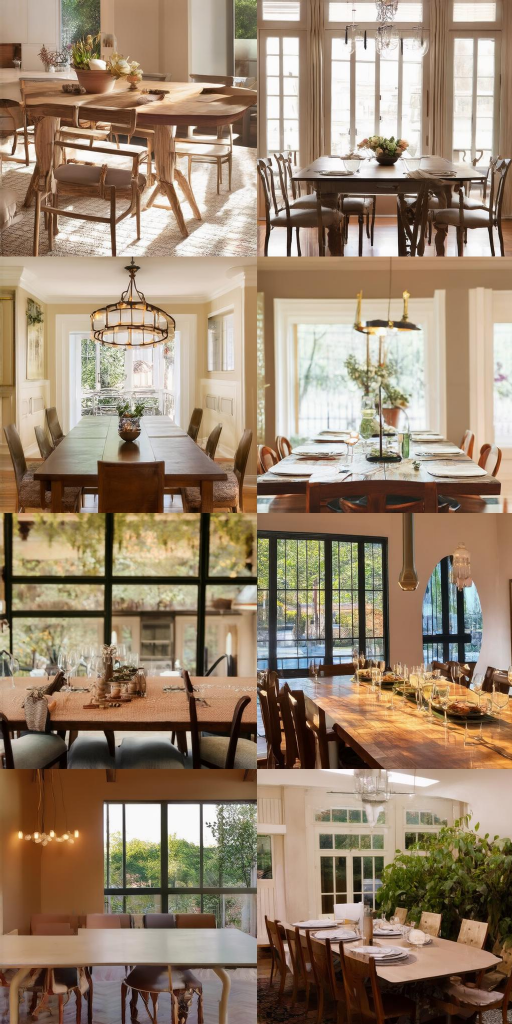}
        \caption{$s=2.0$}
        \label{fig:scale_2_0}
    \end{subfigure}
    \hfill
    \begin{subfigure}[b]{0.155\linewidth}
        \centering
        \includegraphics[width=\linewidth]{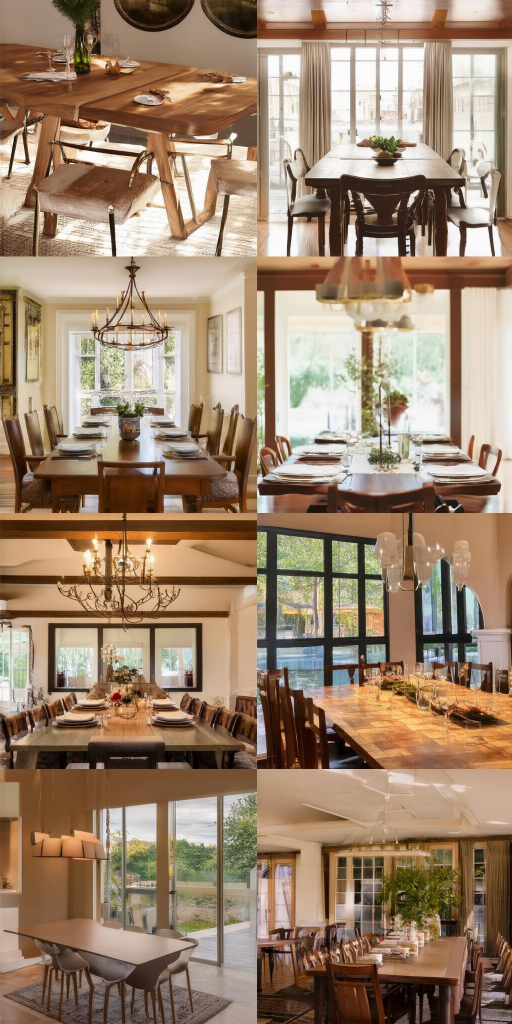}
        \caption{$s=5.0$}
        \label{fig:scale_5_0}
    \end{subfigure}
    \hfill
    \begin{subfigure}[b]{0.155\linewidth}
        \centering
        \includegraphics[width=\linewidth]{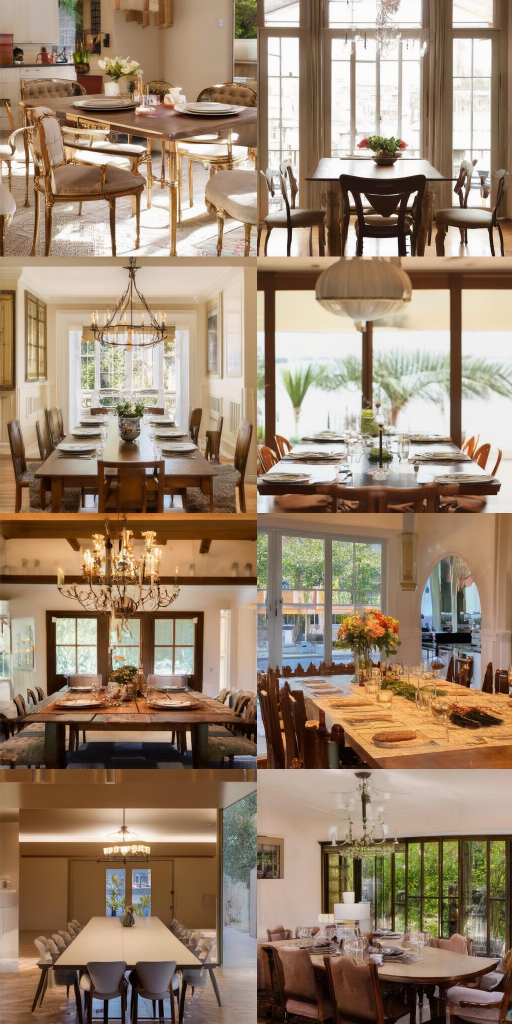}
        \caption{$s=10.0$}
        \label{fig:scale_10_0}
    \end{subfigure}

    \caption{Visual comparison of generated images across different noise scales ($s$).}
    \label{fig:noise_scale_visual_ablation}
\end{figure}

\subsubsection{Robustness to Noise Schedule (cos2 vs t1mt)}
\label{sec:appendix_noise_schedules}

We compare two time schedules for the exploration term used in our sampler:
(i) a cosine-squared schedule, $\gamma(t)=\cos^2\!\big(\pi\,s(t)\big)$, and
(ii) a parabolic schedule, $\gamma(t)=s(t)\,(1-s(t))$, where $s(t)\in[0,1]$ is the normalized time.
Both profiles are bell-shaped; \texttt{cos2} ramps up smoothly from zero, peaks at mid-trajectory, and then decays smoothly, while \texttt{t1mt} follows an inverted-parabola rise–fall pattern.
Unless otherwise stated, all settings are kept identical across schedules.

\paragraph{Quantitative results.}
Table~\ref{tab:noise_sched_ablation} summarizes an ablation at a fixed guidance of $\mathrm{CFG}=3.0$. We observe a small trade-off: \texttt{cos2} yields slightly higher CLIP Score and Vendi Score Pixel, which means more low-level variation, whereas \texttt{t1mt} marginally improves Vendi Score Inception and achieves lower FID/Brisque. Both schedules are viable; \texttt{cos2} mildly favors exploration, while \texttt{t1mt} mildly favors fidelity.

\begin{table}[h]
\centering
\caption{ \small
    Ablation study on our fidelity safeguards at a fixed guidance of $CFG=3.0$. The results confirm that our method is robust to different $\beta_t$ schedules.
}
\label{tab:noise_sched_ablation}
    \begin{tabular}{l ccccc}
    \toprule
    \textbf{$\beta(t)$} & \makecell{\textbf{CLIP Score} $\uparrow$} & \makecell{\textbf{Vendi Score Pixel} $\uparrow$} & \makecell{\textbf{Vendi Score Inception} $\uparrow$} & \makecell{\textbf{FID}  $\downarrow$} & \makecell{\textbf{Brisque} $\downarrow$} \\
    \midrule
    cos2 & 26.61 $\pm$ 0.07 & 2.66 $\pm$ 0.21 & 4.81 $\pm$ 0.09 & 126.52 $\pm$ 0.60 & 20.26 $\pm$ 2.07 \\
    t1mt & 25.26 $\pm$ 0.30 & 2.59 $\pm$ 0.02 & 5.05 $\pm$ 0.01 & 125.45 $\pm$ 2.06 & 18.53 $\pm$ 6.61 \\
    \bottomrule
    \end{tabular}%
\end{table}

\paragraph{Qualitative results.}
Figure~\ref{fig:noise_sched_vis} provides side-by-side visualizations at the same guidance and NFE.
Consistent with the quantitative trends, \texttt{cos2} tends to distribute samples more broadly (more visible pixel-level variation), while \texttt{t1mt} produces slightly cleaner images with comparable semantic coverage.

\begin{figure}[t]
\centering
\begin{subfigure}{0.48\textwidth}
  \centering
  \includegraphics[width=\linewidth]{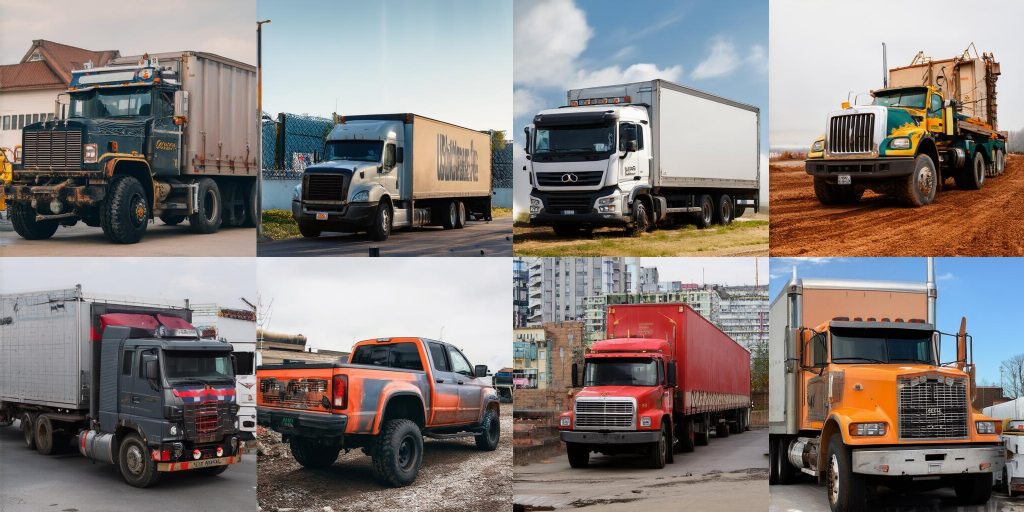}
  \caption{\texttt{cos2}}
\end{subfigure}
\hfill
\begin{subfigure}{0.48\textwidth}
  \centering
  \includegraphics[width=\linewidth]{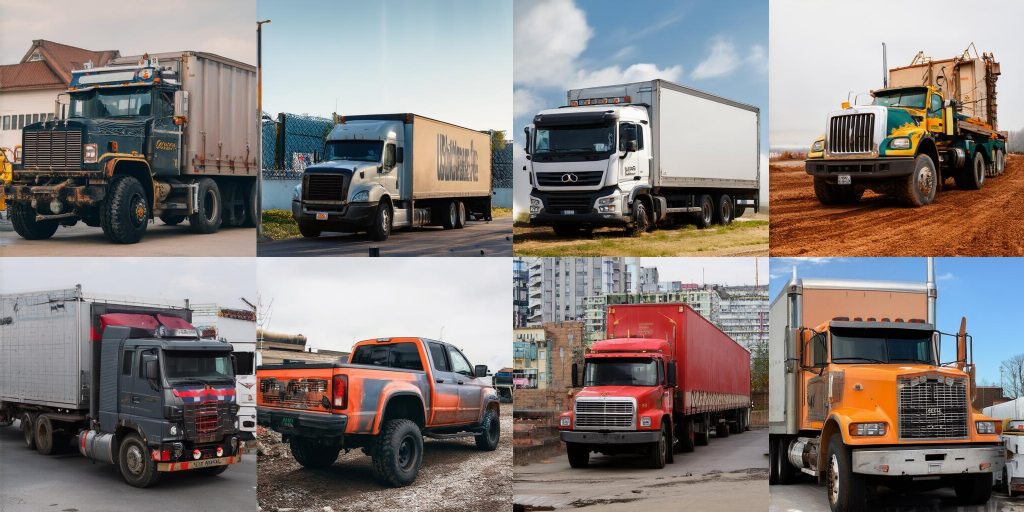}
  \caption{\texttt{t1mt}}
\end{subfigure}
\caption{Qualitative comparison of the two time schedules at the same guidance and NFE.}
\label{fig:noise_sched_vis}
\end{figure}

\section{Limitations}

Although OSCAR demonstrates consistent effectiveness across diverse experimental settings, it still has several unavoidable limitations. First, our method is designed for set-level generation and relies on a jointly active group of trajectories. Therefore, it is not suitable for the strict one-image-per-prompt setting, where no within-prompt sample set is available for computing diversity guidance. Second, OSCAR introduces additional inference-time overhead due to endpoint feature extraction and the set-level diversity gradient computation. While this cost is manageable in moderate-batch scenarios, it is still higher than the original flow-matching sampler. Reducing this overhead and extending the method to asynchronous or memory-based generation settings are promising directions for future work.

\section{More Visual Results}
\FloatBarrier
\label{sec:appendix_more_visuals}

This section provides additional qualitative results to supplement our main findings. We first present extended comparisons for the 512×512 class-conditional generation task on COCO concepts, as shown in \Cref{fig:main_visual_comparison_grid,fig:main_visual_comparison_grid2}. We further showcase visual results from our DIM and CIM evaluations on both coarse and dense prompts, as illustrated in \Cref{fig:coarse_prompt_comparison,fig:dense_prompt_comparison}. These comparisons visually demonstrate the effectiveness of our method in enhancing sample diversity compared to baselines.

\begin{figure*}[htbp]
    \centering
    \caption{
    Comprehensive visual ablation of our fidelity safeguards. Each column corresponds to one variant: (a) w/o OP, (b) w/o RR, (c) w/o OP \& RR, and (d) full OSCAR. Each row shows results for a different prompt, from top to bottom: ``A photo of a bowl'', ``A photo of a truck''.
    }
    \label{fig:main_visual_comparison_grid}

    \begin{subfigure}[t]{0.24\linewidth}
        \centering
        \caption{w/o OP}
        \includegraphics[width=\linewidth]{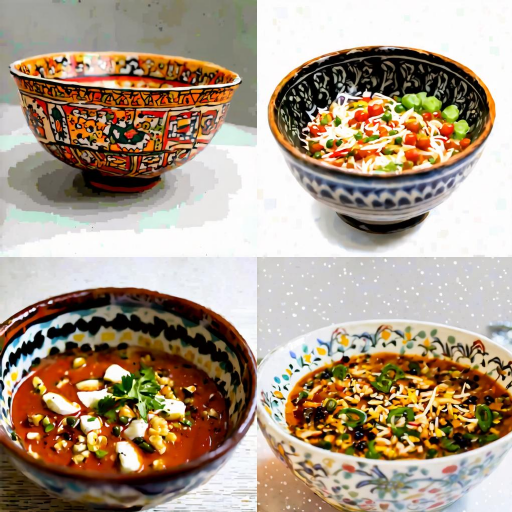}
        \vspace{2pt} 
        \includegraphics[width=\linewidth]{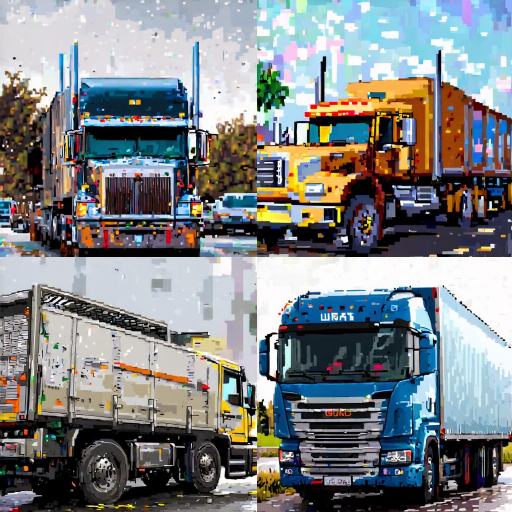}
        \label{fig:grid_op}
    \end{subfigure}
    \hfill
    \begin{subfigure}[t]{0.24\linewidth}
        \centering
        \caption{w/o RR}
        \includegraphics[width=\linewidth]{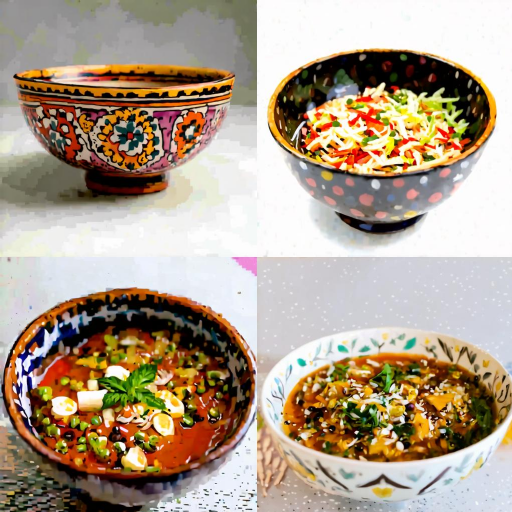}
        \vspace{2pt} 
        \includegraphics[width=\linewidth]{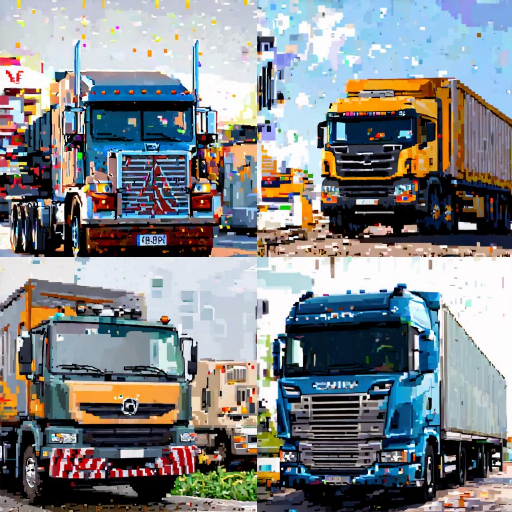}
        \label{fig:grid_RR}
    \end{subfigure}
    \hfill
    \begin{subfigure}[t]{0.24\linewidth}
        \centering
        \caption{w/o OP \& RR}
        \includegraphics[width=\linewidth]{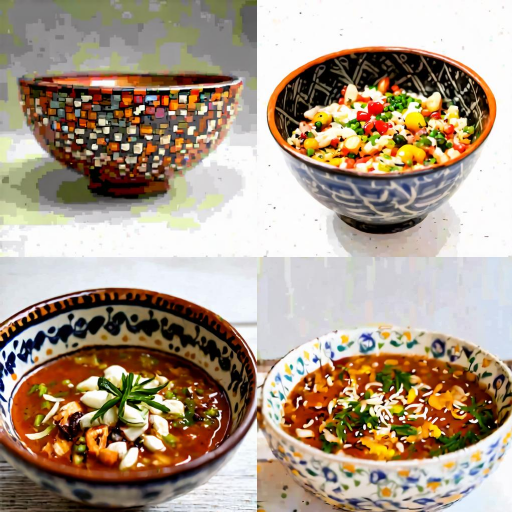}
        \vspace{2pt} 
        \includegraphics[width=\linewidth]{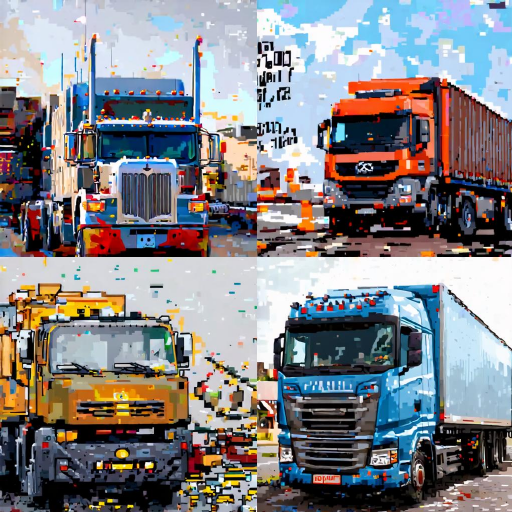}
        \label{fig:grid_op_rr}
    \end{subfigure}
    \hfill
    \begin{subfigure}[t]{0.24\linewidth}
        \centering
        \caption{Our Method}
        \includegraphics[width=\linewidth]{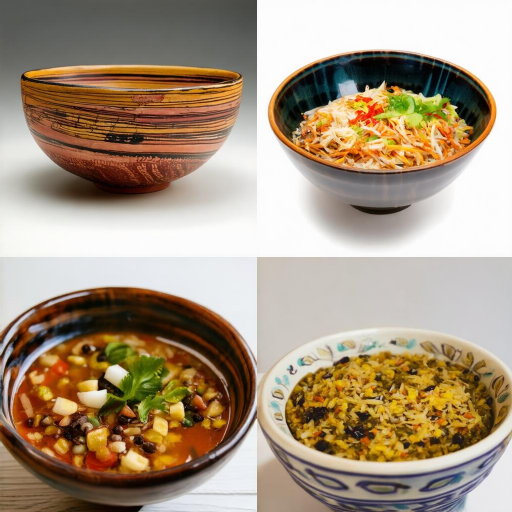}
        \vspace{2pt} 
        \includegraphics[width=\linewidth]{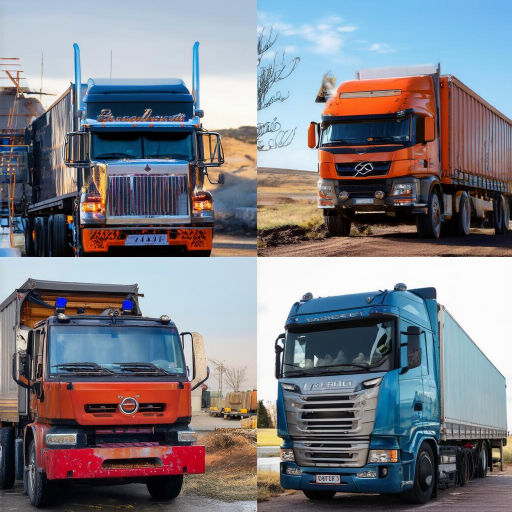}
        \label{fig:grid_ourmethod}
    \end{subfigure}
    \label{fig:visual_comparison_truck_bus_bicycle}
\end{figure*}

\begin{figure*}[H]
    \centering

    \label{fig:main_visual_comparison_grid}

    \begin{subfigure}[t]{0.24\linewidth}
        \centering
        \caption{w/o OP}
        \includegraphics[width=\linewidth]{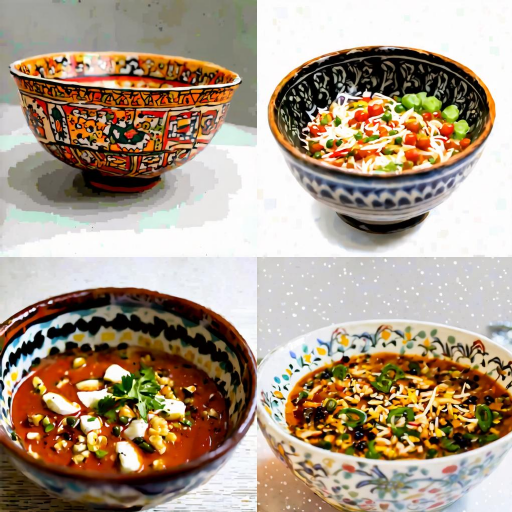}
        \vspace{2pt} 
        \includegraphics[width=\linewidth]{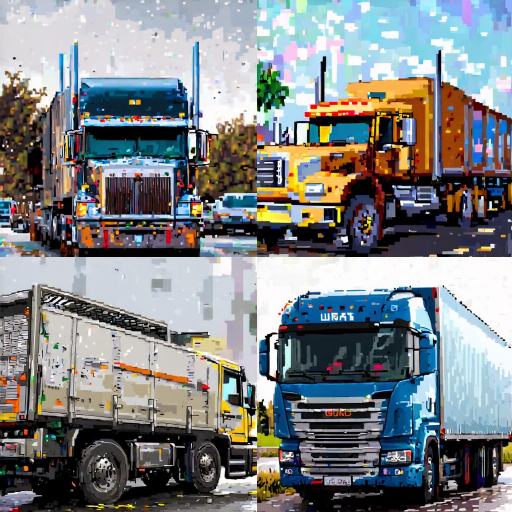}
        \label{fig:grid_op}
    \end{subfigure}
    \hfill
    \begin{subfigure}[t]{0.24\linewidth}
        \centering
        \caption{w/o RR}
        \includegraphics[width=\linewidth]{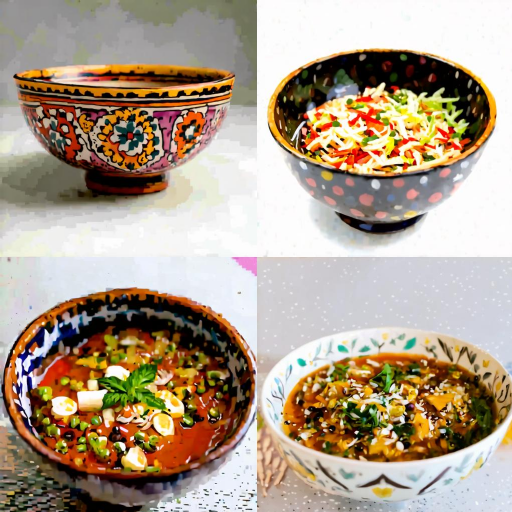}
        \vspace{2pt} 
        \includegraphics[width=\linewidth]{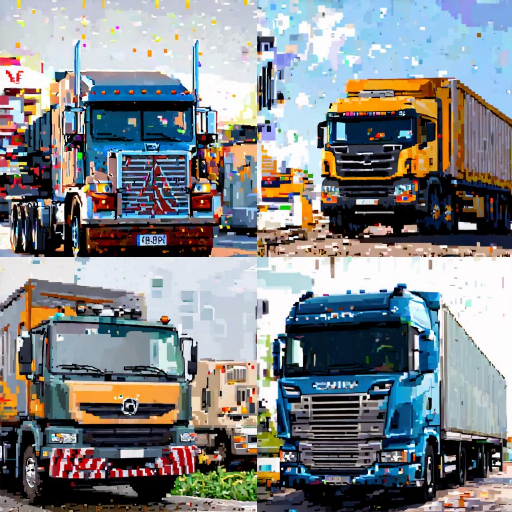}
        \label{fig:grid_RR}
    \end{subfigure}
    \hfill
    \begin{subfigure}[t]{0.24\linewidth}
        \centering
        \caption{w/o OP \& RR}
        \includegraphics[width=\linewidth]{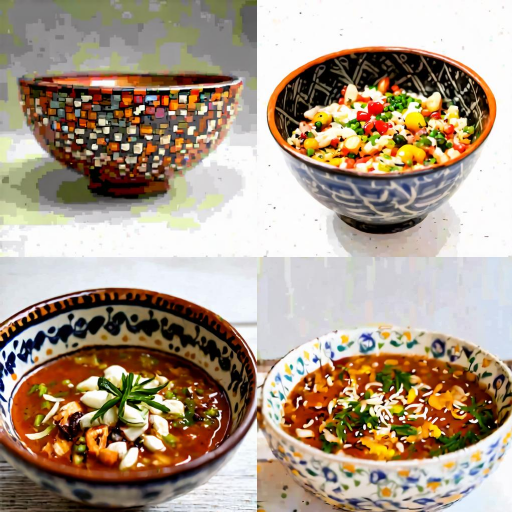}
        \vspace{2pt} 
        \includegraphics[width=\linewidth]{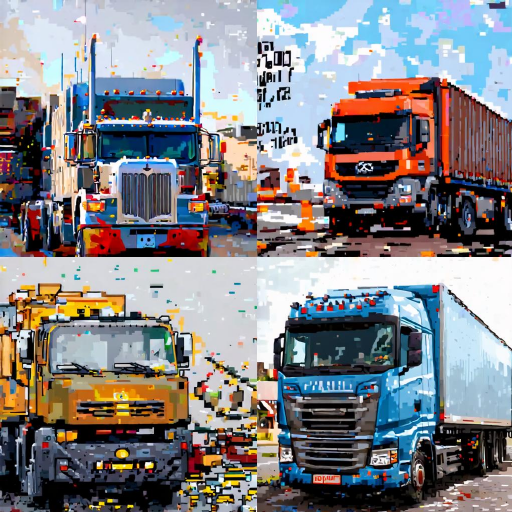}
        \label{fig:grid_op_rr}
    \end{subfigure}
    \hfill
    \begin{subfigure}[t]{0.24\linewidth}
        \centering
        \caption{Our Method}
        \includegraphics[width=\linewidth]{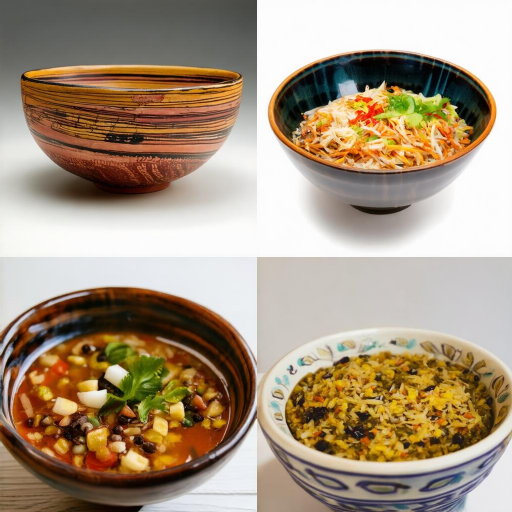}
        \vspace{2pt} 
        \includegraphics[width=\linewidth]{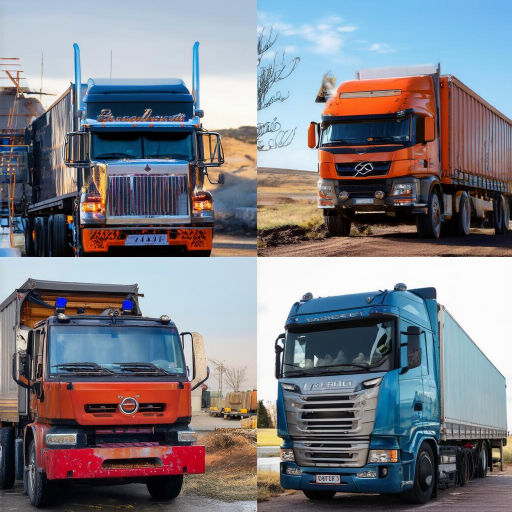}
        \label{fig:grid_ourmethod}
    \end{subfigure}
    \label{fig:visual_comparison_truck_bus_bicycle}
        \caption{
    Comprehensive visual ablation of our fidelity safeguards. Each column corresponds to one variant: (a) w/o OP, (b) w/o RR, (c) w/o OP \& RR, and (d) full OSCAR. Each row shows results for a different prompt, from top to bottom: ``A photo of a bowl'', ``A photo of a truck''.
    }
\end{figure*}

\begin{figure*}[htbp]
    \centering

    \label{fig:main_visual_comparison_grid}

    \begin{subfigure}[t]{0.22\linewidth}
        \centering
        \caption{CADS}
        \includegraphics[width=\linewidth]{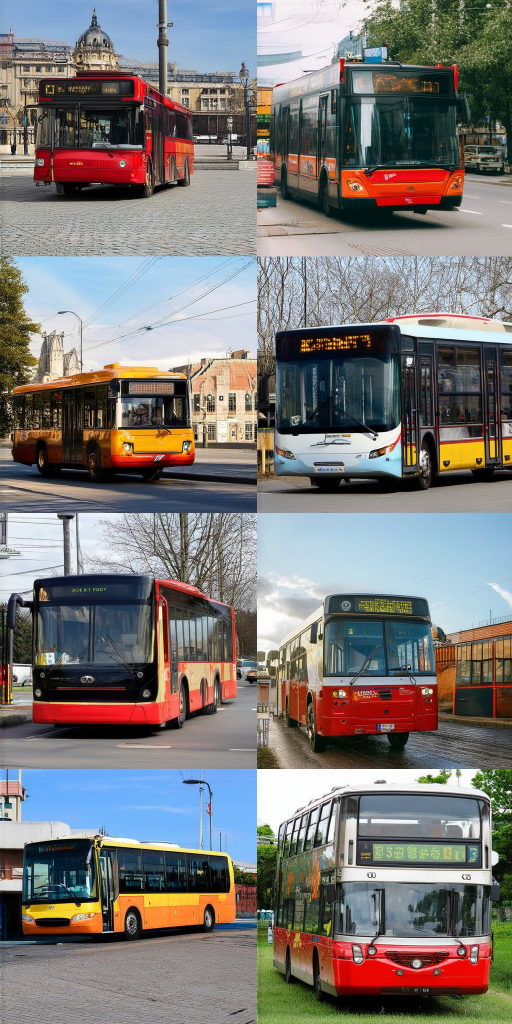}
        \vspace{2pt} 
        \includegraphics[width=\linewidth]{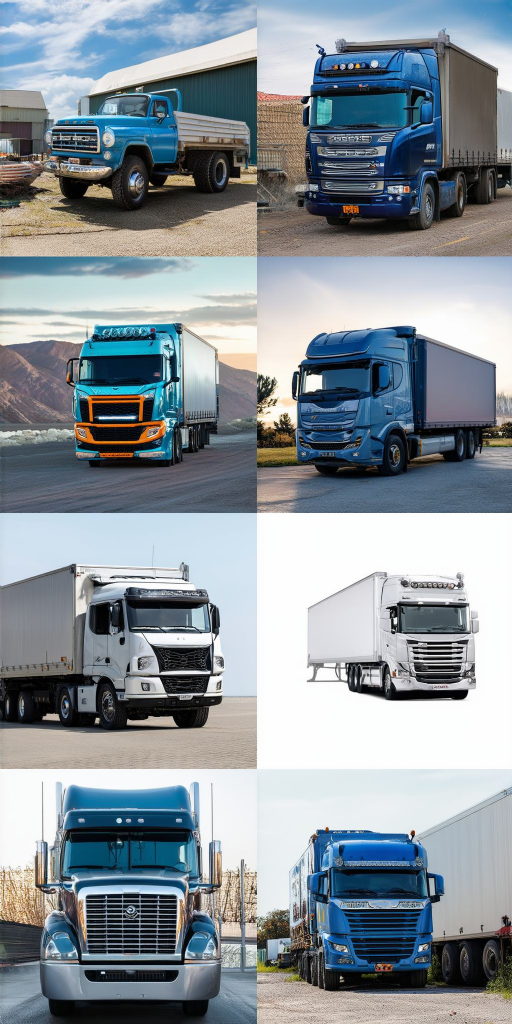}
        \vspace{2pt} 
        \includegraphics[width=\linewidth]{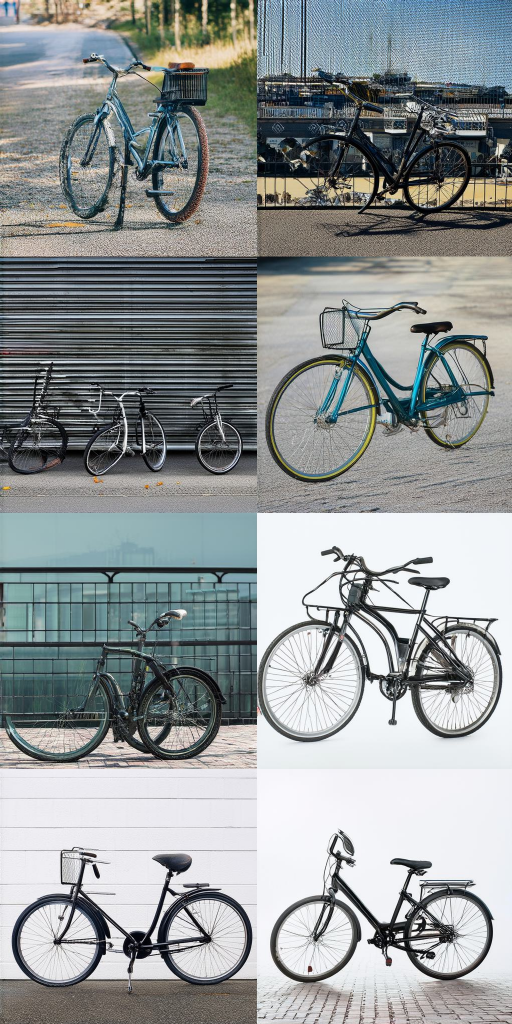}
        \label{fig:grid_cads}
    \end{subfigure}
    \hfill
    \begin{subfigure}[t]{0.22\linewidth}
        \centering
        \caption{DPP}
        \includegraphics[width=\linewidth]{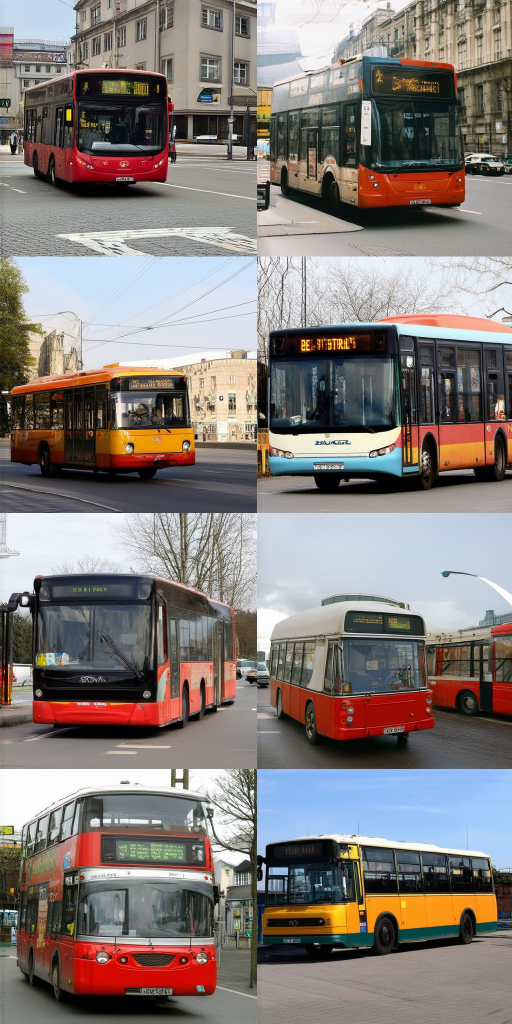}
        \vspace{2pt} 
        \includegraphics[width=\linewidth]{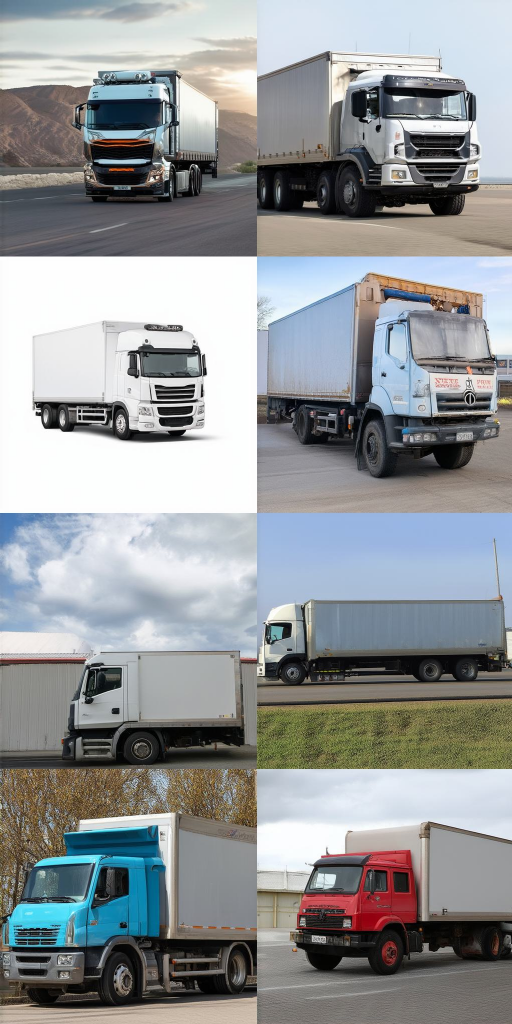}
        \vspace{2pt} 
        \includegraphics[width=\linewidth]{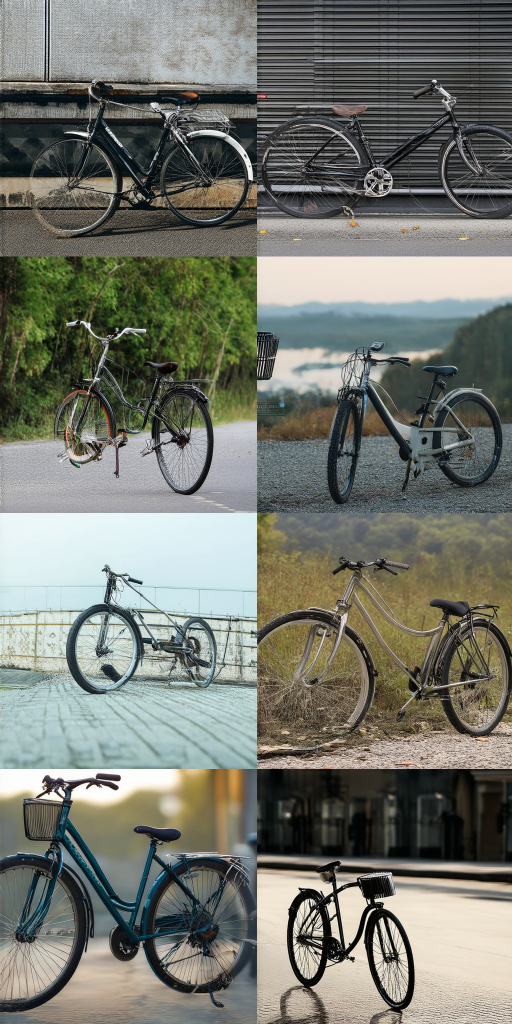}
        \label{fig:grid_dpp}
    \end{subfigure}
    \hfill
    \begin{subfigure}[t]{0.22\linewidth}
        \centering
        \caption{PG}
        \includegraphics[width=\linewidth]{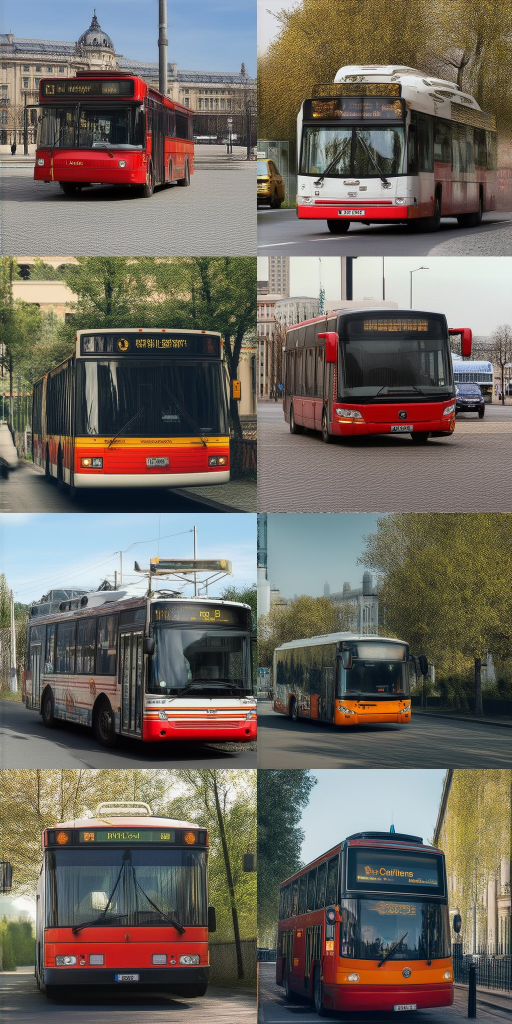}
        \vspace{2pt} 
        \includegraphics[width=\linewidth]{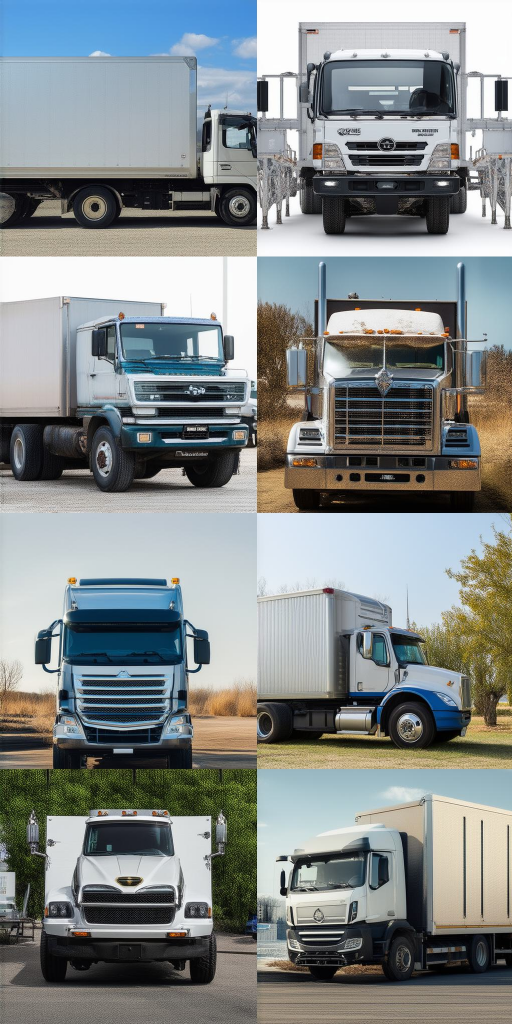}
        \vspace{2pt} 
        \includegraphics[width=\linewidth]{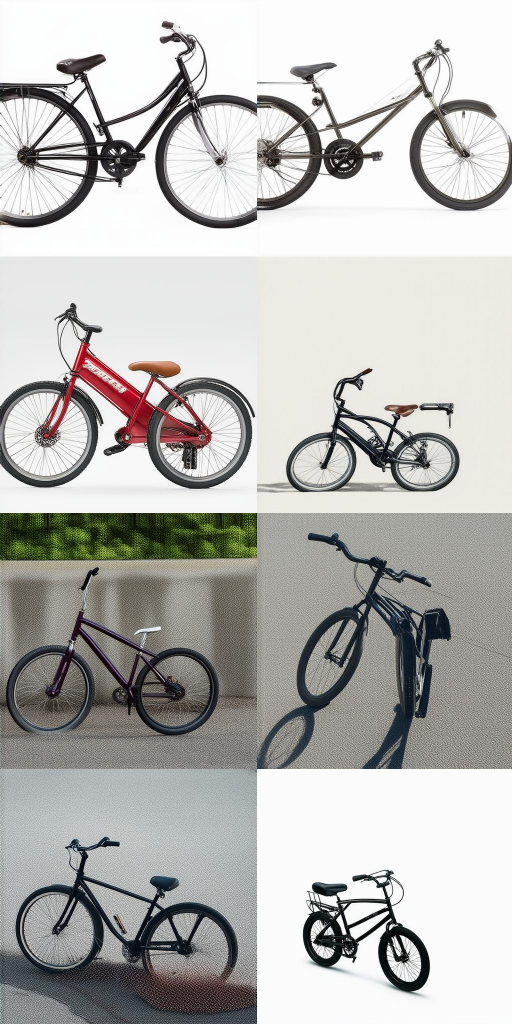}
        \label{fig:grid_pg}
    \end{subfigure}
    \hfill
    \begin{subfigure}[t]{0.22\linewidth}
        \centering
        \caption{Our Method}
        \includegraphics[width=\linewidth]{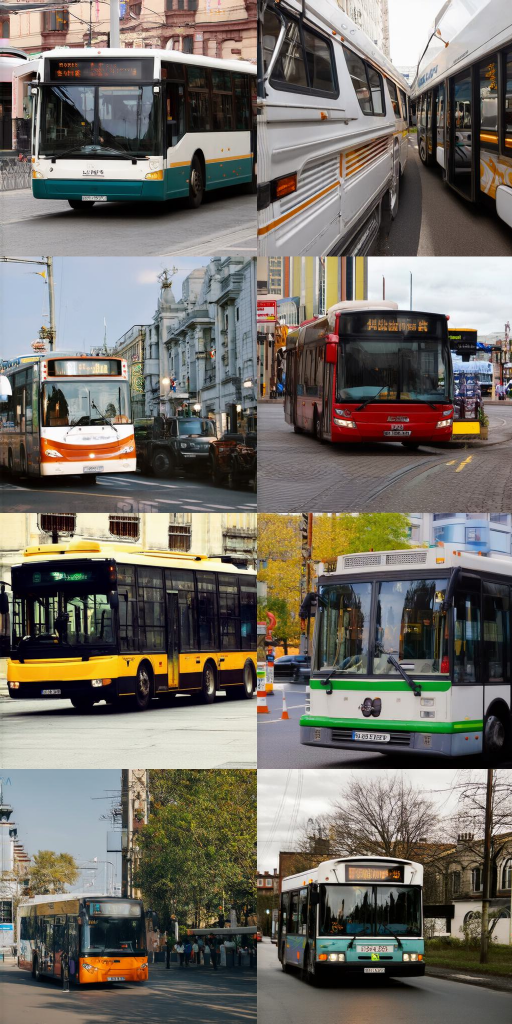}
        \vspace{2pt} 
        \includegraphics[width=\linewidth]{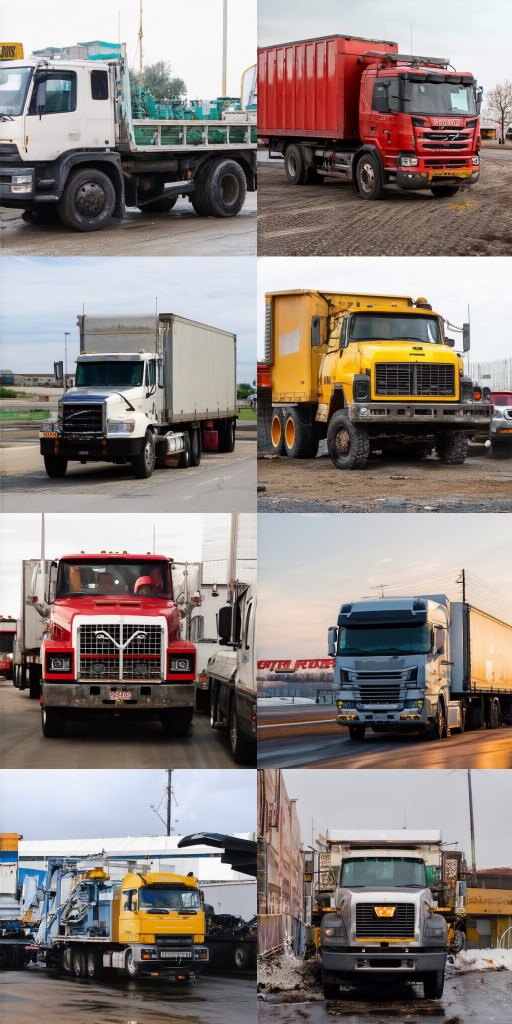}
        \vspace{2pt} 
        \includegraphics[width=\linewidth]{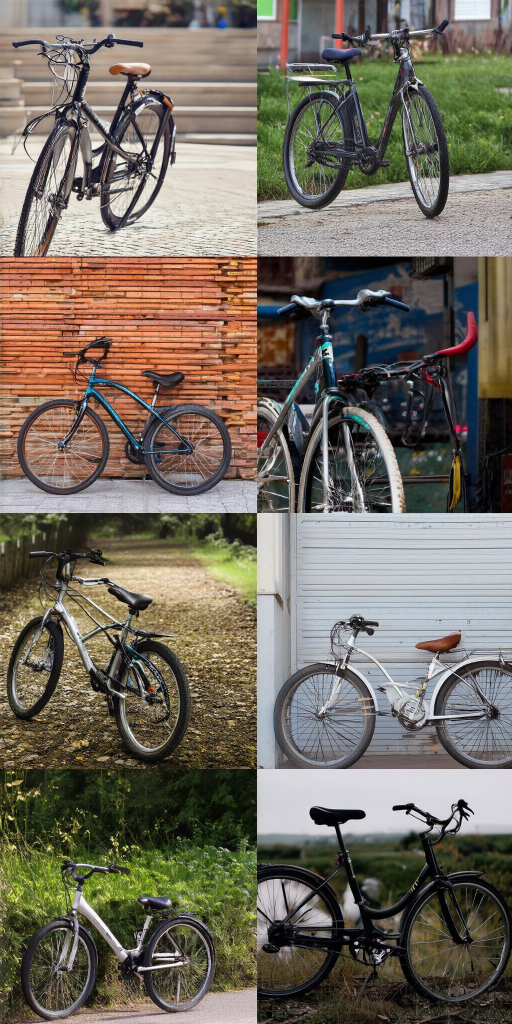}
        \label{fig:grid_ourmethod}
    \end{subfigure}
    \label{fig:visual_comparison_truck_bus_bicycle}
        \caption{
    Comprehensive visual comparison across all methods. Each column corresponds to a single method: (a) CADS, (b) DPP, (c) PG, and (d) Our Method. Each row shows results for a different prompt, in the following order from top to bottom: ``A photo of a bus”, ``A photo of a truck”, and ``A photo of a bicycle”.
    }
\end{figure*}

\begin{figure*}[htbp]
    \centering

    \label{fig:main_visual_comparison_grid2}

    \begin{subfigure}[t]{0.22\linewidth}
        \centering
        \caption{CADS}
        \includegraphics[width=\linewidth]{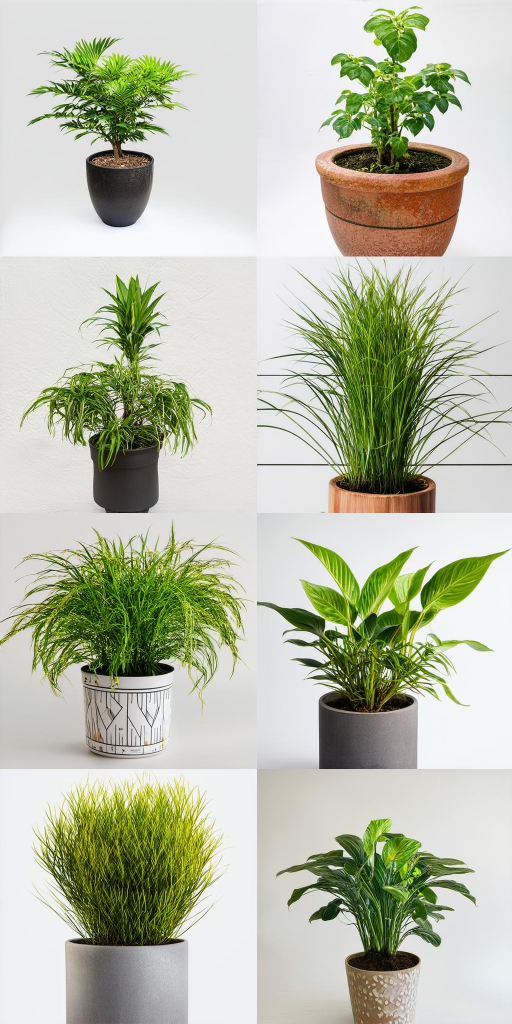}
        \vspace{2pt} 
        \includegraphics[width=\linewidth]{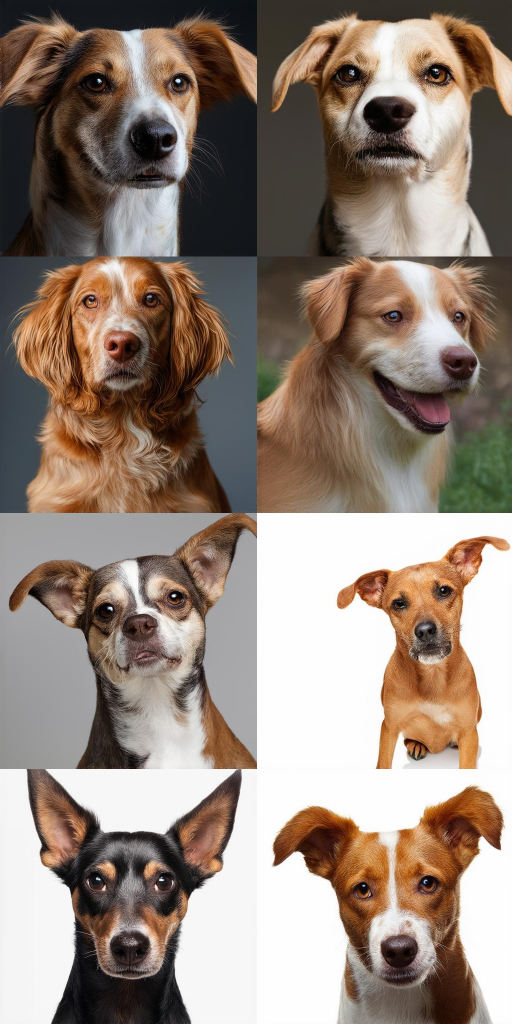}
        \vspace{2pt} 
        \includegraphics[width=\linewidth]{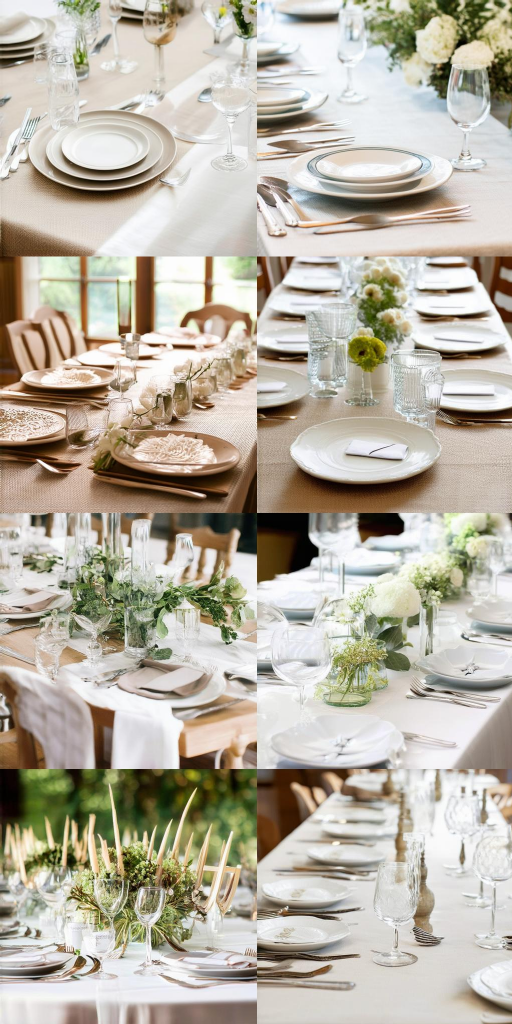}
        \label{fig:grid_cads}
    \end{subfigure}
    \hfill
    \begin{subfigure}[t]{0.22\linewidth}
        \centering
        \caption{DPP}
        \includegraphics[width=\linewidth]{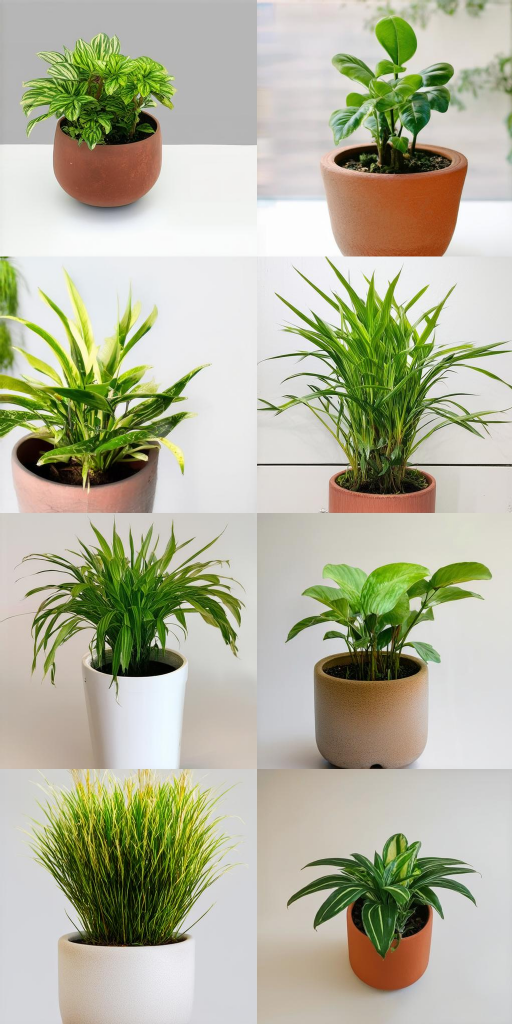}
        \vspace{2pt} 
        \includegraphics[width=\linewidth]{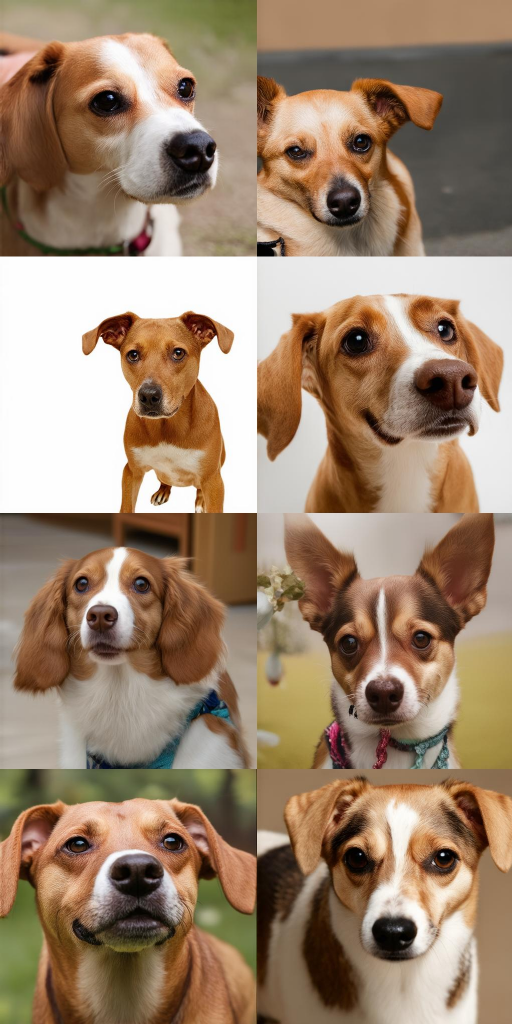}
        \vspace{2pt} 
        \includegraphics[width=\linewidth]{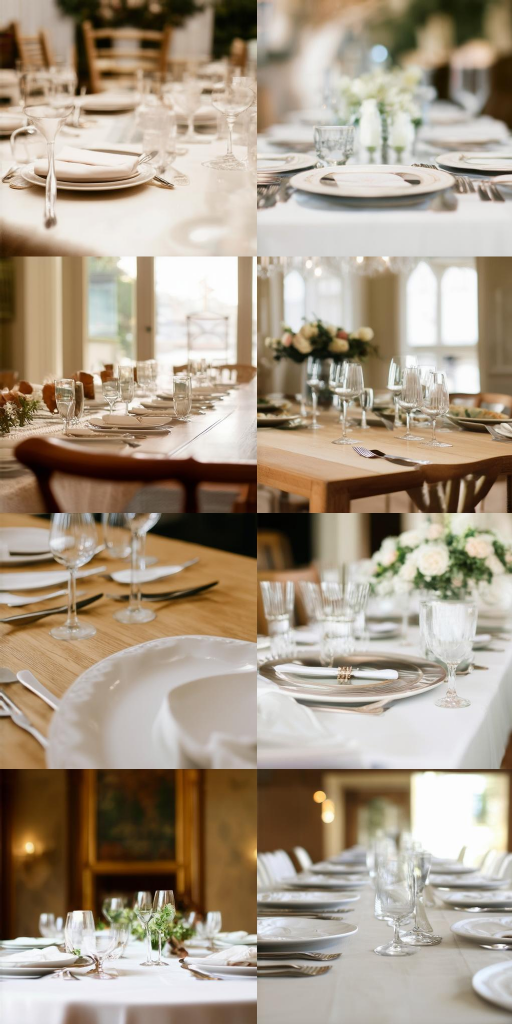}
        \label{fig:grid_dpp}
    \end{subfigure}
    \hfill
    \begin{subfigure}[t]{0.22\linewidth}
        \centering
        \caption{PG}
        \includegraphics[width=\linewidth]{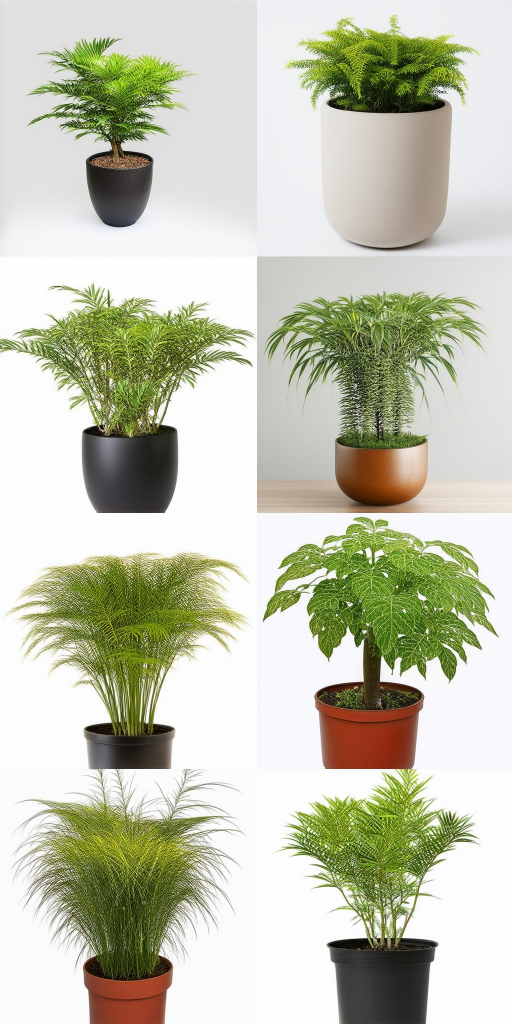}
        \vspace{2pt} 
        \includegraphics[width=\linewidth]{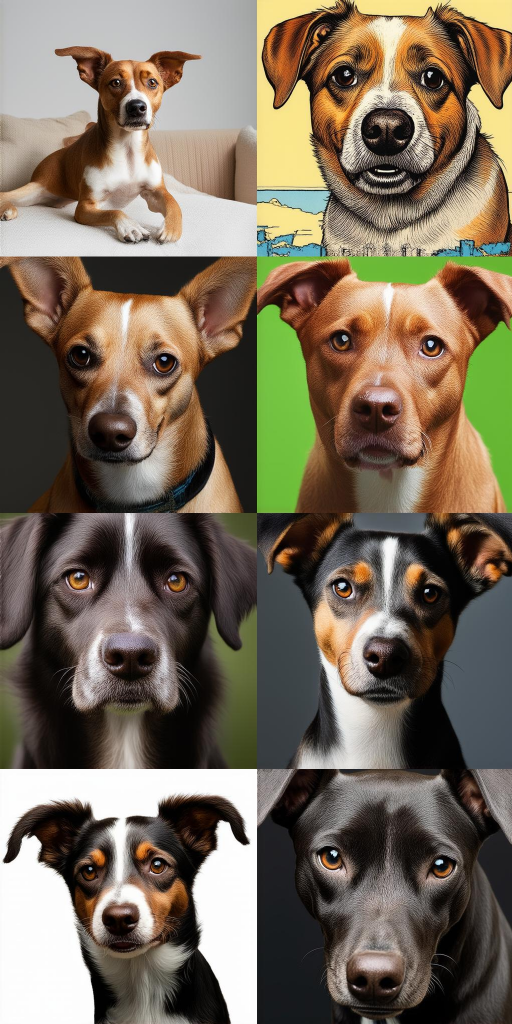}
        \vspace{2pt} 
        \includegraphics[width=\linewidth]{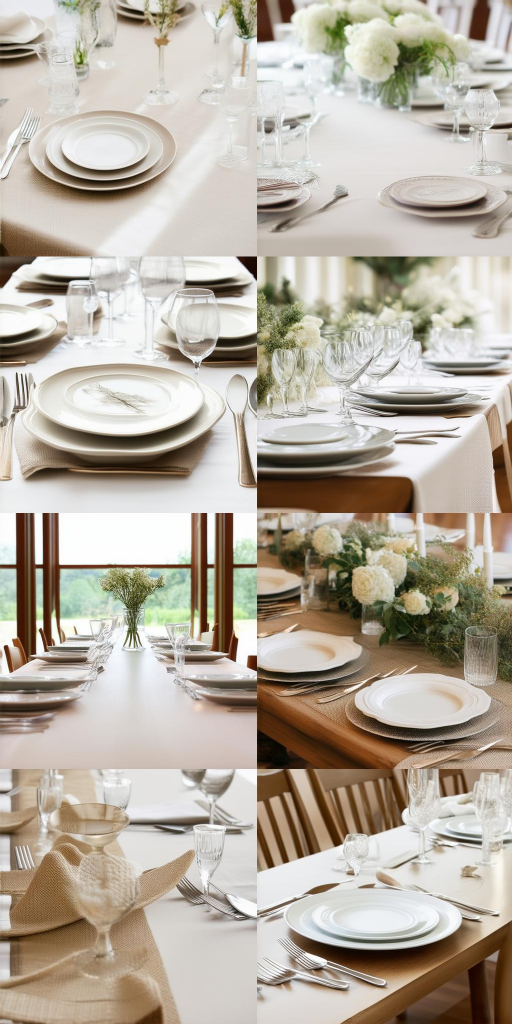}
        \label{fig:grid_pg}
    \end{subfigure}
    \hfill
    \begin{subfigure}[t]{0.22\linewidth}
        \centering
        \caption{Our Method}
        \includegraphics[width=\linewidth]{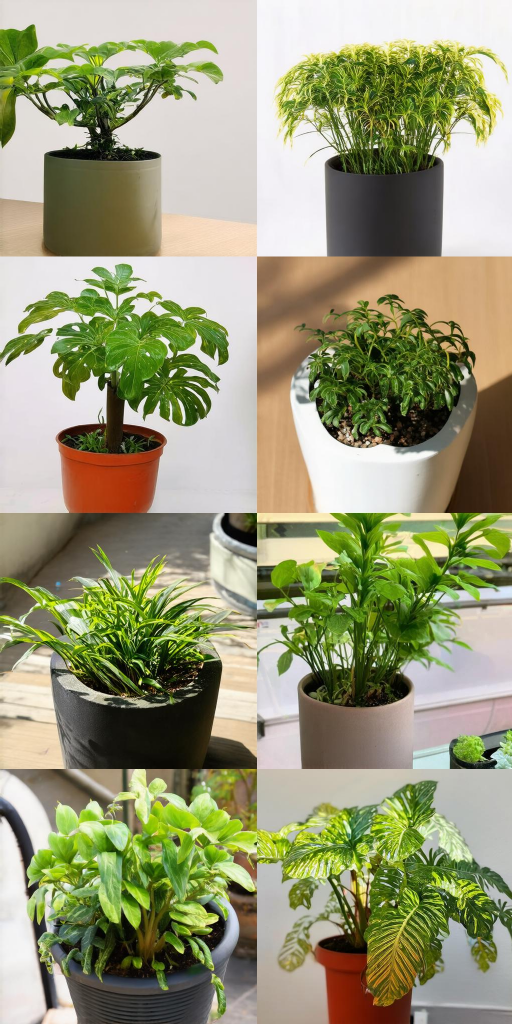}
        \vspace{2pt} 
        \includegraphics[width=\linewidth]{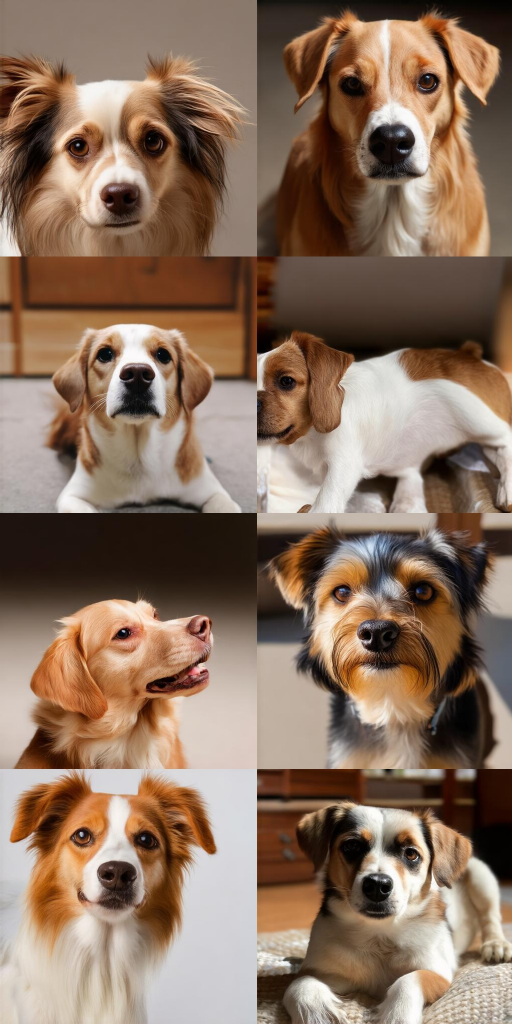}
        \vspace{2pt} 
        \includegraphics[width=\linewidth]{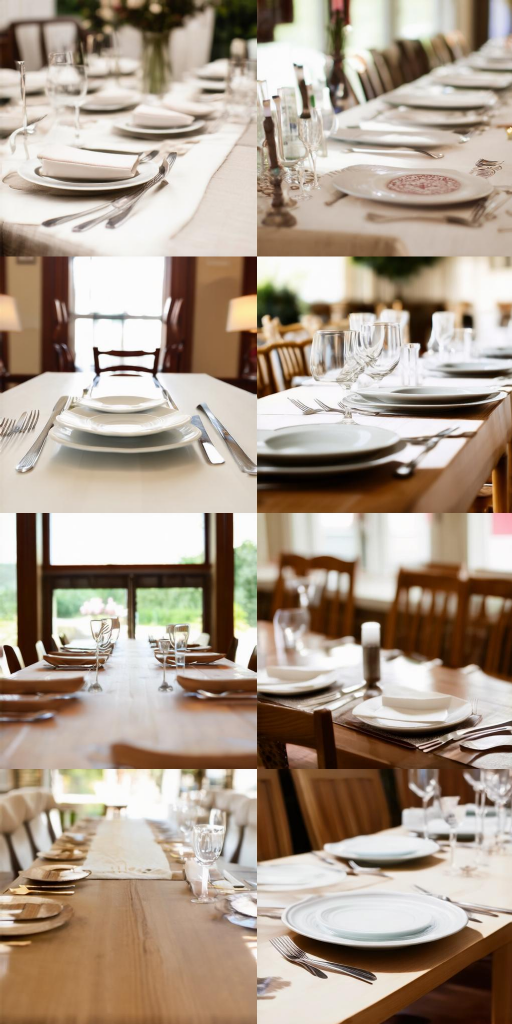}
        \label{fig:grid_ourmethod}
    \end{subfigure}
        \caption{
        Comprehensive visual comparison across all methods. Each column corresponds to a single method: (a) CADS, (b) DPP, (c) PG, and (d) Our Method. Each row shows results for a different prompt, in the following order from top to bottom: ``A photo of a potted plant", ``A photo of a dog", and ``A close-up photo of a dining table". 
    }
\end{figure*}

\begin{figure*}[htbp]
    \centering

    \label{fig:dense_prompt_comparison}
    \begin{subfigure}[t]{0.22\linewidth}
        \centering
        \caption{CADS}
        \includegraphics[width=\linewidth]{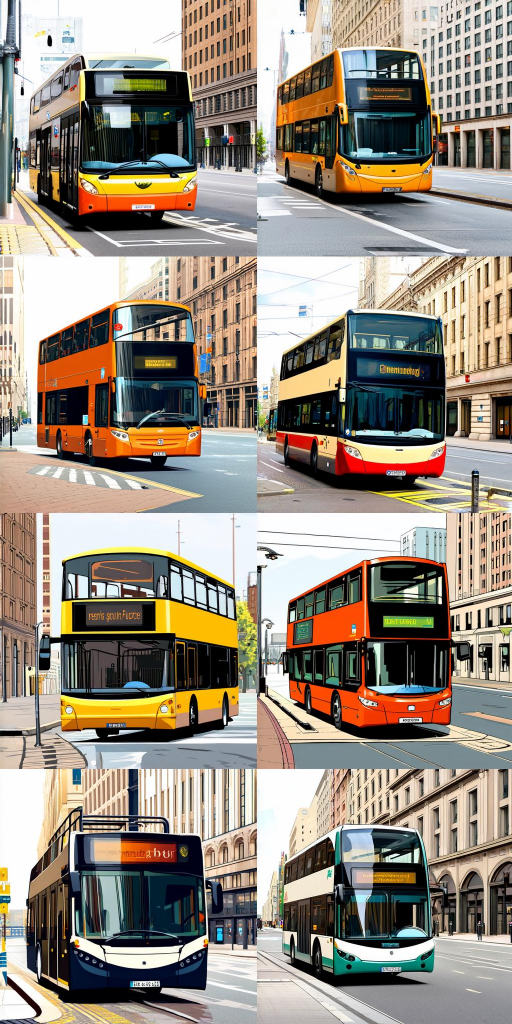}
        \vspace{2pt} 
        \includegraphics[width=\linewidth]{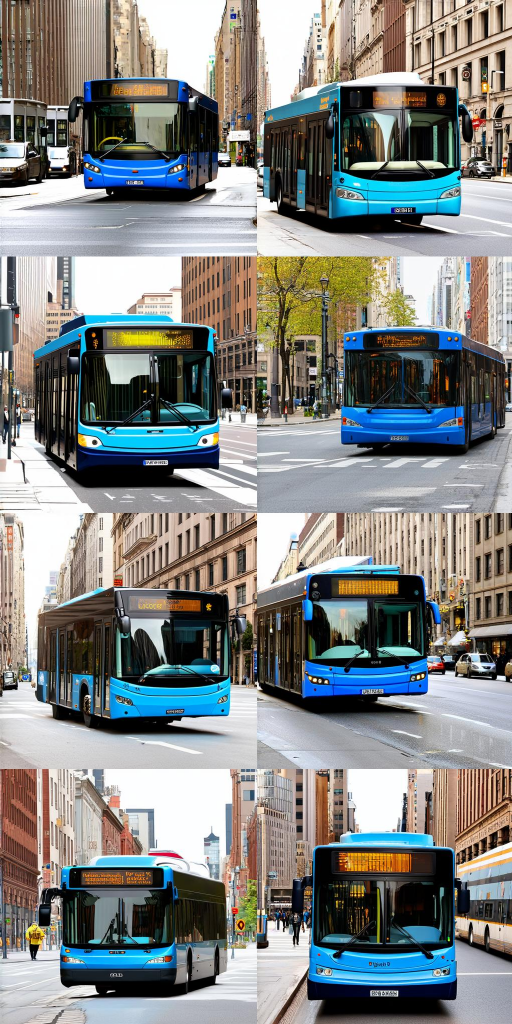}
        \vspace{2pt} 
        \includegraphics[width=\linewidth]{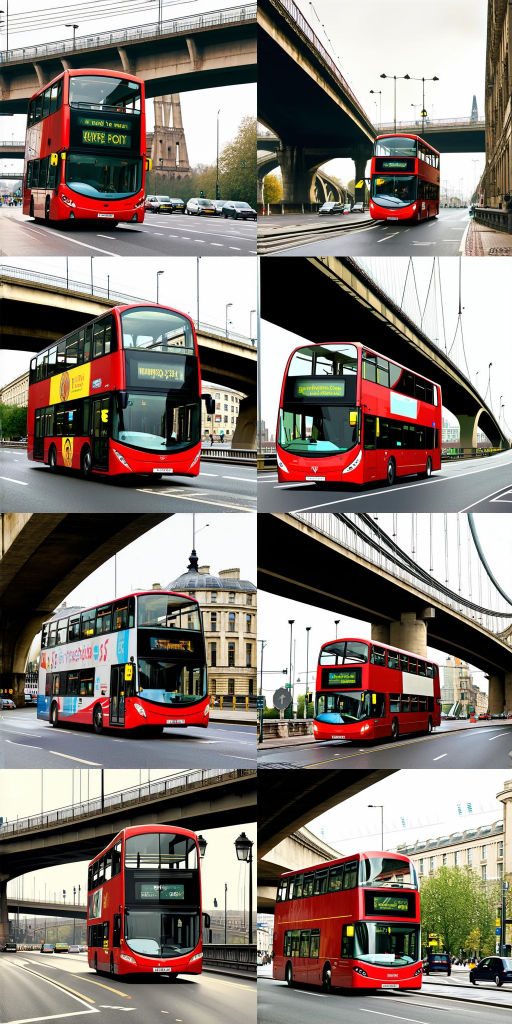}
        \label{fig:grid_cads}
    \end{subfigure}
    \hfill
    \begin{subfigure}[t]{0.22\linewidth}
        \centering
        \caption{DPP}
        \includegraphics[width=\linewidth]{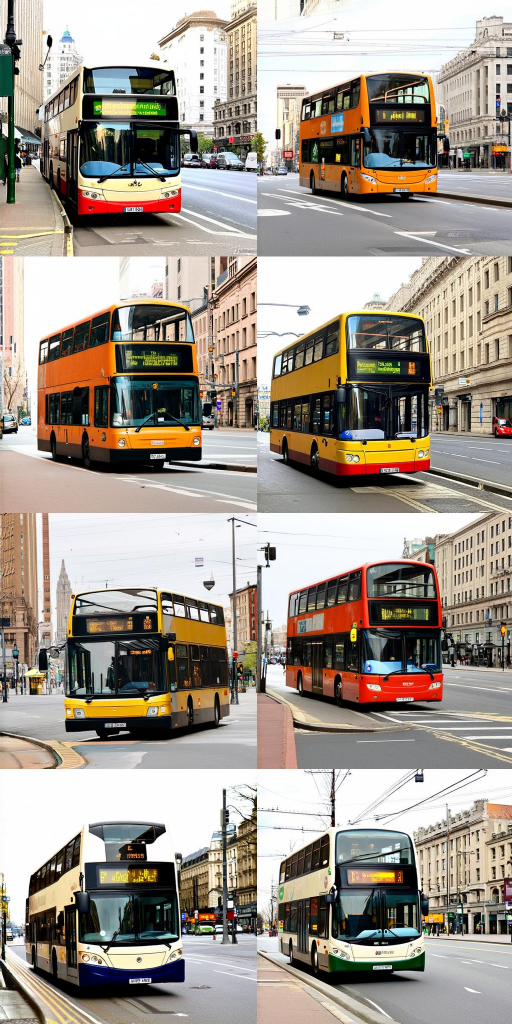}
        \vspace{2pt} 
        \includegraphics[width=\linewidth]{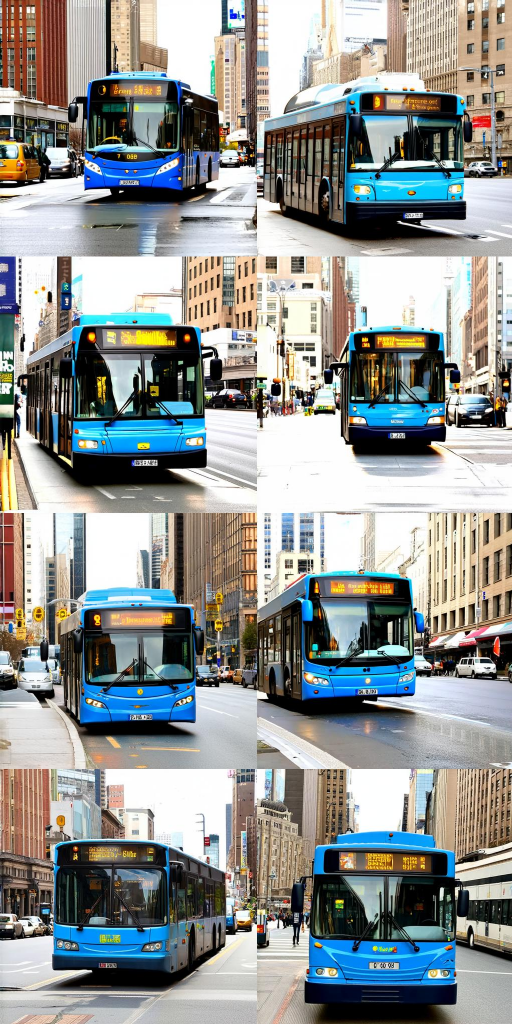}
        \vspace{2pt} 
        \includegraphics[width=\linewidth]{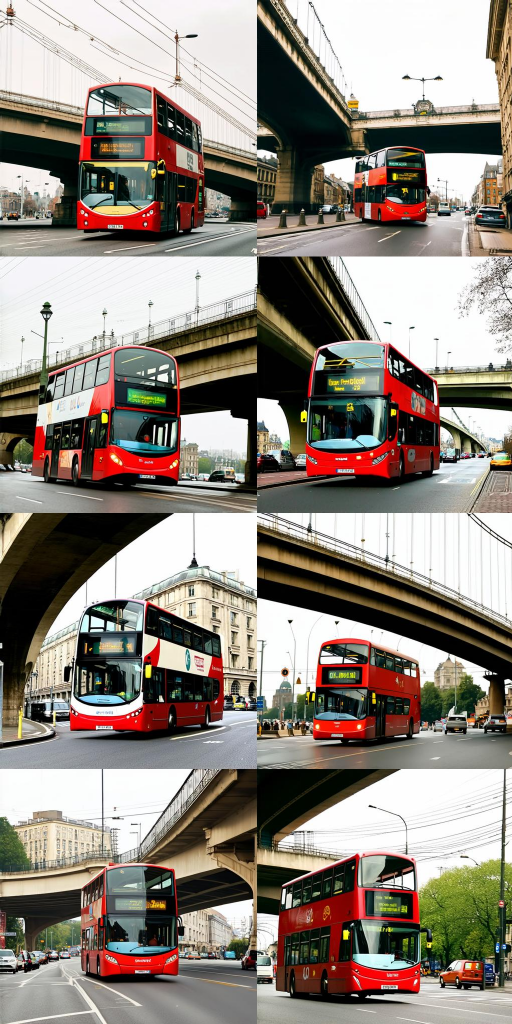}
        \label{fig:grid_dpp}
    \end{subfigure}
    \hfill
    \begin{subfigure}[t]{0.22\linewidth}
        \centering
        \caption{PG}
        \includegraphics[width=\linewidth]{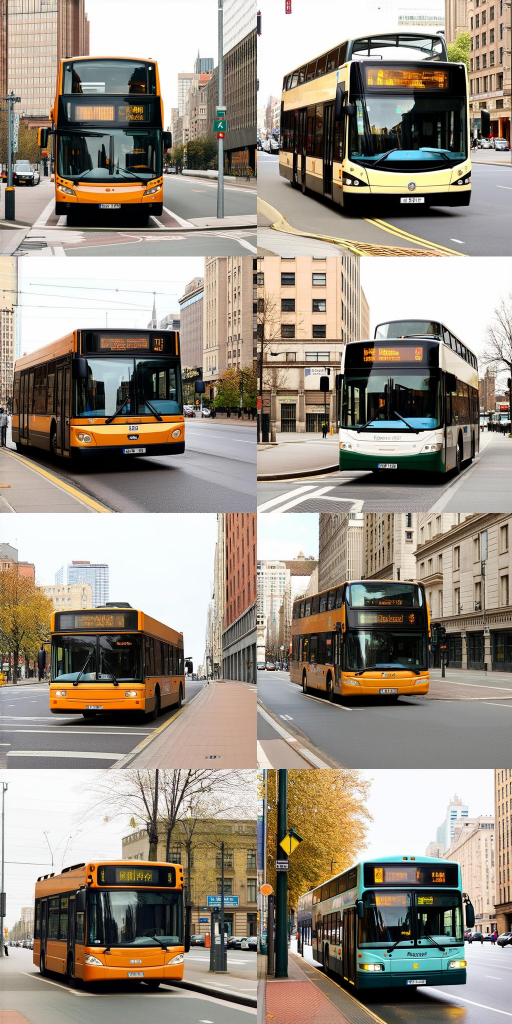}
        \vspace{2pt} 
        \includegraphics[width=\linewidth]{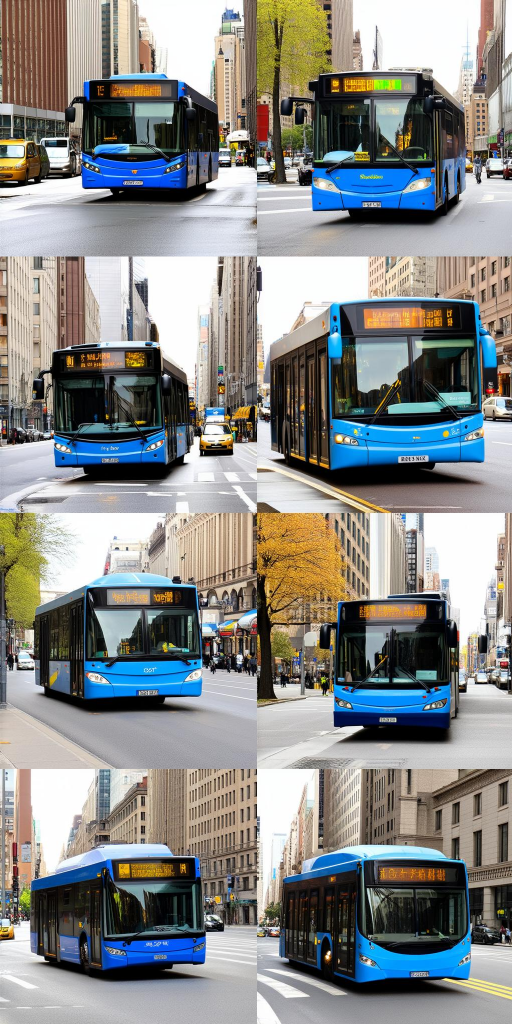}
        \vspace{2pt} 
        \includegraphics[width=\linewidth]{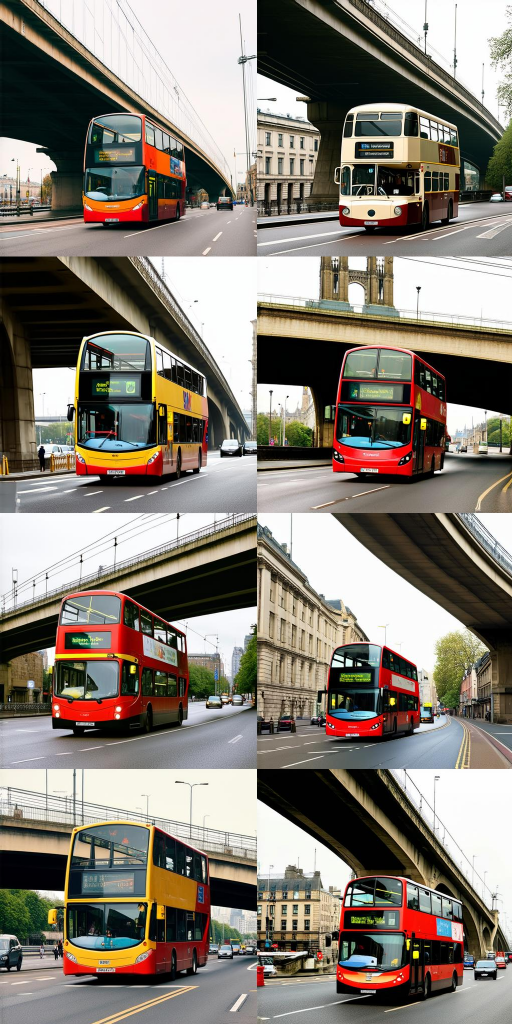}
        \label{fig:grid_pg}
    \end{subfigure}
    \hfill
    \begin{subfigure}[t]{0.22\linewidth}
        \centering
        \caption{Our Method}
        \includegraphics[width=\linewidth]{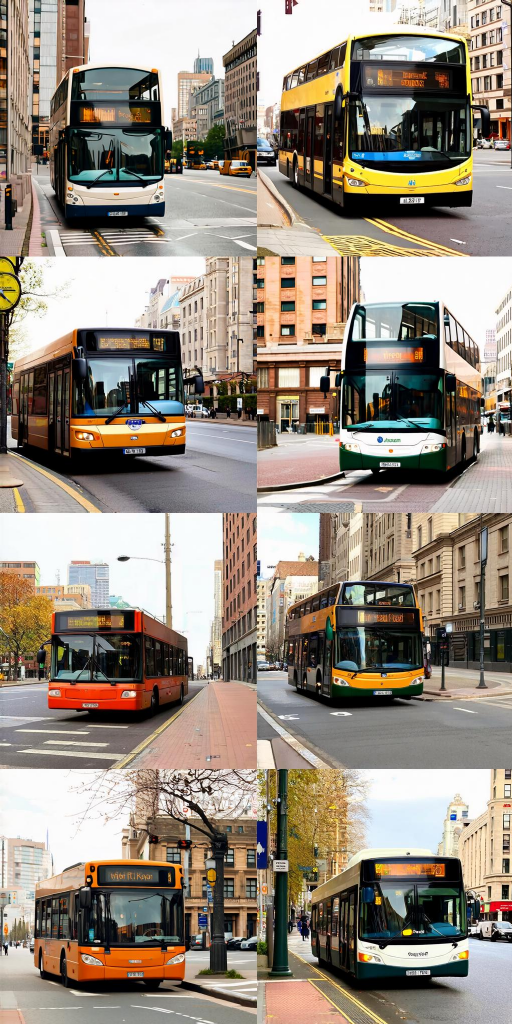}
        \vspace{2pt} 
        \includegraphics[width=\linewidth]{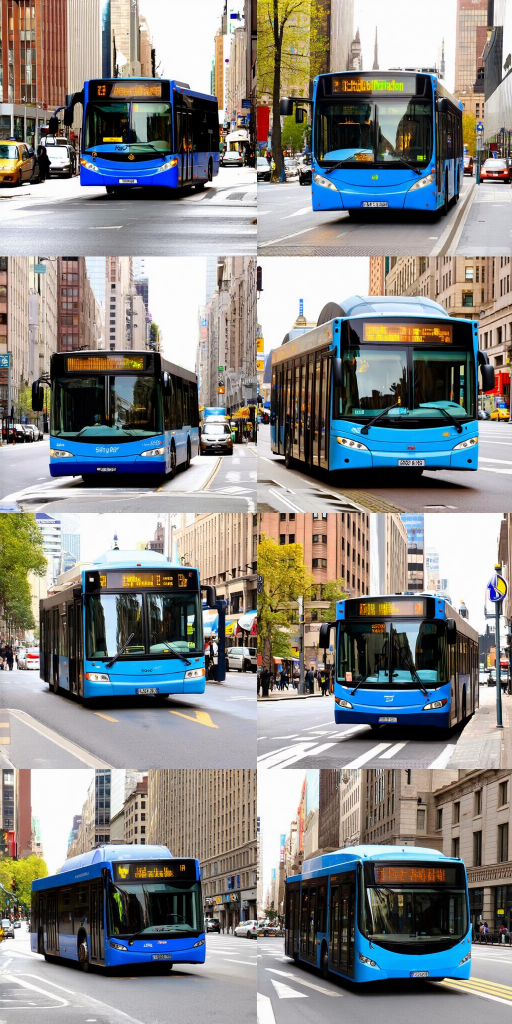}
        \vspace{2pt} 
        \includegraphics[width=\linewidth]{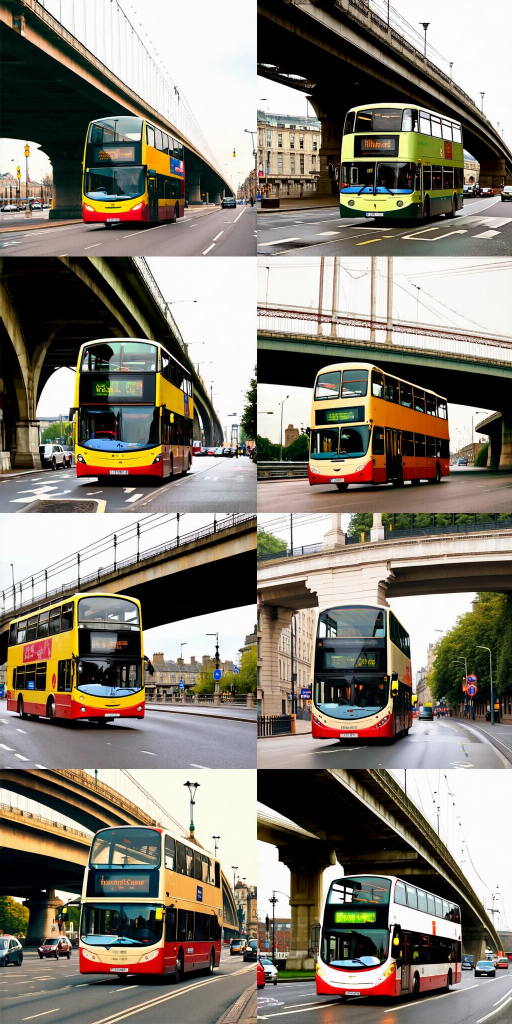}
    \end{subfigure}
        \caption{
        Qualitative comparison of conditional generalization (CIM) across all methods.
        These images were generated using \textit{dense prompts}, where specific attributes were explicitly requested, such as ``a blue bus on a city street'' or ``a single-decked bus drives down the street''.
        This setting evaluates each model's ability to follow precise instructions.
    }
\end{figure*}

\begin{figure*}[htbp]

    \label{fig:coarse_prompt_comparison}
    
    \begin{subfigure}[t]{0.22\linewidth}
        \centering
        \caption{CADS}
        \includegraphics[width=\linewidth]{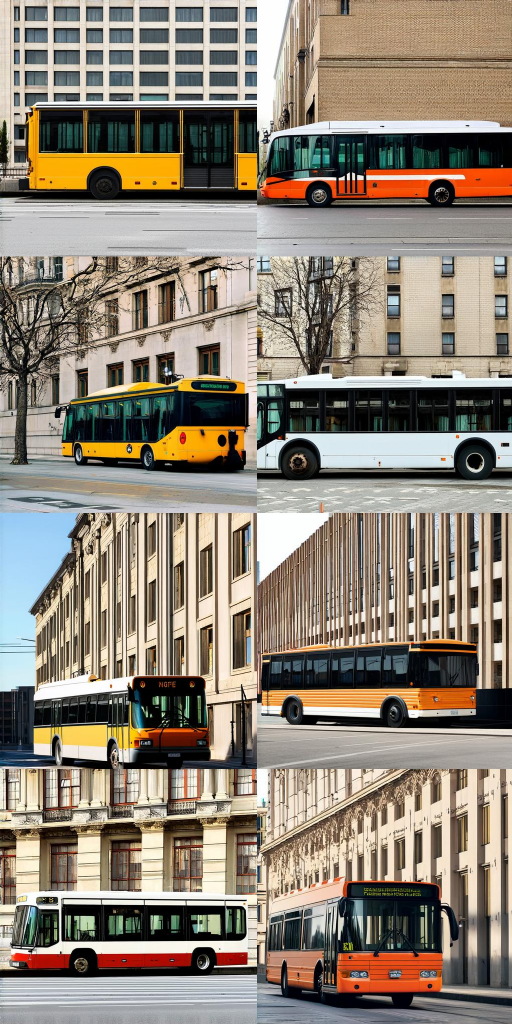}
        \vspace{2pt} 
        \includegraphics[width=\linewidth]{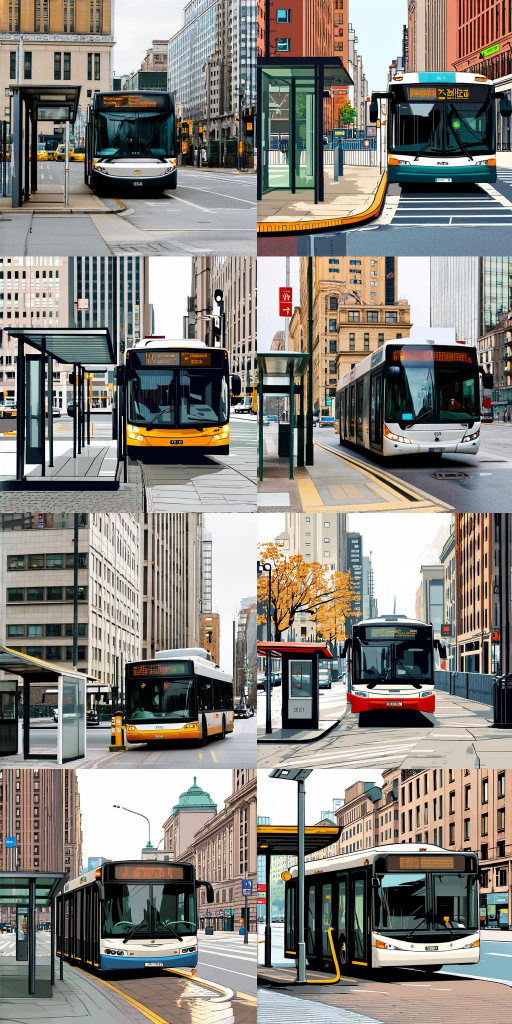}
        \vspace{2pt} 
        \includegraphics[width=\linewidth]{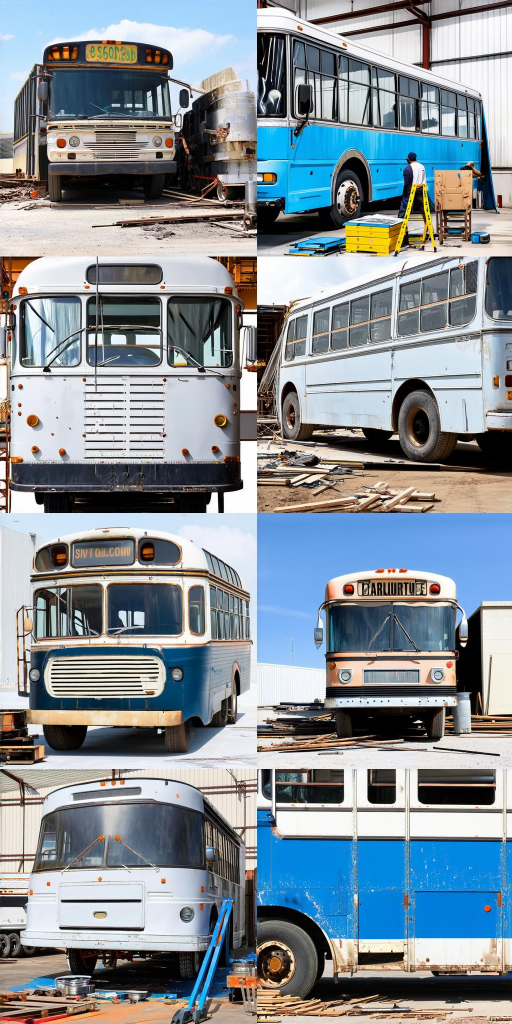}
        \label{fig:grid_cads}
    \end{subfigure}
    \hfill
    \begin{subfigure}[t]{0.22\linewidth}
        \centering
        \caption{DPP}
        \includegraphics[width=\linewidth]{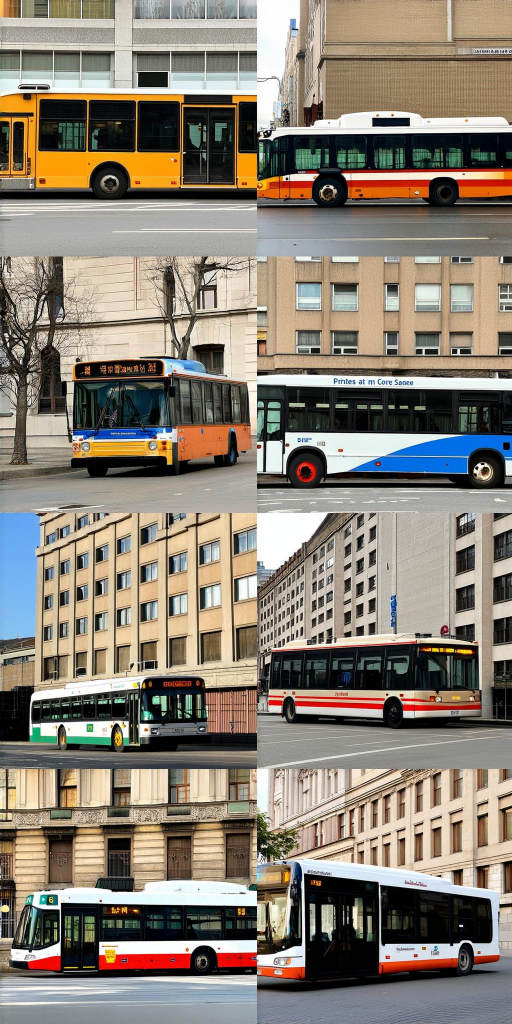}
        \vspace{2pt} 
        \includegraphics[width=\linewidth]{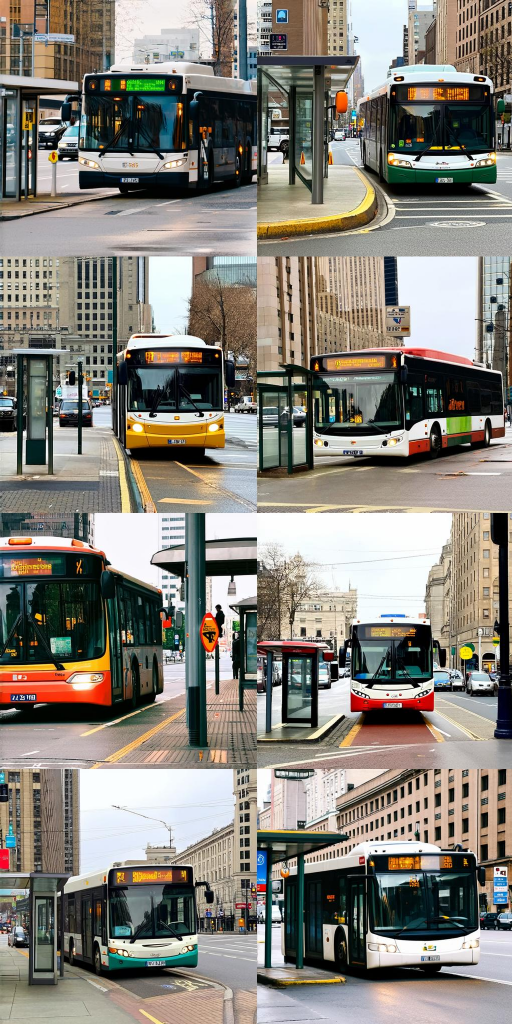}
        \vspace{2pt} 
        \includegraphics[width=\linewidth]{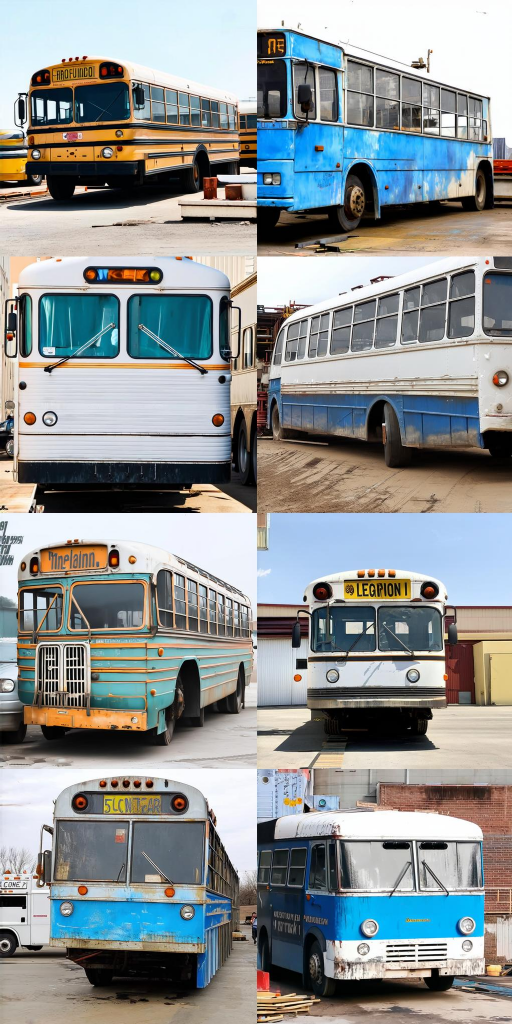}
        \label{fig:grid_dpp}
    \end{subfigure}
    \hfill
    \begin{subfigure}[t]{0.22\linewidth}
        \centering
        \caption{PG}
        \includegraphics[width=\linewidth]{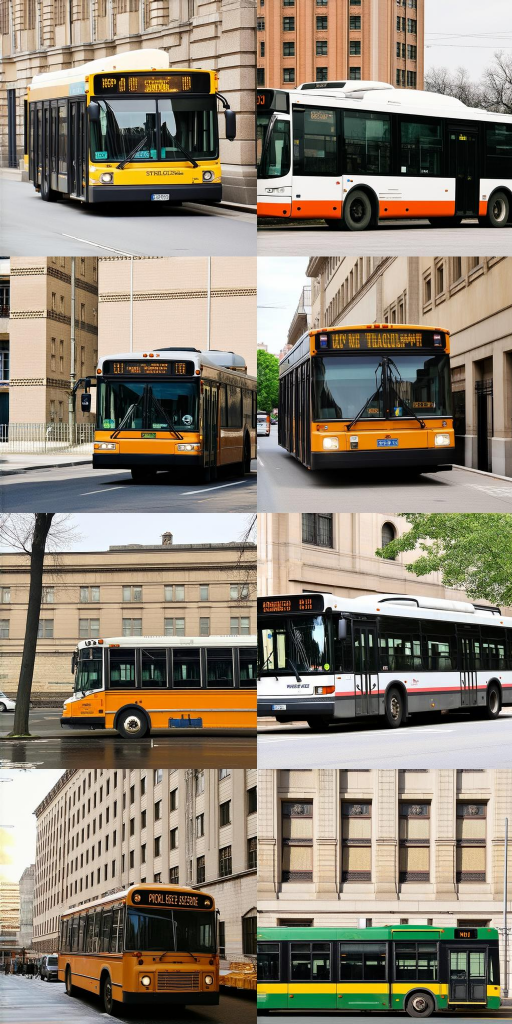}
        \vspace{2pt} 
        \includegraphics[width=\linewidth]{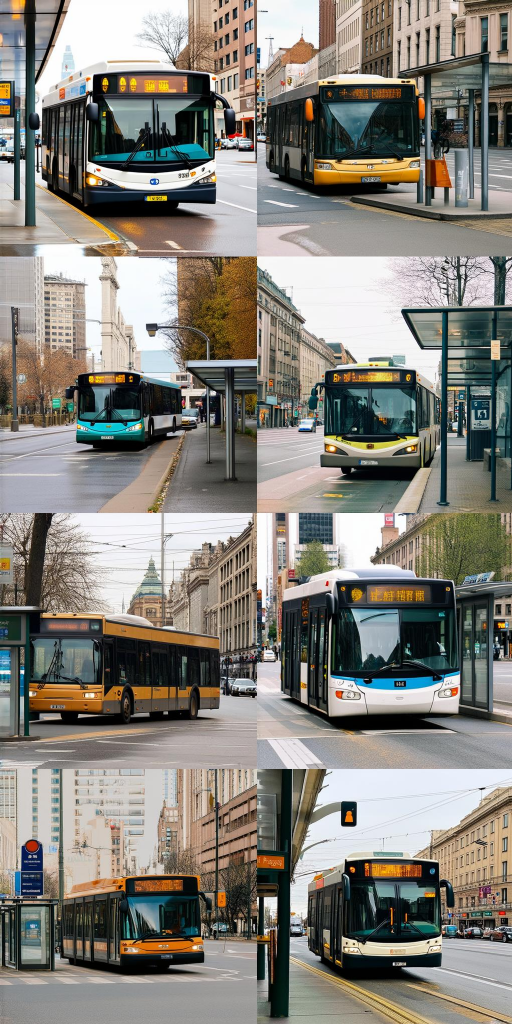}
        \vspace{2pt} 
        \includegraphics[width=\linewidth]{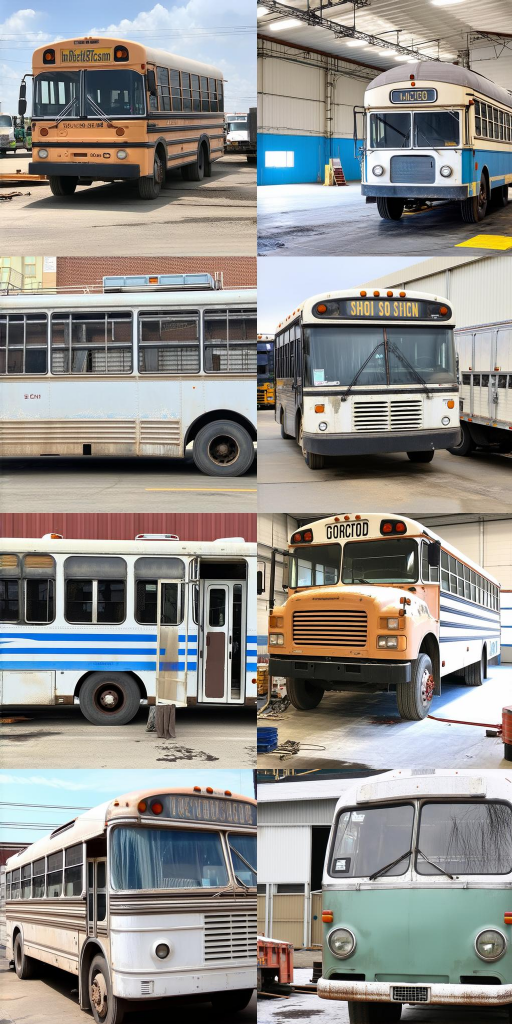}
        \label{fig:grid_pg}
    \end{subfigure}
    \hfill
    \begin{subfigure}[t]{0.22\linewidth}
        \centering
        \caption{Our Method}
        \includegraphics[width=\linewidth]{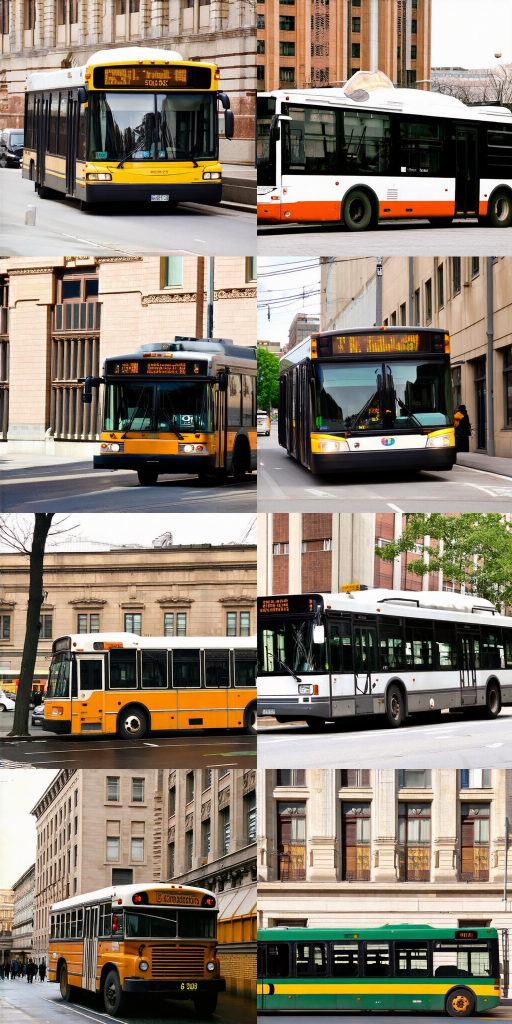}
        \vspace{2pt} 
        \includegraphics[width=\linewidth]{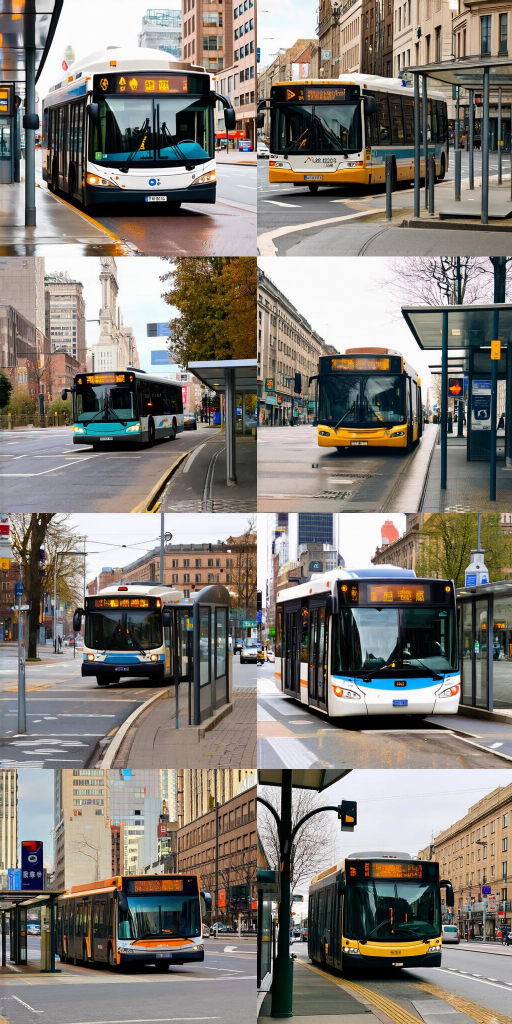}
        \vspace{2pt} 
        \includegraphics[width=\linewidth]{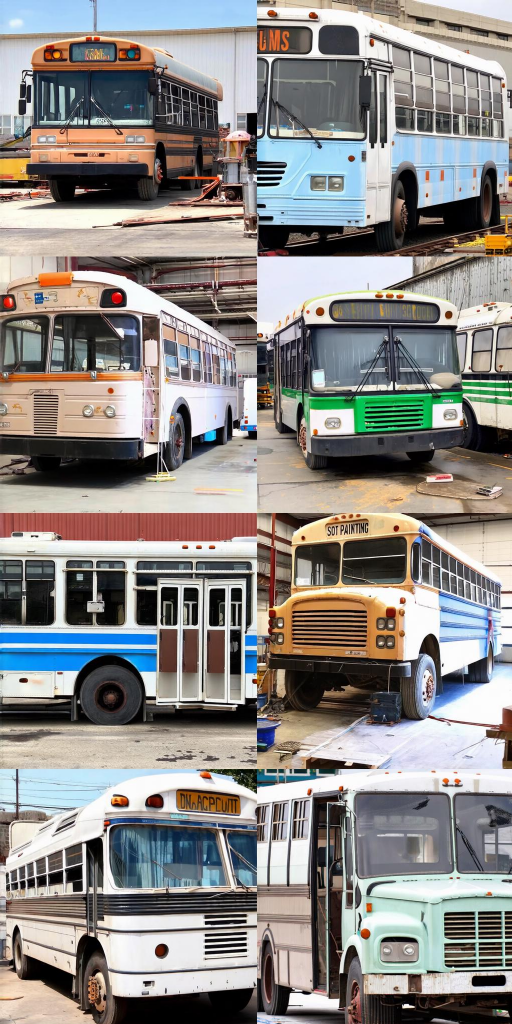}
    \end{subfigure}
        \caption{
        Qualitative comparison of default-mode diversity (DIM) across all methods.
        These images were generated using underspecified \textit{coarse prompts}, such as ``a bus by the side of a building'' or ``a bus stopping at a bus stop in a city''. This experiment evaluates each model's ability to generate a diverse and balanced set of attributes spontaneously.
    }
\end{figure*}

\clearpage
\FloatBarrier

\begin{figure}[p]
    \centering

    \includegraphics[width=\linewidth,height=0.23\textheight,keepaspectratio]{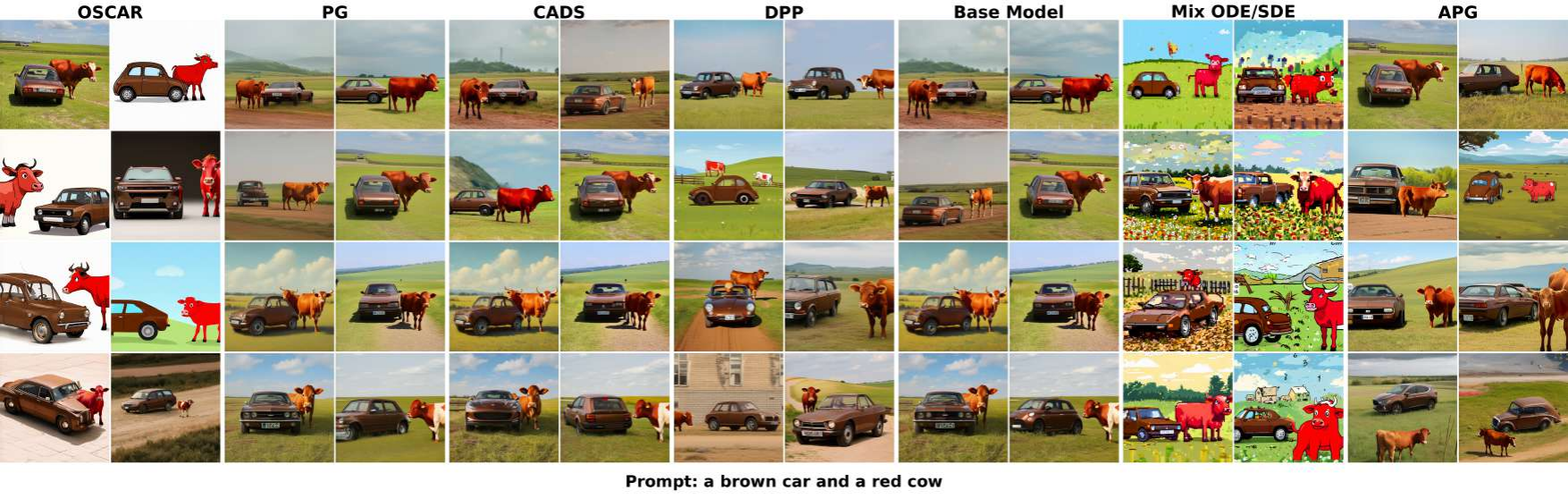}
    \vspace{0.4em}

    \includegraphics[width=\linewidth,height=0.23\textheight,keepaspectratio]{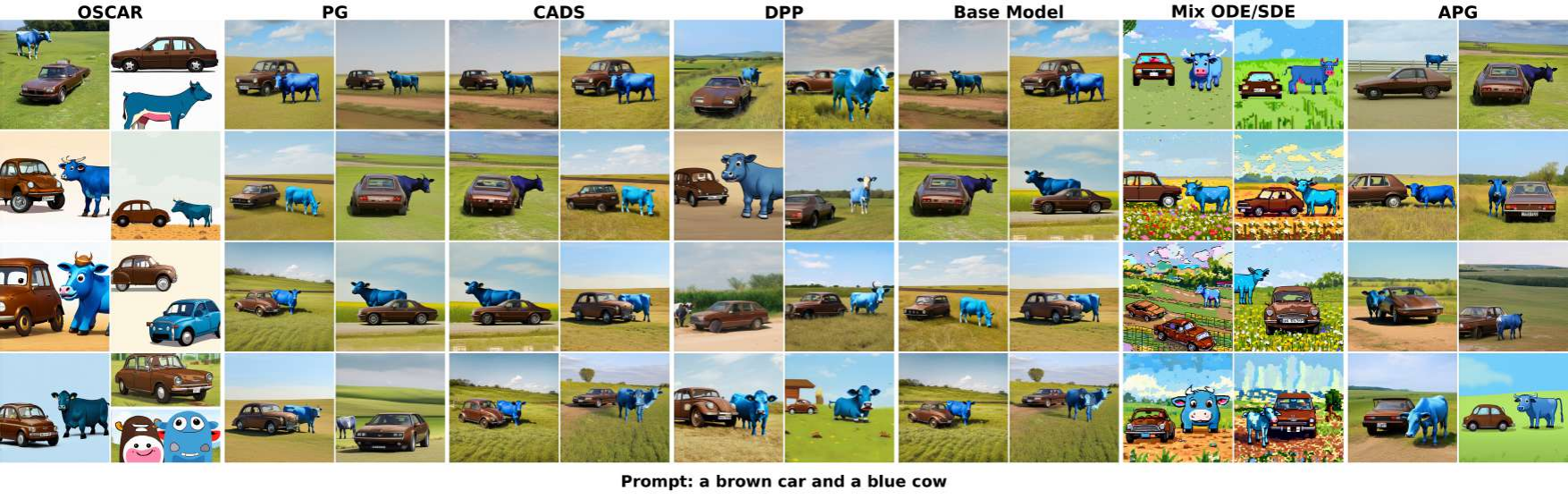}
    \vspace{0.4em}

    \includegraphics[width=\linewidth,height=0.23\textheight,keepaspectratio]{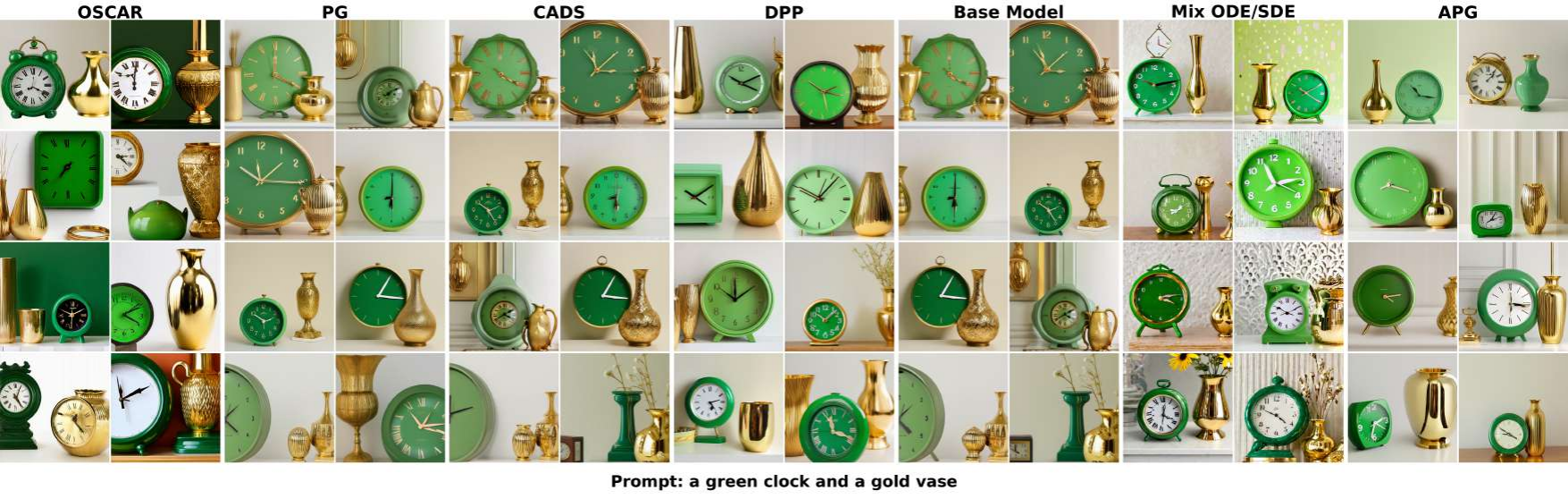}
    \vspace{0.4em}

    \includegraphics[width=\linewidth,height=0.23\textheight,keepaspectratio]{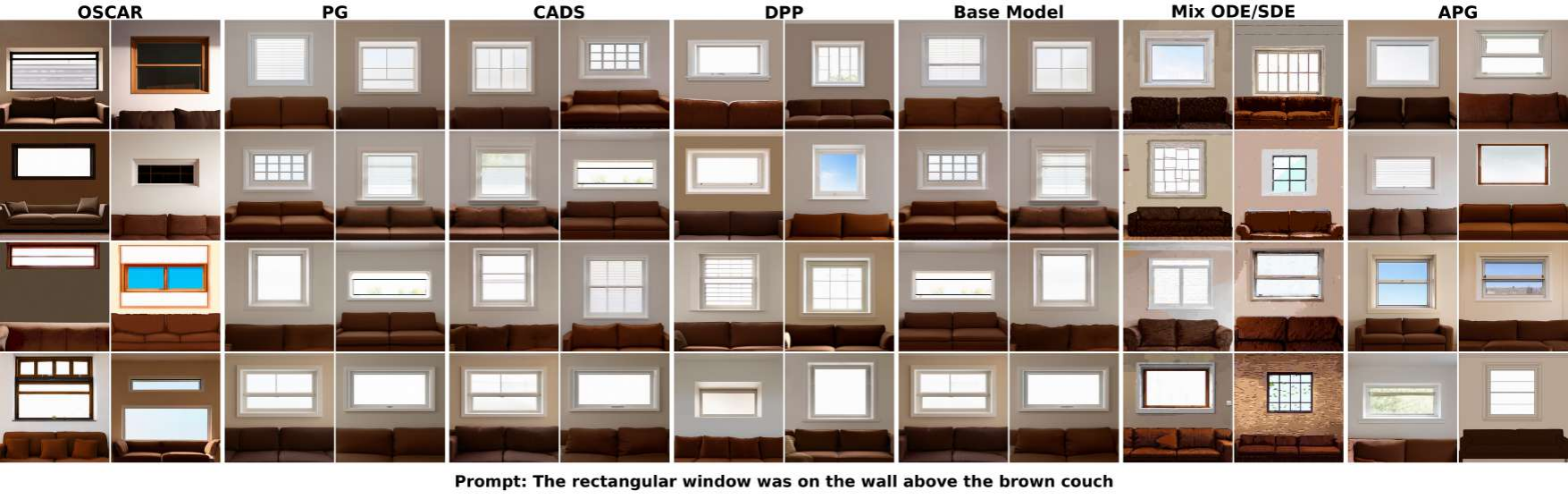}

    \caption{\small
    Qualitative visual comparisons on the \textit{spatial color} and \textit{complex} subsets of T2I-CompBench. The prompts are shown below each image.
    }
    \label{fig:t2i_compbench_visuals_1}
\end{figure}

\clearpage
\FloatBarrier

\begin{figure}[p]
    \centering

    \includegraphics[width=\linewidth,height=0.23\textheight,keepaspectratio]{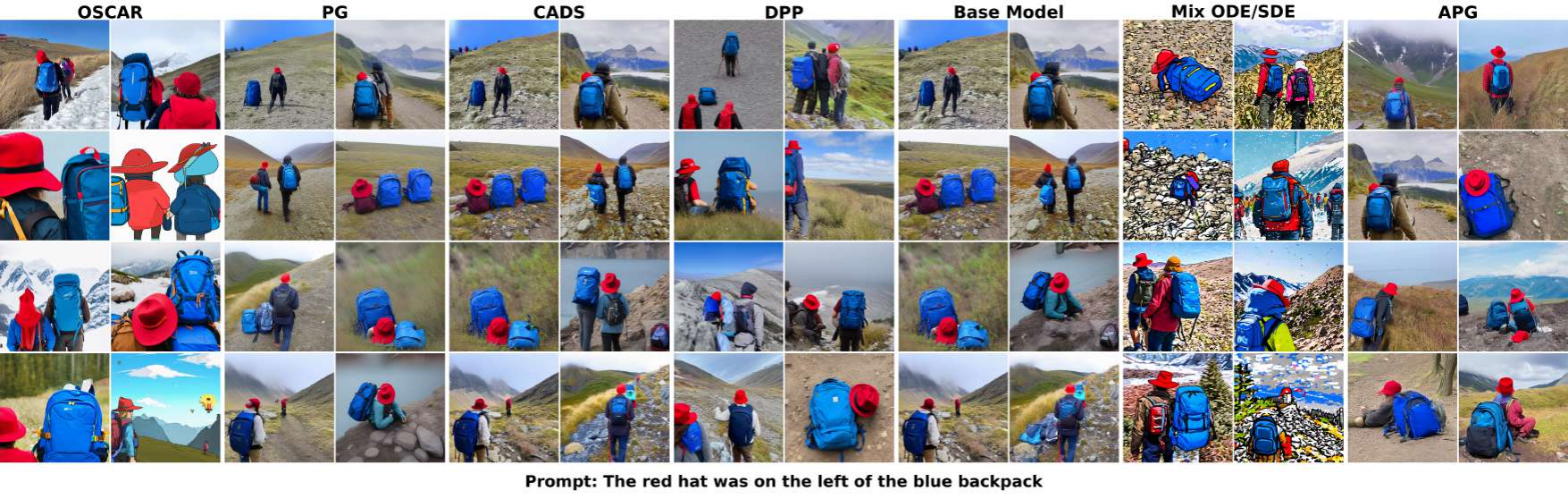}
    \vspace{0.5em}

    \includegraphics[width=\linewidth,height=0.23\textheight,keepaspectratio]{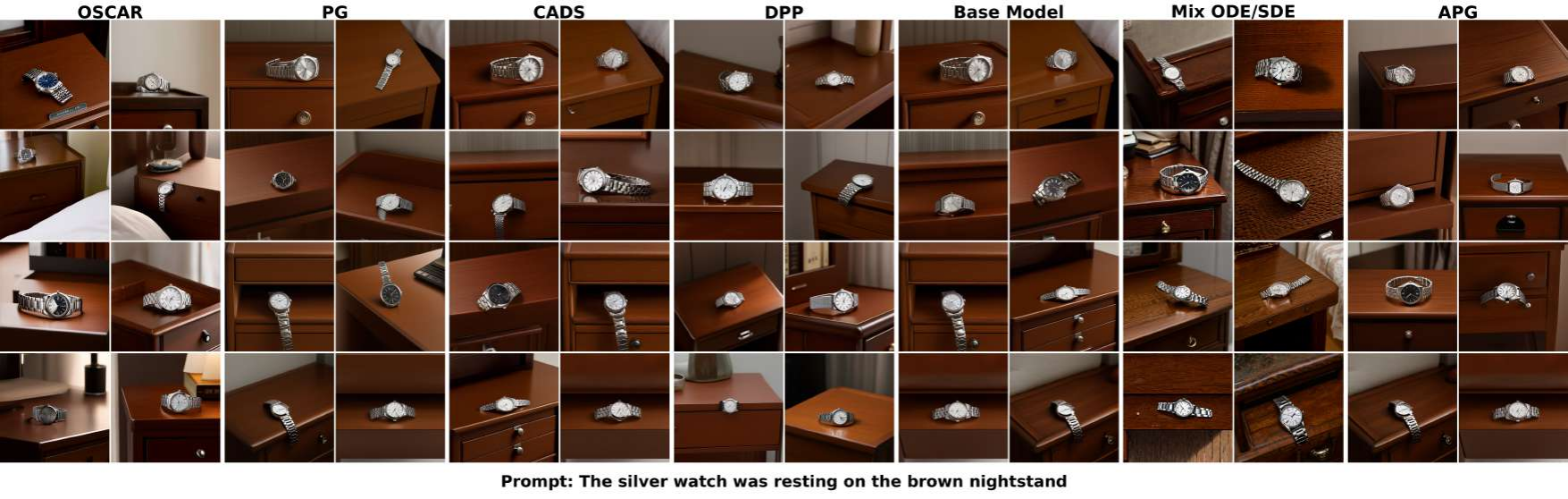}
    \vspace{0.5em}

    \includegraphics[width=\linewidth,height=0.23\textheight,keepaspectratio]{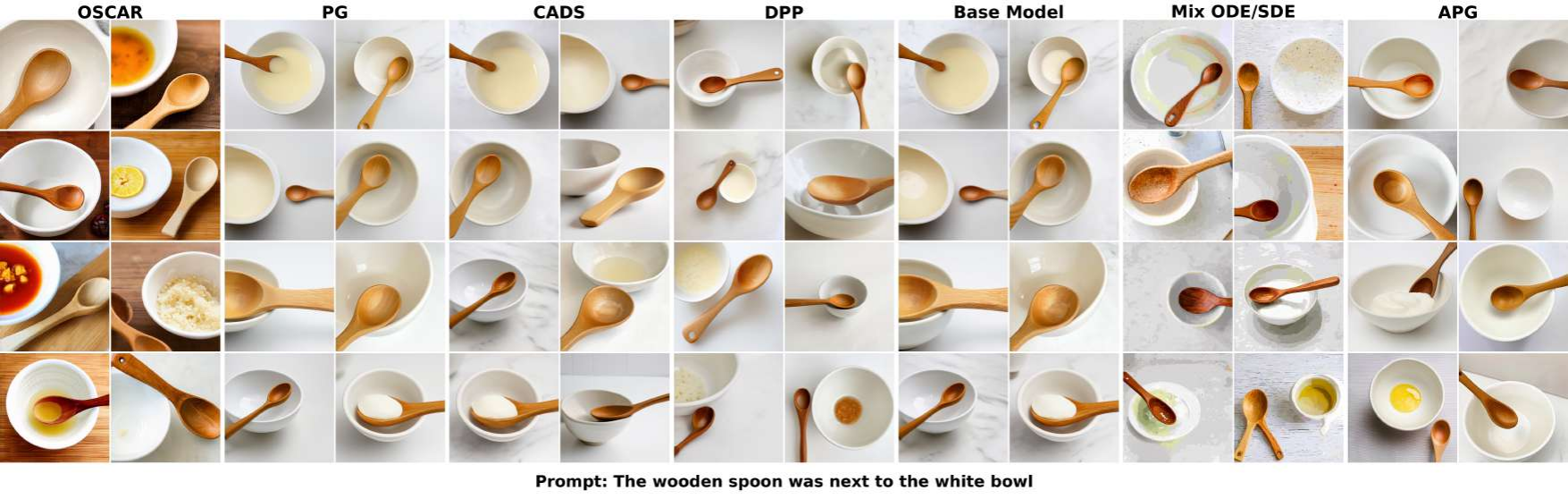}

    \includegraphics[width=\linewidth,height=0.23\textheight,keepaspectratio]{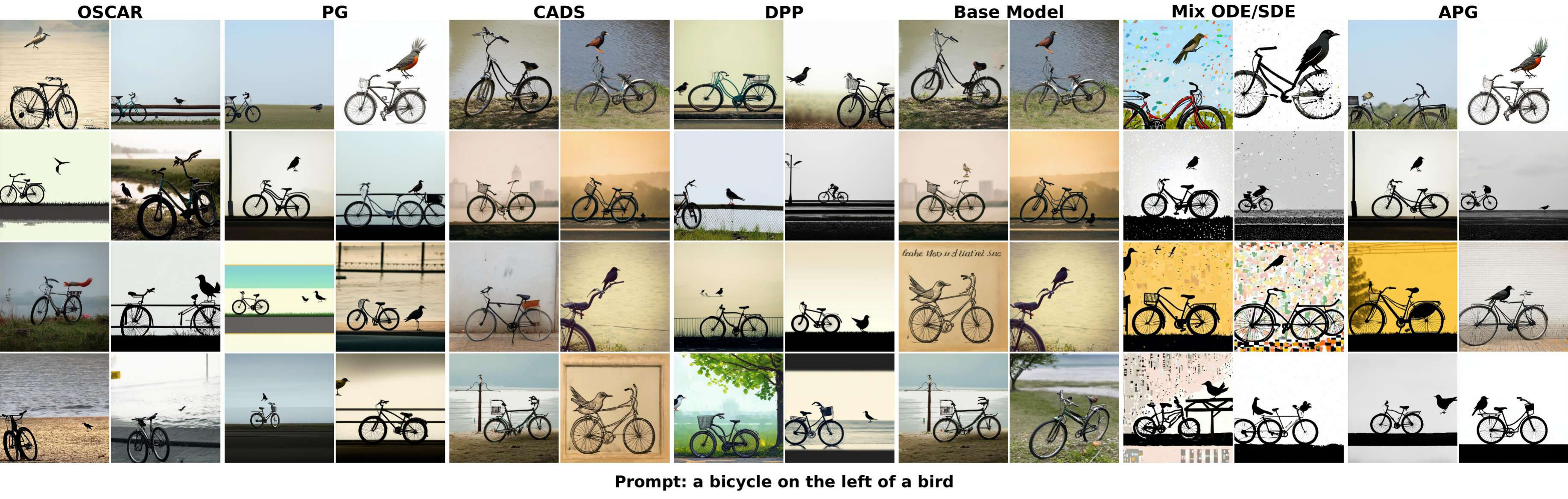}

    \caption{\small
    Additional qualitative visual comparisons on the \textit{complex} subset of T2I-CompBench. The prompts are shown below each image.
    }
    \label{fig:t2i_compbench_visuals_2}
\end{figure}

\FloatBarrier

\clearpage
\FloatBarrier

\begin{figure}[p]
    \centering

    \includegraphics[width=\linewidth,height=0.20\textheight,keepaspectratio]{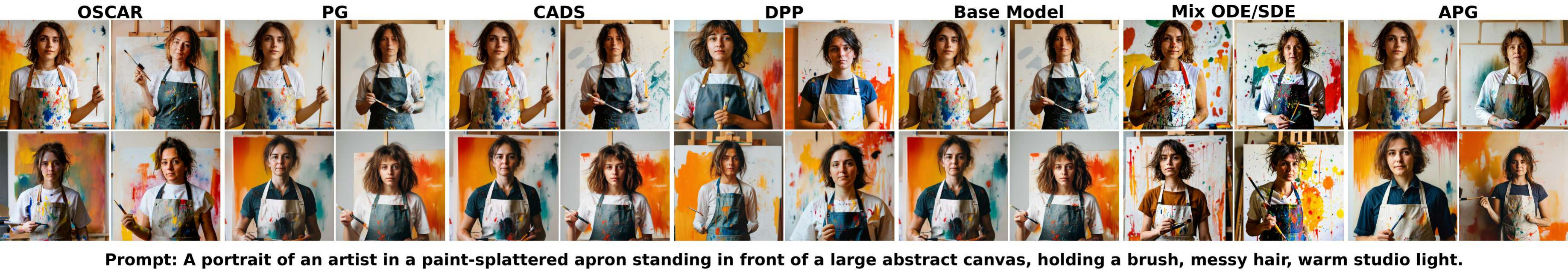}
    \vspace{0.4em}

    \includegraphics[width=\linewidth,height=0.20\textheight,keepaspectratio]{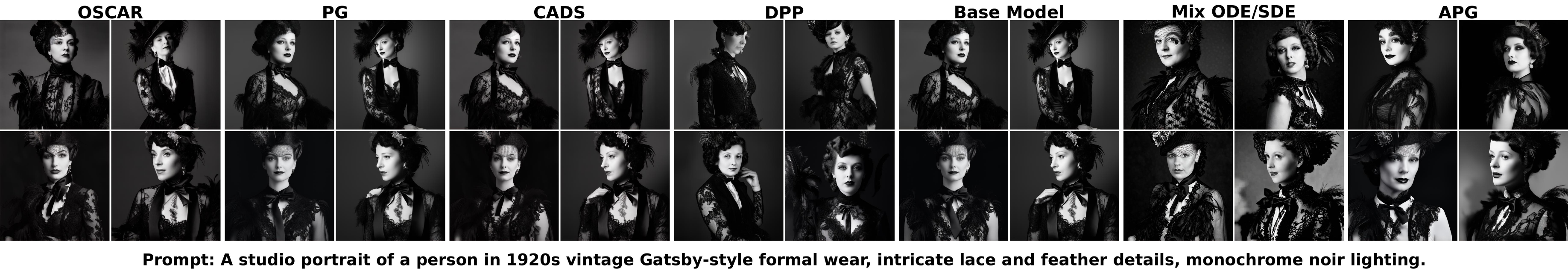}
    \vspace{0.4em}

    \includegraphics[width=\linewidth,height=0.20\textheight,keepaspectratio]{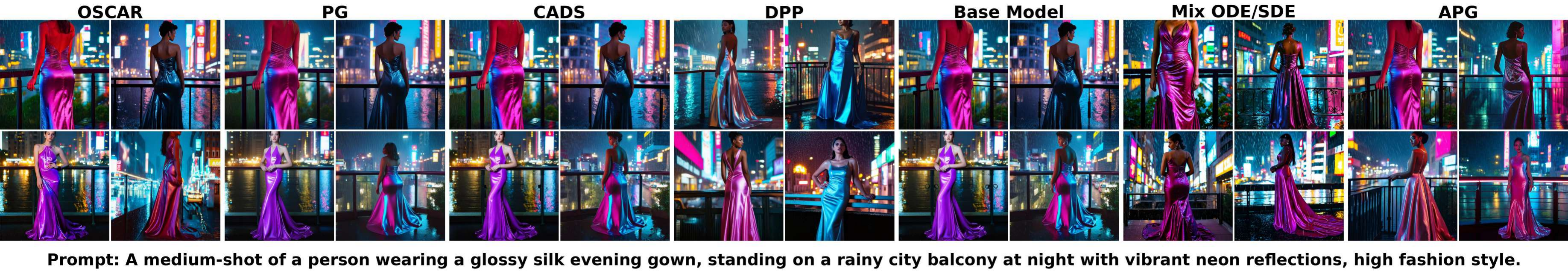}

    \caption{\small
    Portrait-focused qualitative comparisons on challenging prompts. These examples are designed to stress fine-grained visual fidelity, including facial plausibility, skin and hand details, material rendering, and complex lighting conditions.
    }
    \label{fig:portrait_fidelity_visuals}
\end{figure}

\FloatBarrier

\end{document}